\documentclass{article}

\usepackage[nonatbib,preprint]{neurips_2019}

\usepackage{./natbib}

\usepackage[utf8]{inputenc} 
\usepackage[T1]{fontenc}    
\usepackage{hyperref}       
\usepackage{url}            
\usepackage{booktabs}       
\usepackage{amsfonts}       
\usepackage{nicefrac}       
\usepackage{microtype}      

\usepackage{algorithm}
\usepackage{algpseudocode}

\usepackage{microtype}
\usepackage{graphicx}
\usepackage{amsmath}
\usepackage{amsthm}
\usepackage{subcaption}
\usepackage{xcolor}

\usepackage{amsmath,amsfonts,bm}

\def\eqref#1{equation~\ref{#1}}

\def\1{\bm{1}}

\DeclareMathAlphabet{\mathsfit}{\encodingdefault}{\sfdefault}{m}{sl}
\SetMathAlphabet{\mathsfit}{bold}{\encodingdefault}{\sfdefault}{bx}{n}

\newcommand{\E}{\mathbb{E}}

\newcommand{\R}{\mathbb{R}}

\newcommand{\Var}{\mathrm{Var}}

\newtheorem{theorem}{Theorem}[section]

\newtheorem{lemma}[theorem]{Lemma}
\newtheorem{proposition}[theorem]{Proposition}

\newcommand{\e}[1]{\mathbb{E}\left[#1\right]}
\newcommand{\var}[1]{\mathrm{Var}\left[#1\right]}
\newcommand{\cov}[1]{\mathrm{Cov}\left[#1\right]}

\newtheorem{assumption}{Assumption}[section]

\newcommand{\noisevector}{\begin{bmatrix} \epsilon_i \\ 0 \end{bmatrix}}

\title{Reducing the variance in online optimization by transporting past gradients}

\author{%
  Sébastien M. R. Arnold
  \thanks{Work done while at Mila.} \\
  University of Southern California \\
  Los Angeles, CA \\
  \texttt{seb.arnold@usc.edu} \\
  \And
  Pierre-Antoine Manzagol \\
  Google Brain \\
  Montréal, QC \\
  \texttt{manzagop@google.com} \\
  \And
  Reza Babanezhad \\
  University of British Columbia \\
  Vancouver, BC \\
  \texttt{rezababa@cs.ubc.ca} \\
  \And
  Ioannis Mitliagkas \\
  Mila, Université de Montréal \\
  Montréal, QC \\
  \texttt{ioannis@iro.umontreal.ca} \\
  \And
  Nicolas Le Roux \\
  Mila, Google Brain \\
  Montréal, QC \\
  \texttt{nlr@google.com} \\
}

\begin{document}

\maketitle

\begin{abstract}
Most stochastic optimization methods use gradients once before discarding them. While variance reduction methods have shown that reusing past gradients can be beneficial when there is a finite number of datapoints, they do not easily extend to the online setting. One issue is the staleness due to using past gradients. We propose to correct this staleness using the idea of {\em implicit gradient transport} (IGT) which transforms gradients computed at previous iterates into gradients evaluated at the current iterate without using the Hessian explicitly. In addition to reducing the variance and bias of our updates over time, IGT can be used as a drop-in replacement for the gradient estimate in a number of well-understood methods such as heavy ball or Adam. We show experimentally that it achieves state-of-the-art results on a wide range of architectures and benchmarks. Additionally, the IGT gradient estimator yields the optimal asymptotic convergence rate for online stochastic optimization in the restricted setting where the Hessians of all component functions are equal.\footnote{Open-source implementation available at: \url{https://github.com/seba-1511/igt.pth}}
\end{abstract}

    \section{Introduction}
    We wish to solve the following minimization problem:
    \begin{align}
        \label{eq:stochastic_optim}
        \theta^\ast 	&= \arg\min_\theta E_{x\sim p} [f(\theta, x)] \; ,
    \end{align}
    where we only have access to samples $x$ and to a first-order oracle that gives us, for a given $\theta$ and a given $x$, the derivative of $f(\theta, x)$ with respect to $\theta$, i.e. $\frac{\partial f(\theta, x)}{\partial \theta} = g(\theta, x)$.
    It is known~\cite{robbins1951stochastic} that, when $f$ is smooth and strongly convex, there is a converging algorithm for Problem~\ref{eq:stochastic_optim} that takes the form $\theta_{t+1}	= \theta_t - \alpha_t g(\theta_t, x_t)$, where $x_t$ is a sample from $p$. This algorithm, dubbed stochastic gradient (SG), has a convergence rate of $O(1/t)$ (see for instance~\cite{bubeck2015convex}), within a constant factor of the minimax rate for this problem. When one has access to the true gradient $g(\theta) = E_{x \sim p} [g(\theta, x)]$ rather than just a sample, this rate dramatically improves to $O(e^{-\nu t})$ for some $\nu > 0$.

    In addition to hurting the convergence speed, noise in the gradient makes optimization algorithms harder to tune. Indeed, while full gradient descent is convergent for constant stepsize $\alpha$, and also amenable to line searches to find a good value for that stepsize, the stochastic gradient method from~\cite{robbins1951stochastic} with a constant stepsize only converges to a ball around the optimum~\cite{schmidt2014convergence}.\footnote{Under some conditions, it does converge linearly to the optimum~\citep[e.g.,][]{Vaswani19}} Thus, to achieve convergence, one needs to use a decreasing stepsize. While this seems like a simple modification, the precise decrease schedule can have a dramatic impact on the convergence speed. While theory prescribes $\alpha_t = O(t^{-\alpha})$ with $\alpha \in (1/2, 1]$ in the smooth case, practictioners often use larger stepsizes like $\alpha_t = O(t^{-1/2})$ or even constant stepsizes.

    When the distribution $p$ has finite support, Eq.~\ref{eq:stochastic_optim} becomes a finite sum and, in that setting, it is possible to achieve efficient variance reduction and drive the noise to zero, allowing stochastic methods to achieve linear convergence rates~\cite{leroux2012stochastic,johnson2013accelerating,zhang2013linear,mairal2013optimization,Shalev-Shwartz2013sdca,defazio2014saga}. Unfortunately, the finite support assumption is critical to these algorithms which, while valid in many contexts, does not have the broad applicability of the standard SG algorithm. Several works have extended these approaches to the online setting by applying these algorithms while increasing $N$~\cite{babanezhad2015wasting,hofmann2015variance} but they need to revisit past examples mutiple times and are not truly online.
    
    Another line of work reduces variance by averaging iterates~\cite{polyak1992acceleration,lacoste2012simpler,bach2013non,flammarion2015averaging,dieuleveut2017harder,dieuleveut2017bridging,JMLR:v18:16-595}. While these methods converge for a constant stepsize in the stochastic case\footnote{Under some conditions on $f$.}, their practical speed is heavily dependent on the fraction of iterates kept in the averaging, a hyperparameter that is thus hard to tune, and they are rarely used in deep learning.

    Our work combines two existing ideas and adds a third: a) At every step, it updates the parameters using a weighted average of past gradients, like in SAG~\citep{leroux2012stochastic,schmidt2017minimizing}, albeit with a different weighting scheme; b) It reduces the bias and variance induced by the use of these old gradients by transporting them to ``equivalent'' gradients computed at the current point, similar to~\cite{gower2017tracking}; c) It does so implicitly by computing the gradient at a parameter value different from the current one. The resulting gradient estimator can then be used as a plug-in replacement of the stochastic gradient within any optimization scheme. Experimentally, both SG using our estimator and its momentum variant outperform the most commonly used optimizers in deep learning.
    


    \section{Momentum and other approaches to dealing with variance}
    \label{sec:momentum-effect-variance}
    
    Stochastic variance reduction methods use an average of past gradients to reduce the variance of the gradient estimate. At first glance, it seems like their updates are similar to that of momentum~\cite{polyak1964some}, also known as the heavy ball method, which performs the following updates\footnote{This is slightly different from the standard formulation but equivalent for constant $\gamma_t$.}:
    \begin{align*}
        v_t 	&= \gamma_t v_{t-1} + (1 - \gamma_t) g(\theta_t, x_t), \qquad v_{0}  = g(\theta_0, x_0) \\
        \theta_{t+1}	&= \theta_t - \alpha_t v_t\; .
    \end{align*}
    When $\gamma_t = \gamma$, this leads to $\displaystyle \theta_{t+1}	= \theta_t - \alpha_t \left(\gamma^t g(\theta_0, x_0) + (1 - \gamma) \sum_{i=1}^t  \gamma^{t-i} g(\theta_i, x_i)\right)$.
    Hence, the heavy ball method updates the parameters of the model using an average of past gradients, bearing similarity with SAG~\cite{leroux2012stochastic}, albeit with exponential instead of uniform weights.
    
    Interestingly, while momentum is a popular method for training deep networks, its theoretical analysis in the stochastic setting is limited~\citep{sutskever2013importance}, except in the particular setting when the noise converges to 0 at the optimum~\citep{loizou2017momentum}. Also surprising is that, despite the apparent similarity with stochastic variance reduction methods, current convergence rates are slower when using $\gamma > 0$ in the presence of noise~\cite{Schmidt11_inexact}, although this might be a limitation of the analysis.
    
    \subsection{Momentum and variance}
    We propose here an analysis of how, on quadratics, using past gradients as done in momentum does not lead to a decrease in variance. If gradients are stochastic, then $\Delta_t = \theta_t - \theta^\ast$ is a random variable. Denoting $\epsilon_i$ the noise at timestep $i$, i.e. $\displaystyle g(\theta_i, x_i) = g(\theta_i) + \epsilon_i$, and writing $\Delta_t - E[\Delta_t] = \alpha \sum_{i=0}^t N_{i, t}\epsilon_i$, with $N_{i,t}$ the impact of the noise of the $i$-th datapoint on the $t$-th iterate, we may now analyze the total impact of each $\epsilon_i$ on the iterates.
    Figure~\ref{fig:N-gamma} shows the impact of $\epsilon_i$ on $\Delta_t - E[\Delta_t]$ as measured by $N_{i,t}^2$ for three datapoints ($i=1$, $i=25$ and $i=50$) as a function of $t$ for stochastic gradient ($\gamma=0$, left) and momentum ($\gamma=0.9$, right).
    As we can see, when using momentum, the variance due to a given datapoint first increases as the noise influences both the next iterate (through the parameter update) and the subsequent updates (through the velocity).
    Due to the weight $1-\gamma$ when a point is first sampled, a larger value of $\gamma$ leads to a lower immediate impact of the noise of a given point on the iterates. However, a larger $\gamma$ also means that the noise of a given gradient is kept longer, leading to little or no decrease of the total variance (dashed blue curve). Even in the case of stochastic gradient, the noise at a given timestep carries over to subsequent timesteps, even if the old gradients are not used for the update, as the iterate itself depends on the noise.
    
\begin{figure}
\begin{center}
\begin{subfigure}{.42\textwidth}
    \includegraphics[width=\textwidth]{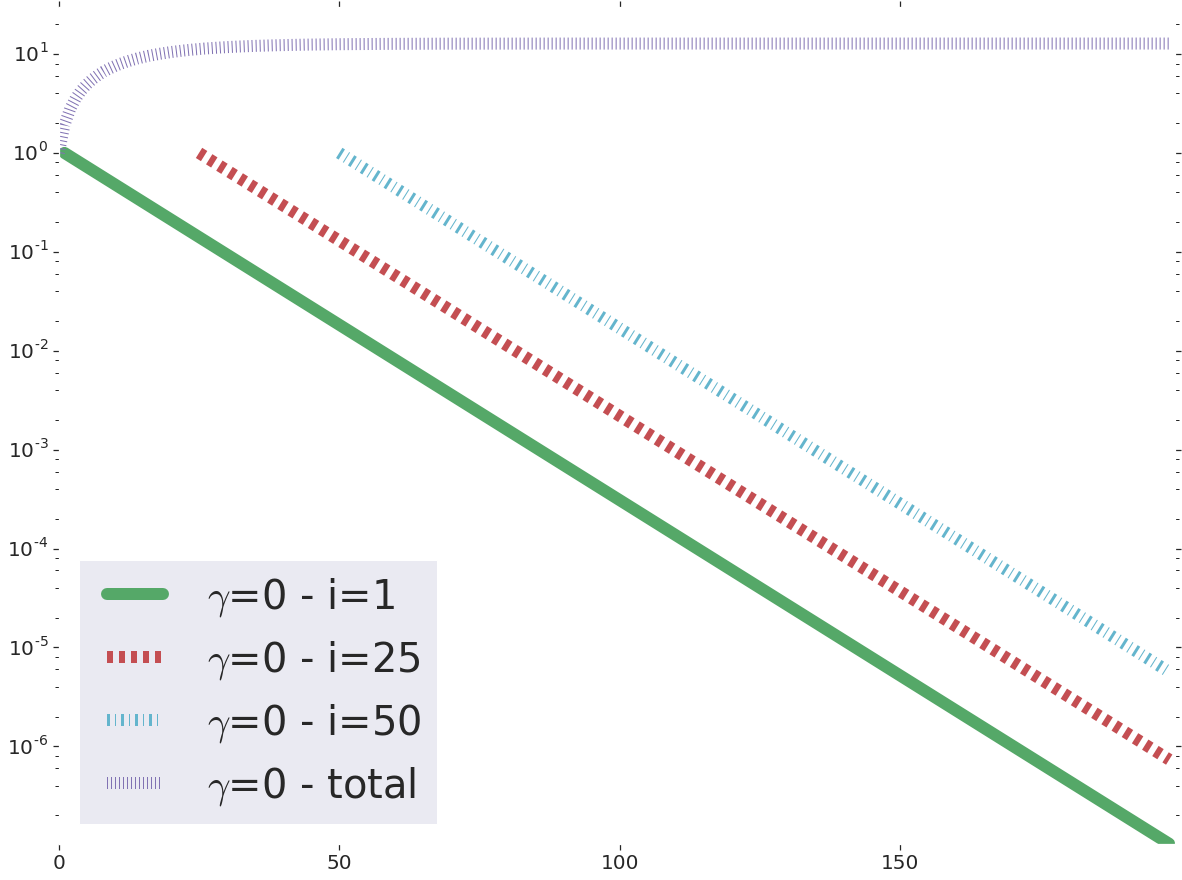}
    \caption{Stochastic gradient \label{fig:N_gamma_stoch}}
\end{subfigure}
\hspace{15mm}
\begin{subfigure}{.42\textwidth}
    \includegraphics[width=\textwidth]{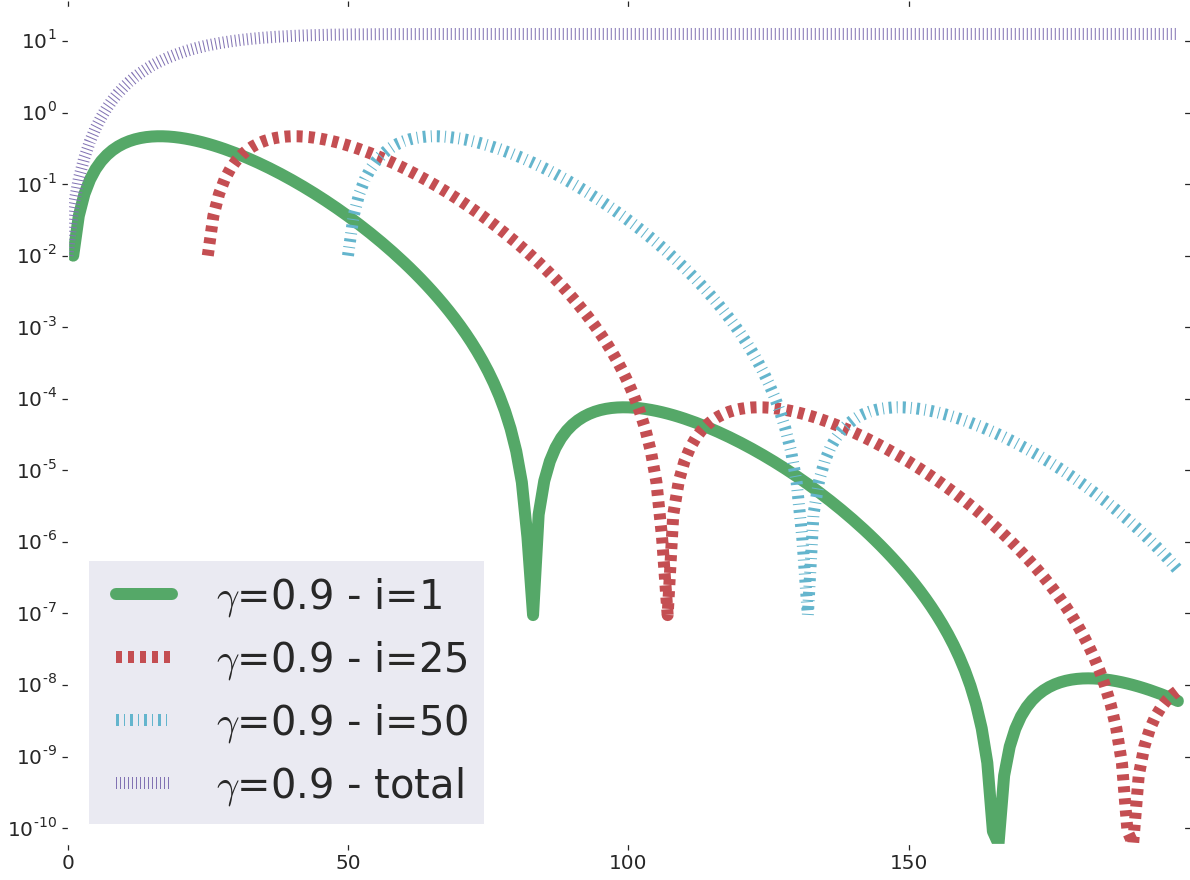}
    \caption{Momentum - $\gamma = 0.9$ \label{fig:N_gamma_mom}}
\end{subfigure}\\
\begin{subfigure}{.42\textwidth}
    \includegraphics[width=\textwidth]{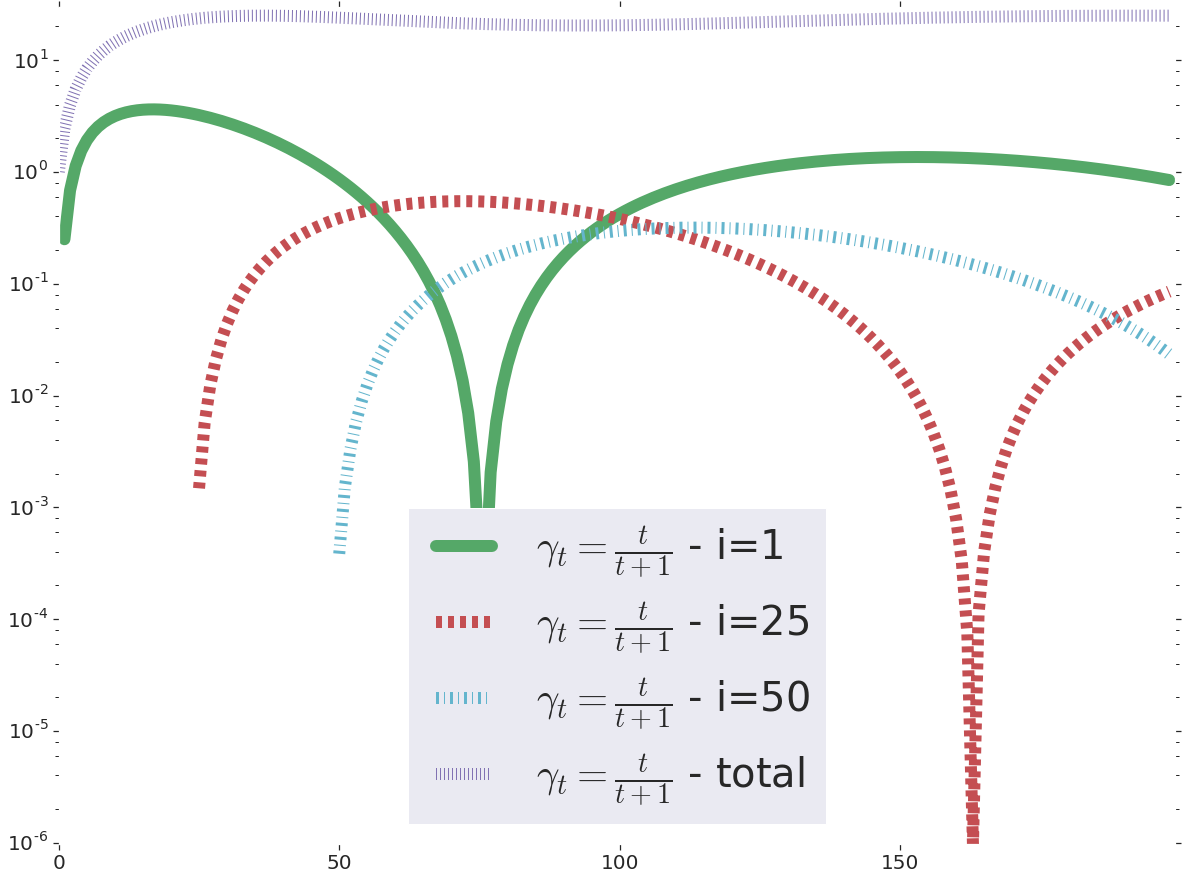}
    \caption{Momentum - $\gamma_t = 1 - \frac{1}{t}$ \label{fig:N_gamma_mominc}}
\end{subfigure}
\hspace{15mm}
\begin{subfigure}{.42\textwidth}
\includegraphics[width=\textwidth]{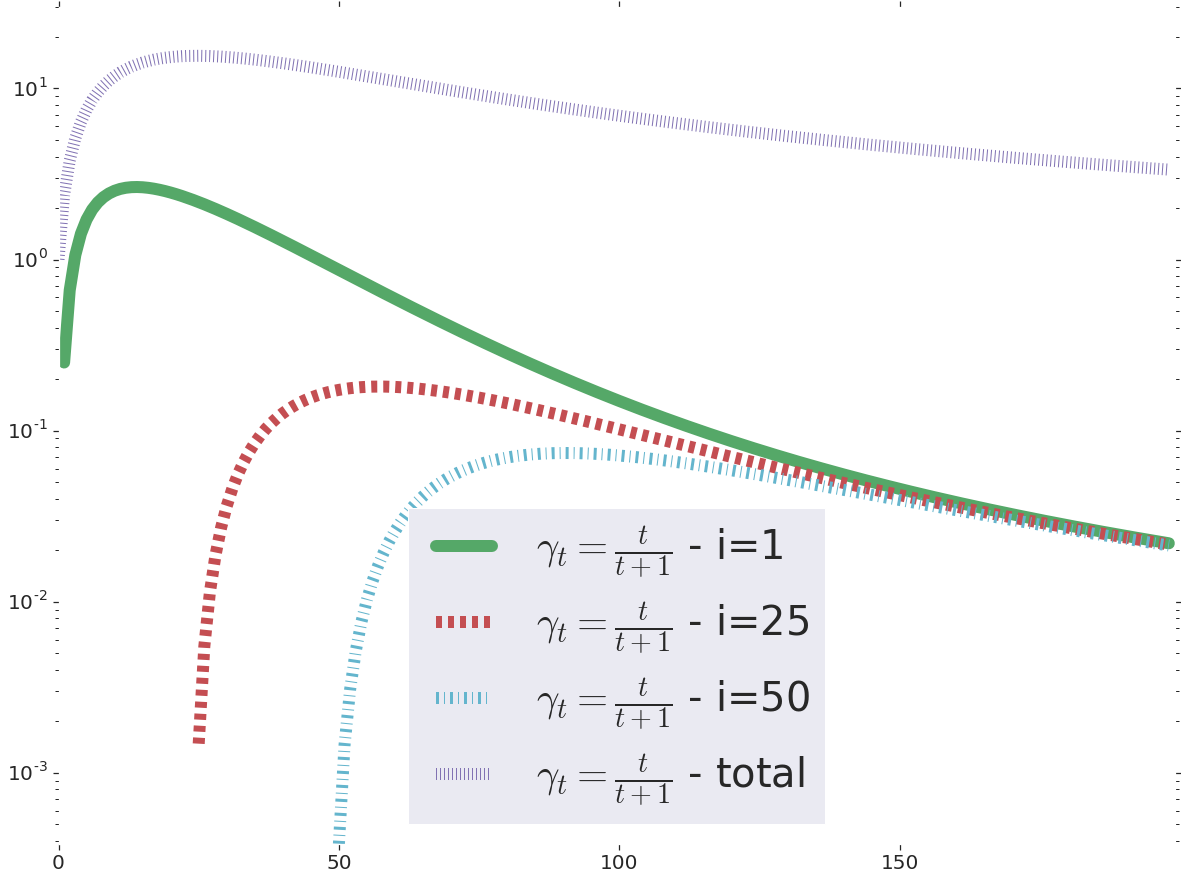}
\caption{Momentum - $\gamma_t = 1 - \frac{1}{t}$ with IGT.\label{fig:N_gamma_mominc_shift}}
\end{subfigure}
\end{center}
\caption{Variance induced over time by the noise from three different datapoints ($i=1$, $i=25$ and $i=50$) as well as the total variance for SG ($\gamma=0$, \emph{top left}), momentum with fixed $\gamma=0.9$ (\emph{top right}), momentum with increasing $\gamma_t = 1 - \frac{1}{t}$ without (\emph{bottom left}) and with (\emph{bottom right}) transport. The impact of the noise of each gradient $\epsilon_i$ increases for a few iterations then decreases. Although a larger $\gamma$ reduces the maximum impact of a given datapoint, the total variance does not decrease. With transport, noises are now equal and total variance decreases. The y-axis is on a log scale.} \label{fig:N-gamma}
\end{figure}
    
    At every timestep, the contribution to the noise of the 1st, the 25th and the 50th points in Fig.~\ref{fig:N-gamma} is unequal. If we assume that the $\epsilon_i$ are i.i.d., then the total variance would be minimal if the contribution from each point was equal. Further, one can notice that the impact of datapoint $i$ is only a function of $t-i$ and not of $t$. This guarantees that the total noise will not decrease over time.

    To address these two points, one can increase the momentum parameter over time. In doing so, the noise of new datapoints will have a decreasing impact on the total variance as their gradient is multiplied by $1 - \gamma_t$. Figure~\ref{fig:N_gamma_mominc} shows the impact $N_{i,t}^2$ of each noise $\epsilon_i$ for an increasing momentum $\gamma_t = 1 - \frac{1}{t}$. The peak of noise for $i=25$ is indeed lower than that of $i=1$. However, the variance still does not go to 0. This is because, as the momentum parameter increases, the update is an average of many gradients, including stale ones. Since these gradients were computed at iterates already influenced by the noise over previous datapoints, that past noise is amplified, as testified by the higher peak at $i=1$ for the increasing momentum. Ultimately, increasing momentum does not lead to a convergent algorithm in the presence of noise when using a constant stepsize.

    

    \subsection{SAG and Hessian modelling}
    The impact of the staleness of the gradients on the convergence is not limited to momentum. In SAG, for instance, the excess error after $k$ updates is proportional to $\left(1 - \min\left\{\frac{1}{16\widehat{\kappa}}, \frac{1}{8N}\right\}\right)^k$, compared to the excess error of the full gradient method which is $\left(1 - \frac{1}{\kappa}\right)^k$ where $\kappa$ is the condition number of the problem.~\footnote{The $\widehat{\kappa}$ in the convergence rate of SAG is generally larger than the $\kappa$ in the full gradient algorithm.} The difference between the two rates is larger when the minimum in the SAG rate is the second term. This happens either when $\widehat{\kappa}$ is small, i.e. the problem is well conditioned and a lot of progress is made at each step, or when $N$ is large, i.e. there are many points to the training set. Both cases imply that a large distance has been travalled between two draws of the same datapoint.

    Recent works showed that correcting for that staleness by modelling the Hessian~\cite{wai2017curvature,gower2017tracking} leads to improved convergence. As momentum uses stale gradients, the velocity is an average of current and past gradients and thus can be seen as an estimate of the true gradient at a point which is not the current one but rather a convex combination of past iterates. As past iterates depend on the noise of previous gradients, this bias in the gradients amplifies the noise and leads to a non-converging algorithm. We shall thus ``transport'' the old stochastic gradients  $g(\theta_i, x_i)$ to make them closer to their corresponding value at the current iterate, $g(\theta_t, x_i)$. Past works did so using the Hessian or an explicit approximation thereof, which can be expensive and difficult to compute and maintain. We will resort to using \textit{implicit transport}, a new method that aims at compensating the staleness of past gradients without making explicit use of the Hessian.

    \section{Converging optimization through implicit gradient transport}
    \label{sec:implicit_transport}
    Before showing how to combine the advantages of both increasing momentum and gradient transport, we demonstrate how to transport gradients implicitly. This transport is only exact under a strong assumption that will not hold in practice. However, this result will serve to convey the intuition behind implicit gradient transport. We will show in Section~\ref{sec:ATA} how to mitigate the effect of the unsatisfied assumption.

    \subsection{Implicit gradient transport}
    Let us assume that we received samples $x_0, \ldots, x_t$ in an online fashion. We wish to approach the full gradient $g_t(\theta_t)= \frac{1}{t+1} \sum_{i=0}^t g(\theta_t, x_i)$ as accurately as possible. We also assume here that a) We have a noisy estimate $\widehat{g}_{t-1}(\theta_{t-1})$ of $g_{t-1}(\theta_{t-1})$; b) We can compute the gradient $g(\theta, x_t)$ at any location $\theta$. We shall seek a $\theta$ such that
    \begin{align*}
        \frac{t}{t+1}\widehat{g}_{t-1}(\theta_{t-1}) + \frac{1}{t+1} g(\theta, x_t) \approx g_t(\theta_t) \; .
    \end{align*}
    To this end, we shall make the following assumption:
    \begin{assumption}
        \label{ass:quad}
        All individual functions $f(\cdot, x)$ are quadratics with the same Hessian $H$.
    \end{assumption}
    This is the same assumption as~\cite[Section 4.1]{flammarion2015averaging}. Although it is unlikely to hold in practice, we shall see that our method still performs well when that assumption is violated.

    Under Assumption~\ref{ass:quad}, we then have (see details in Appendix)
    \begin{align*}
        g_t(\theta_t)		&= \frac{t}{t+1} g_{t-1}(\theta_{t}) + \frac{1}{t+1} g(\theta_t, x_t)\\
        &\approx \frac{t}{t+1} \widehat{g}_{t-1}(\theta_{t-1}) + \frac{1}{t+1} g(\theta_t + t (\theta_t - \theta_{t-1}), x_t) \; .
    \end{align*}
    
    Thus, we can transport our current estimate of the gradient by computing the gradient on the new point at a shifted location $\theta = \theta_t + t (\theta_t - \theta_{t-1})$. This extrapolation step is reminiscent of Nesterov's acceleration with the difference that the factor in front of $\theta_t - \theta_{t-1}$, $t$, is not bounded.

    \subsection{Combining increasing momentum and implicit gradient transport}
    \label{sec:igt}
    We now describe our main algorithm, Implicit Gradient Transport (IGT). IGT uses an increasing momentum $\gamma_t = \frac{t}{t + 1}$. At each step, when updating the velocity, it computes the gradient of the new point at an extrapolated location so that the velocity $v_t$ is a good estimate of the true gradient $g(\theta_t)$. 


    We can rewrite the updates to eliminate the velocity $v_t$, leading to the update:
    \begin{align}
        \label{eq:igt_update}
        \tag{IGT}
        \theta_{t+1} &= \frac{2t + 1}{t+1} \theta_t - \frac{t}{t+1} \theta_{t-1}- \frac{\alpha}{t+1}  g\left(\theta_t + t (\theta_t - \theta_{t-1}), x_t\right) \; .
    \end{align}

    We see in Fig.~\ref{fig:N_gamma_mominc_shift} that IGT allows a reduction in the total variance, thus leading to convergence with a constant stepsize. This is captured by the following proposition:
    \begin{proposition}
        \label{prop:cvg}
        If $f$ is a quadratic function with positive definite Hessian $H$ with largest eigenvalue $L$ and condition number $\kappa$ and if the stochastic gradients satisfy: $g(\theta, x) = g(\theta) + \epsilon$ with $\epsilon$ a random i.i.d. noise with covariance bounded by $BI$, then Eq.~\ref{eq:igt_update} with stepsize $\alpha = 1 / L$ leads to iterates $\theta_t$ satisfying
        \begin{align*}
            E[\|\theta_t - \theta^\ast\|^2] &\leq \left(1 - \frac{1}{\kappa}\right)^{2t} \|\theta_0 - \theta^\ast\|^2 + \frac{d\alpha^2 B \bar{\nu}_0^2}{t} \; ,
        \end{align*}
        with $\nu = (2 + 2\log\kappa)\kappa$ for every $t > 2\kappa$.
    \end{proposition}
    The proof of Prop.~\ref{prop:cvg} is provided in the appendix.


    Despite this theoretical result, two limitations remain: First, Prop.~\ref{prop:cvg} shows that IGT does not improve the dependency on the conditioning of the problem; Second, the assumption of equal Hessians is unlikely to be true in practice, leading to an underestimation of the bias. We address the conditioning issue in the next section and the assumption on the Hessians in Section~\ref{sec:ATA}.

    \subsection{IGT as a plug-in gradient estimator}
    \label{sec:plug-in}

    We demonstrated that the IGT estimator has lower variance than the stochastic gradient estimator for quadratic objectives. IGT can also be used as a drop-in replacement for the stochastic gradient in an existing, popular first order method: the heavy ball (HB). This is captured by the following two propositions:
    \begin{proposition}[Non-stochastic]
    \label{prop:hb-igt-nonstoch}
    In the non-stochastic case, where $B=0$, variance is equal to $0$ and Heavyball-IGT achieves the accelerated linear rate $O\big(\left(\frac{\sqrt{\kappa}-1}{\sqrt{\kappa}+1}\right)^t\big)$ using the known, optimal heavy ball tuning, $\mu=\left( \frac{\sqrt{\kappa}-1}{\sqrt{\kappa}+1} \right)^2$, 
    $\alpha = (1+\sqrt{\mu})^2/L$.
    \end{proposition}
    \begin{proposition}[Online, stochastic]
    \label{prop:hb-igt-stoch}
    When $B>0$, there exist constant hyperparameters $\alpha>0$, $\mu>0$ such that $\|E[\theta_t - \theta^\ast]\|^2$ converges to zero linearly, and the variance is $\tilde{O}(1/t)$.
    \end{proposition}
    
    The pseudo-code can be found in Algorithm~\ref{alg:heavyball-igt}.
    
    \begin{algorithm}
    \caption{Heavyball-IGT}
    \label{alg:heavyball-igt}
    \begin{algorithmic}[1] 
        \Procedure{Heavyball-IGT}{Stepsize $\alpha$, Momentum $\mu$, Initial parameters $\theta_0$}
        \State $v_0 \gets g(\theta_0, x_0) \quad , \quad w_0 \gets - \alpha v_0 \quad, \quad \theta_1 \gets \theta_0 + w_0$
        \For{$t=1, \ldots, T-1$}
        	\State $\gamma_t \gets \frac{t}{t+1}$
        	\State $v_t \gets \gamma_t v_{t-1} + (1 - \gamma_t) g\left(\theta_t + \frac{\gamma_t}{1 - \gamma_t} (\theta_t - \theta_{t-1}), x_t\right)$
	\State $w_t \gets \mu w_{t-1} - \alpha v_t$
	\State $\theta_{t+1} \gets \theta_t + w_t$
        \EndFor\label{hbigtendfor}
        \State \textbf{return} $\theta_T$
        \EndProcedure
    \end{algorithmic}
    \end{algorithm}
    
    \section{IGT and Anytime Tail Averaging}
    \label{sec:ATA}
    So far, IGT weighs all gradients equally. This is because, with equal Hessians, one can perfectly transport these gradients irrespective of the distance travelled since they were computed. In practice, the individual Hessians are not equal and might change over time. In that setting, the transport induces an error which grows with the distance travelled. We wish to average a linearly increasing number of gradients, to maintain the $O(1/t)$ rate on the variance, while forgetting about the oldest gradients to decrease the bias. To this end, we shall use \emph{anytime tail averaging}~\cite{leroux2019anytime}, named in reference to the tail averaging technique used in optimization~\cite{JMLR:v18:16-595}.

    Tail averaging is an online averaging technique where only the last points, usually a constant fraction $c$ of the total number of points seen, is kept. Maintaining the exact average at every timestep is memory inefficient and anytime tail averaging performs an approximate averaging using $\gamma_t    = \frac{c(t-1)}{1 + c(t-1)}\left(1 - \frac{1}{c}\sqrt{\frac{1-c}{t(t-1)}}\right)$. We refer the reader to~\cite{leroux2019anytime} for additional details.

\section{Impact of IGT on bias and variance in the ideal case}
\label{sec:igt_ideal}
To understand the behaviour of IGT when Assumption~\ref{ass:quad} is verified, we minimize a strongly convex quadratic function with Hessian $Q \in \R^{100\times 100}$ with condition number $1000$, and we have access to the gradient corrupted by noise $\epsilon_t$, where $\epsilon_t \sim N(0, 0.3 \cdot I_{100})$. In that scenario where all Hessians are equal and implicit gradient transport is exact, Fig.~\ref{fig:quad} confirms the $O(1/t)$ rate of IGT with constant stepsize while SGD and HB only converge to a ball around the optimum.

To further understand the impact of IGT, we study the quality of the gradient estimate. Standard stochastic methods control the variance of the parameter update by scaling it with a decreasing stepsize, which slows the optimization down. With IGT, we hope to have a low variance while maintaining a norm of the update comparable to that obtained with gradient descent. To validate the quality of our estimator, we optimized a quadratic function using IGT, collecting iterates $\theta_t$. For each iterate, we computed the squared error between the true gradient and either the stochastic or the IGT gradient. In this case where both estimators are unbiased, this is the trace of the noise covariance of our estimators. The results in Figure~\ref{fig:noise} show that, as expected, this noise decreases linearly for IGT and is constant for SGD.

We also analyse the direction and magnitude of the gradient of IGT on the same quadratic setup.
Figure~\ref{fig:cosine} displays the cosine similarity between the true gradient and either the stochastic or the IGT gradient, as a function of the distance to the optimum. We see that, for the same distance, the IGT gradient is much more aligned with the true gradient than the stochastic gradient is, confirming that variance reduction happens without the need for scaling the estimate.

\begin{figure}
    \centering
    \begin{subfigure}[b]{0.32\textwidth}
        \includegraphics[width=\textwidth]{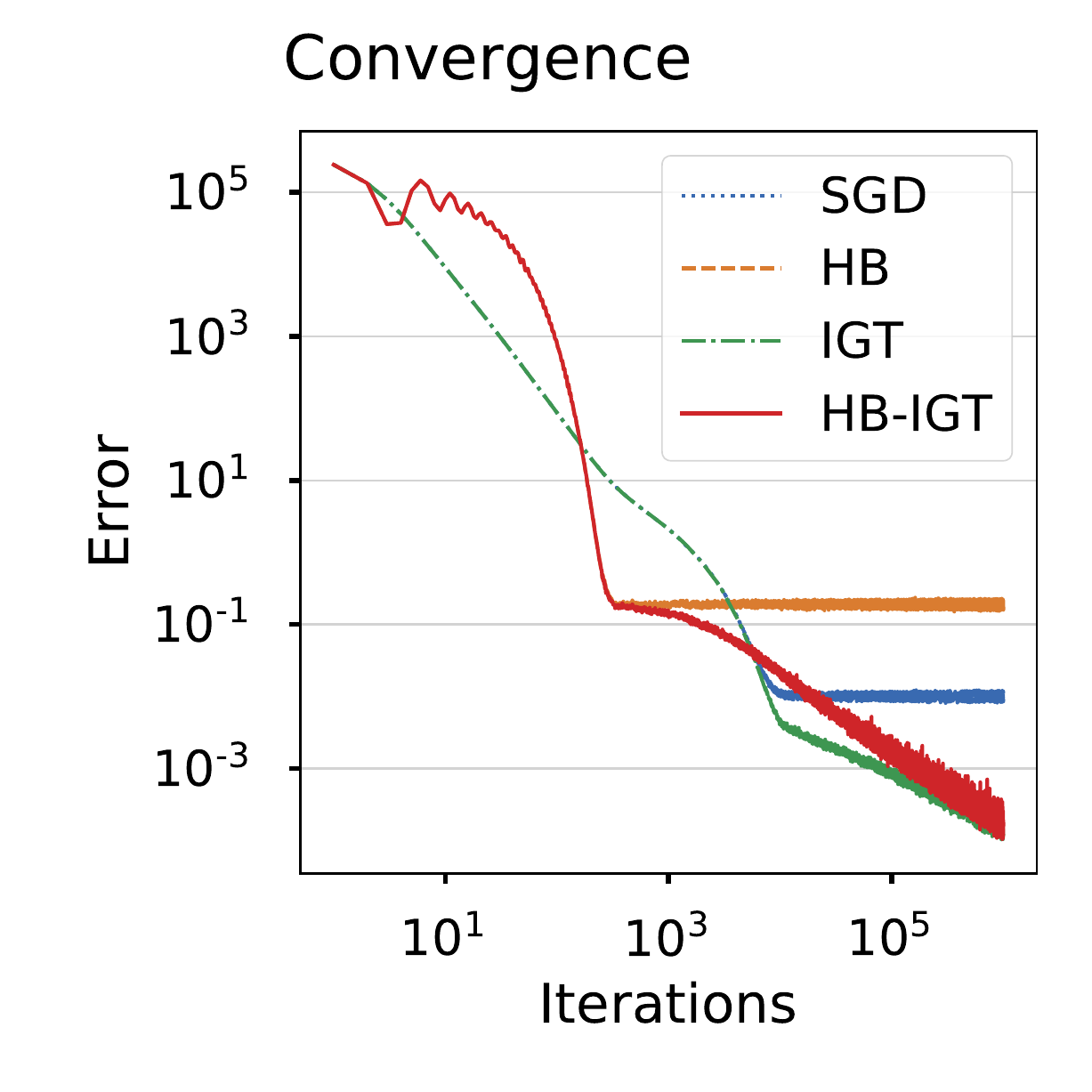}
        \caption{\label{fig:quad}}
    \end{subfigure}
    \begin{subfigure}[b]{0.32\textwidth}
        \includegraphics[width=\textwidth]{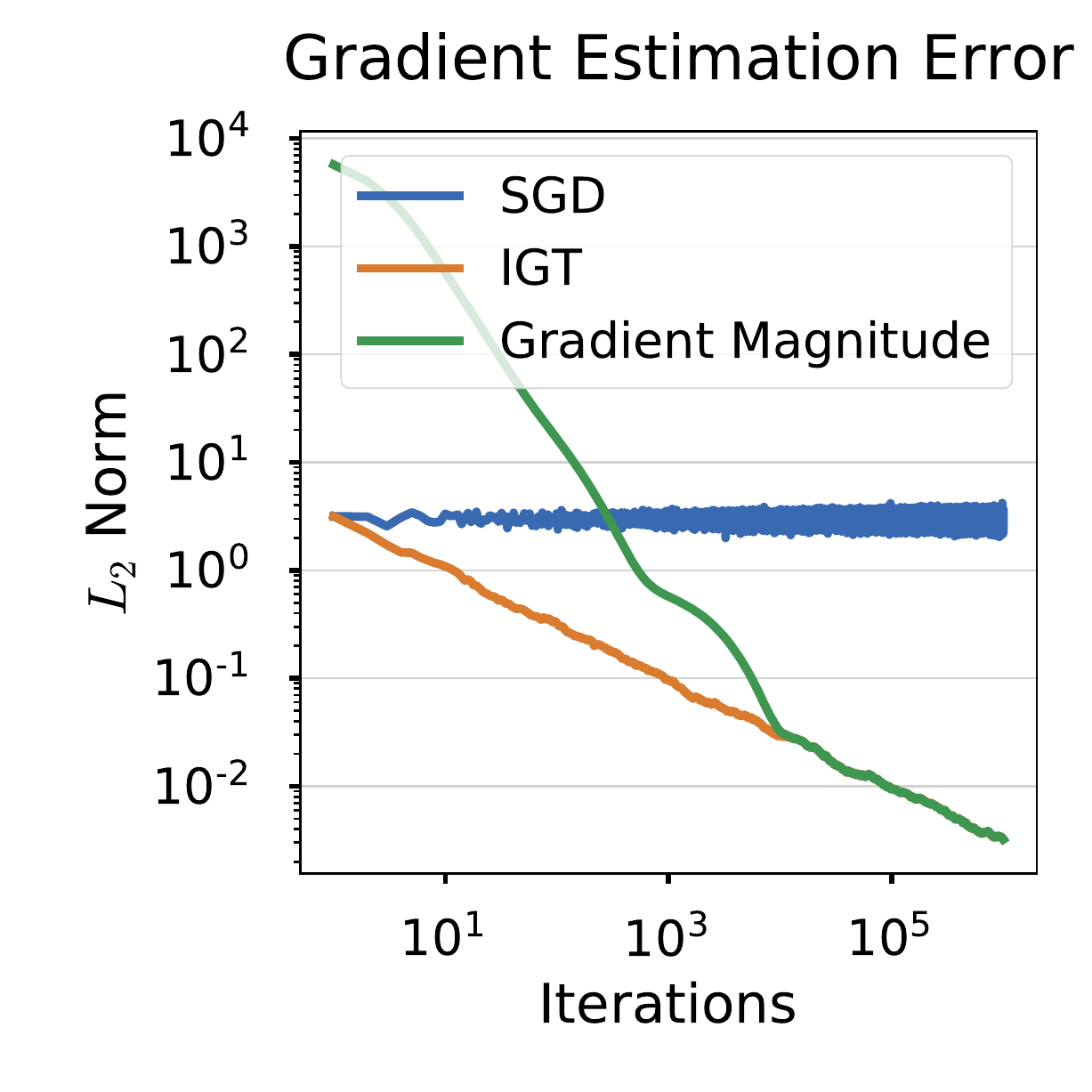}
        \caption{\label{fig:noise}}
    \end{subfigure}
    \begin{subfigure}[b]{0.32\textwidth}
        \includegraphics[width=\textwidth]{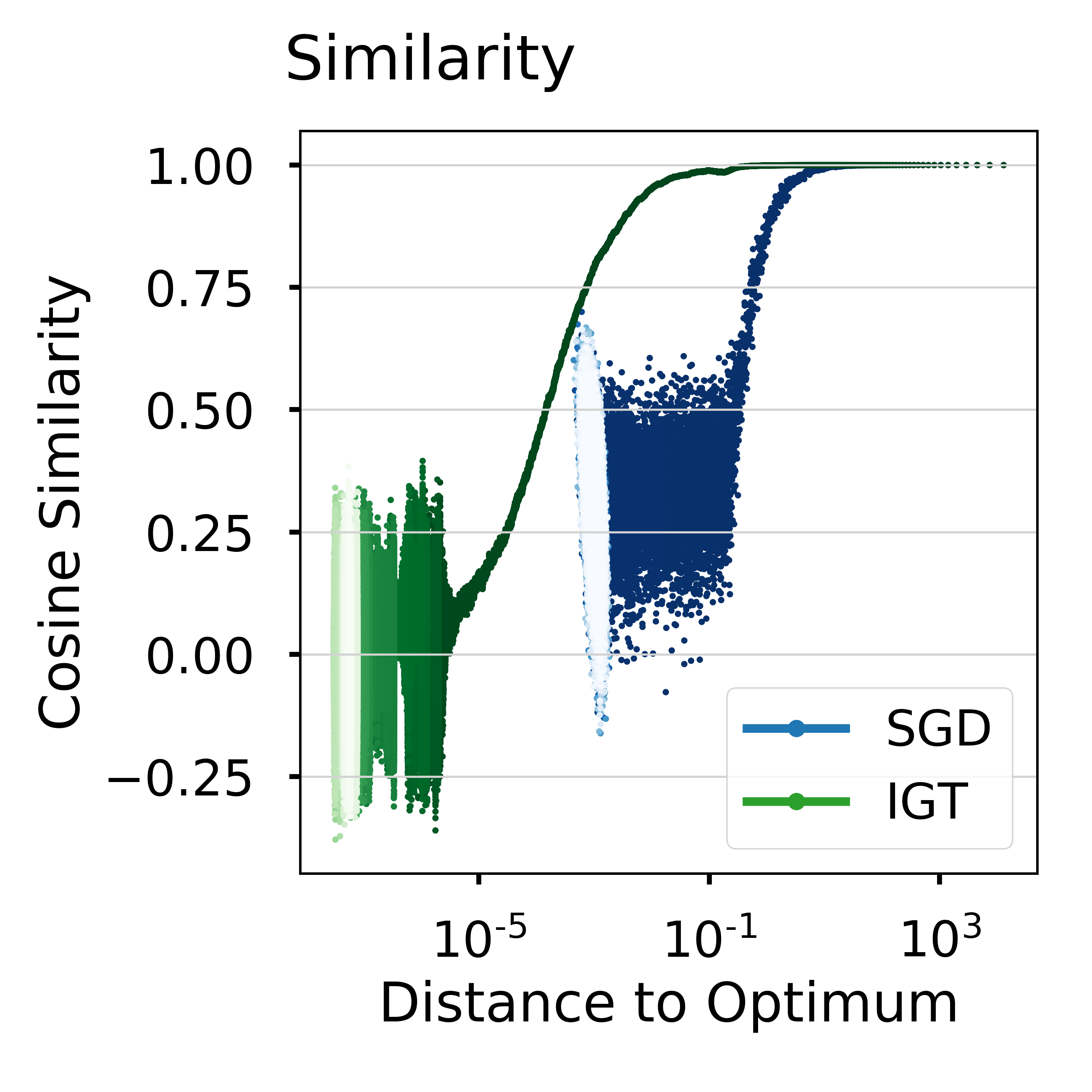} 
        \caption{\label{fig:cosine}}
    \end{subfigure}

    \caption{
        Analysis of IGT on quadratic loss functions.
        (\subref{fig:quad}) Comparison of convergence curves for multiple algorithms.
        As expected, the IGT family of algorithms converges to the solution while stochastic gradient algorithms can not.
        (\subref{fig:noise}) The blue and orange curves show the norm of the noise component in the SGD and IGT gradient estimates, respectively.
        The noise component of SGD remains constant, while it decreases at a rate $1/\sqrt{t}$ for IGT.
        The green curve shows the norm of the IGT gradient estimate.
        (\subref{fig:cosine}) Cosine similarity between the full gradient and the SGD/IGT estimates.
        \label{fig:quadratic}}
\end{figure}

\section{Experiments}
\label{sec:experiments}

While Section~\ref{sec:igt_ideal} confirms the performance of IGT in the ideal case, the assumption of identical Hessians almost never holds in practice. In this section, we present results on more realistic and larger scale machine learning settings.
All experiments are extensively described in the Appendix \ref{sec:exp-details} and additional baselines compared in Appendix \ref{sec:additional-experiments}.

\subsection{Supervised learning}

\paragraph{CIFAR10 image classification}
We first consider the task of training a ResNet-56 model~\cite{DBLP:journals/corr/HeZRS15} on the CIFAR-10 image classification dataset~\cite{cifar10}. We use TF official models code and setup~\cite{TfResnet}, varying only the optimizer: SGD, HB, Adam and our algorithm with anytime tail averaging both on its own (ITA) and combined with Heavy Ball (HB-ITA). We tuned the step size for each algorithm by running experiments using a logarithmic grid. To factor in ease of tuning~\cite{wilson2017marginal}, we used Adam's default parameter values and a value of 0.9 for HB's parameter. We used a linearly decreasing stepsize as it was shown to be simple and perform well~\cite{shallue2018measuring}. For each optimizer we selected the hyperparameter combination that is fastest to reach a consistently attainable target train loss~\cite{shallue2018measuring}. Selecting the
hyperparameter combination reaching the lowest training loss yields qualitatively identical curves.
Figure~\ref{fig:cifar10-train-test} presents the results, showing that IGT with the exponential anytime tail average
performs favourably, both on its own and combined with Heavy Ball: the learning curves show faster improvement
and are much less noisy.

\begin{figure}
\begin{center}
\begin{subfigure}{.48\textwidth}
    \includegraphics[width=\textwidth]{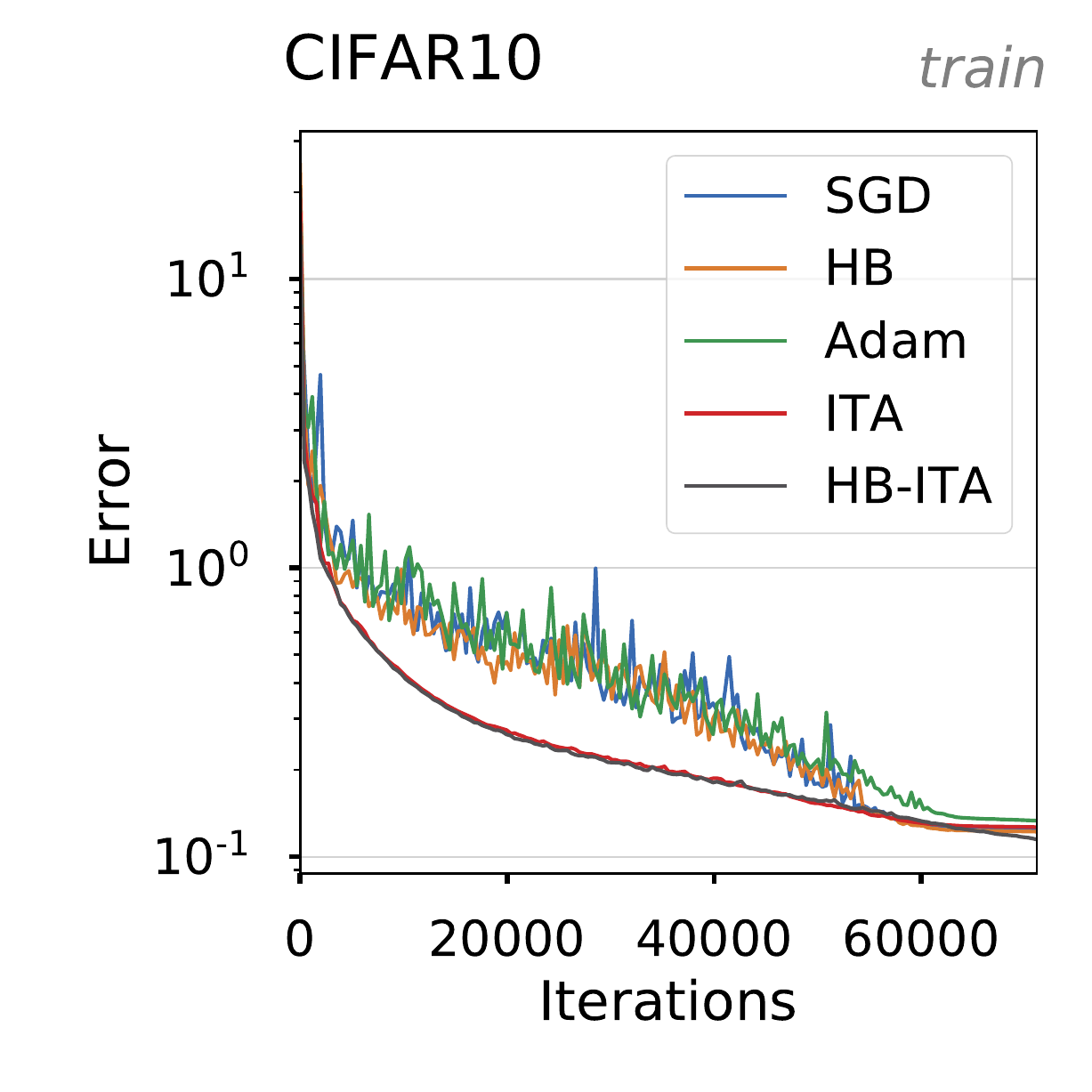}
\end{subfigure}
\begin{subfigure}{.48\textwidth}
    \includegraphics[width=\textwidth]{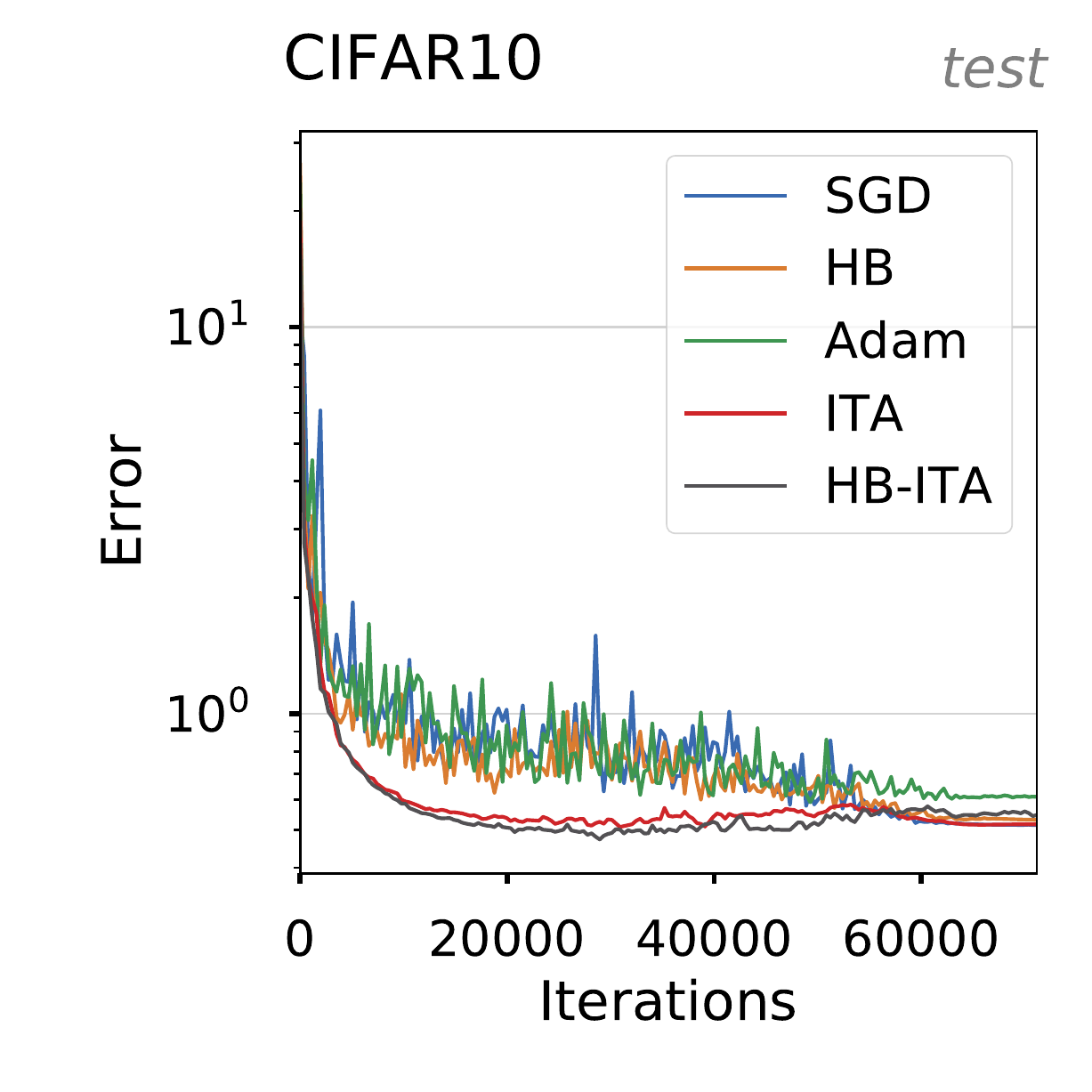}
\end{subfigure}
\\
\begin{subfigure}{.48\textwidth}
    \includegraphics[width=\textwidth]{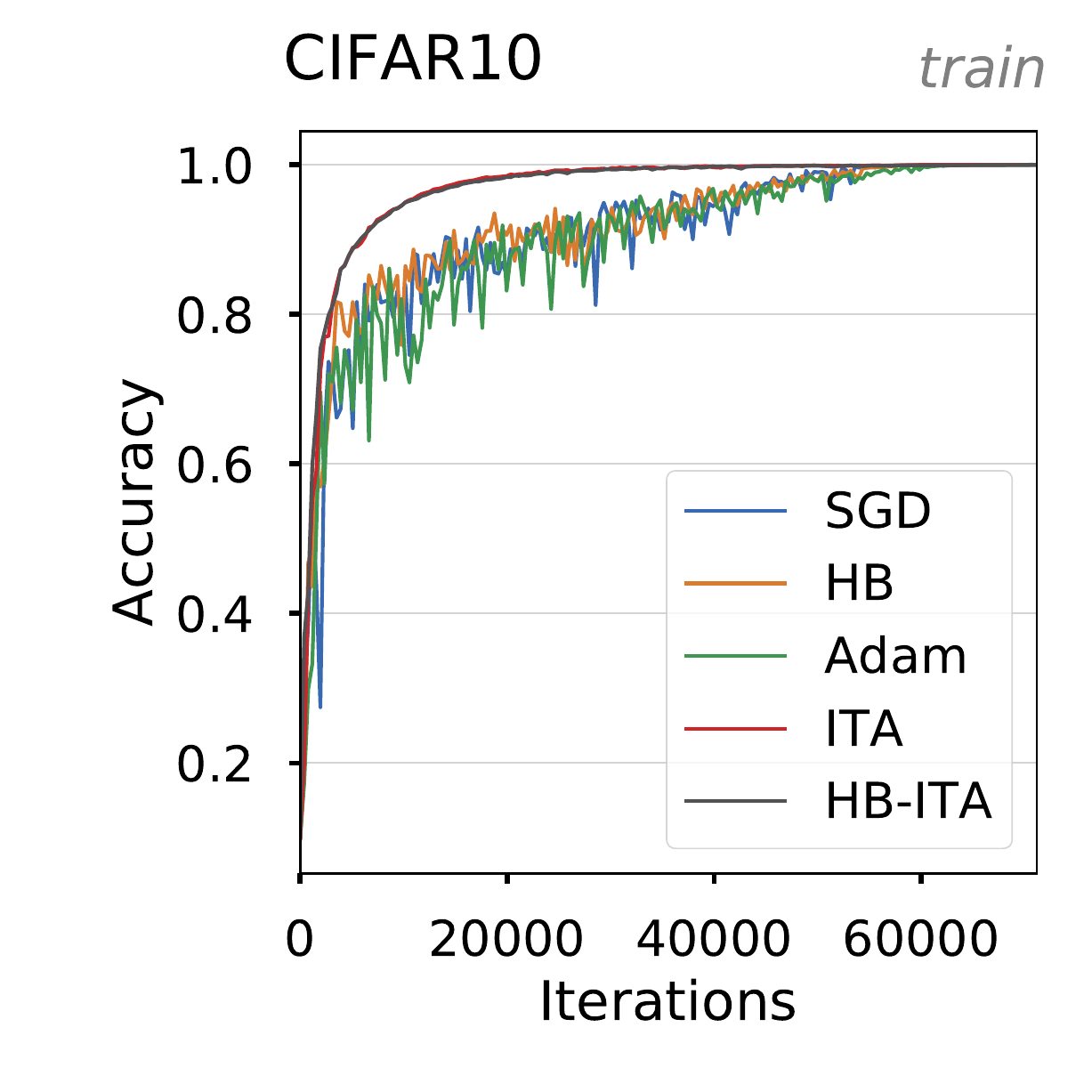}
\end{subfigure}
\begin{subfigure}{.48\textwidth}
    \includegraphics[width=\textwidth]{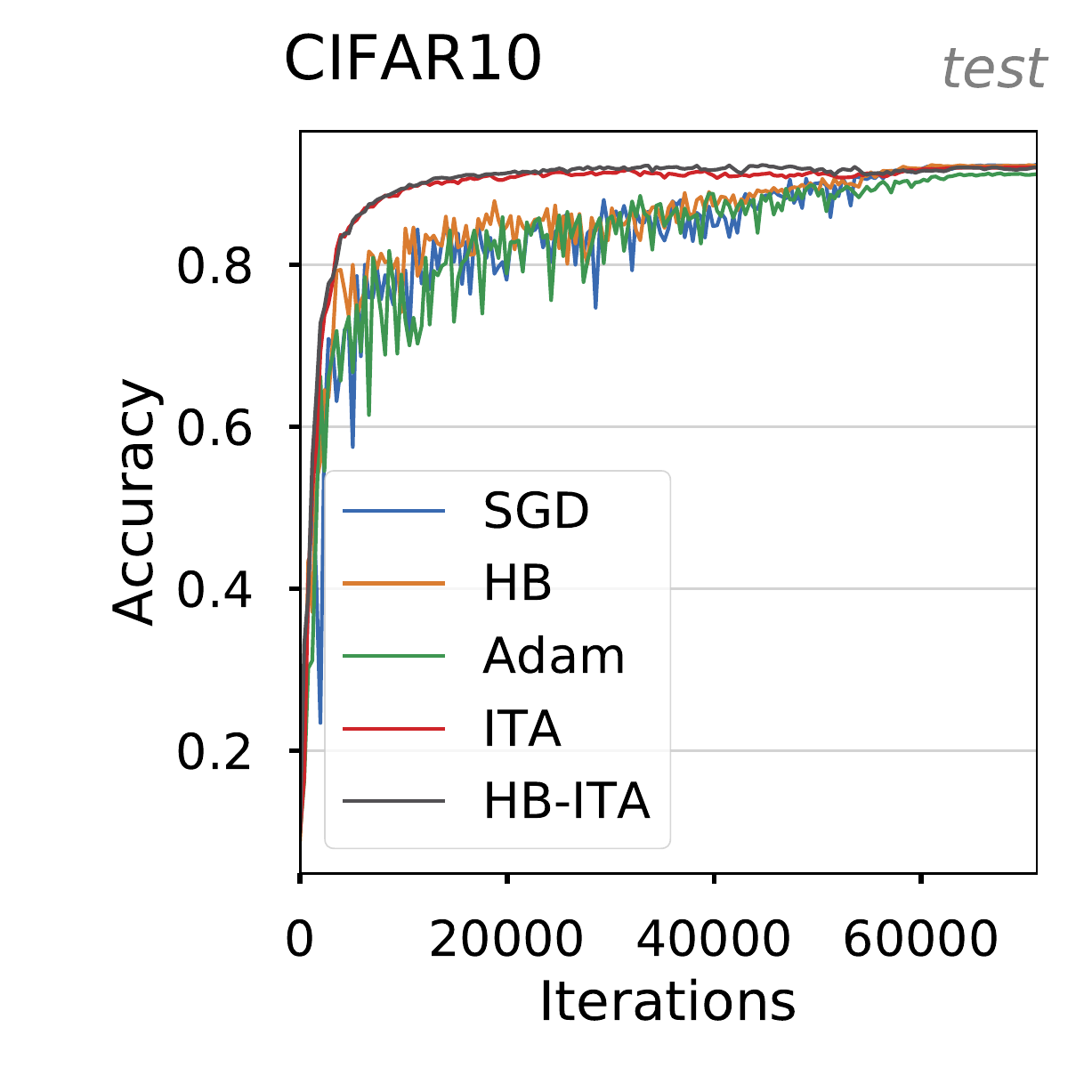}
\end{subfigure}
\end{center}
\caption{Resnet-56 on CIFAR10. \textbf{Left}: Train loss. \textbf{Center}: Train accuracy.
\textbf{Right}: Test accuracy.}  \label{fig:cifar10-train-test}
\end{figure}

\paragraph{ImageNet image classification}
We also consider the task of training a ResNet-50 model\cite{DBLP:journals/corr/HeZRS15} on the larger ImageNet dataset~\cite{ILSVRC15}. The setup is similar to the one used for CIFAR10 with the difference that we trained using larger minibatches (1024 instead of 128). In Figure~\ref{fig:imagenet-train-test}, one can see that IGT is as fast as Adam for the train loss, faster for the train accuracy and reaches the same final performance, which Adam does not. We do not see the noise reduction we observed with CIFAR10, which could be explained by the larger batch size (see Appendix~\ref{sec:appendix_cifar}).

\begin{figure}
\begin{center}
\begin{subfigure}{.48\textwidth}
    \includegraphics[width=\textwidth]{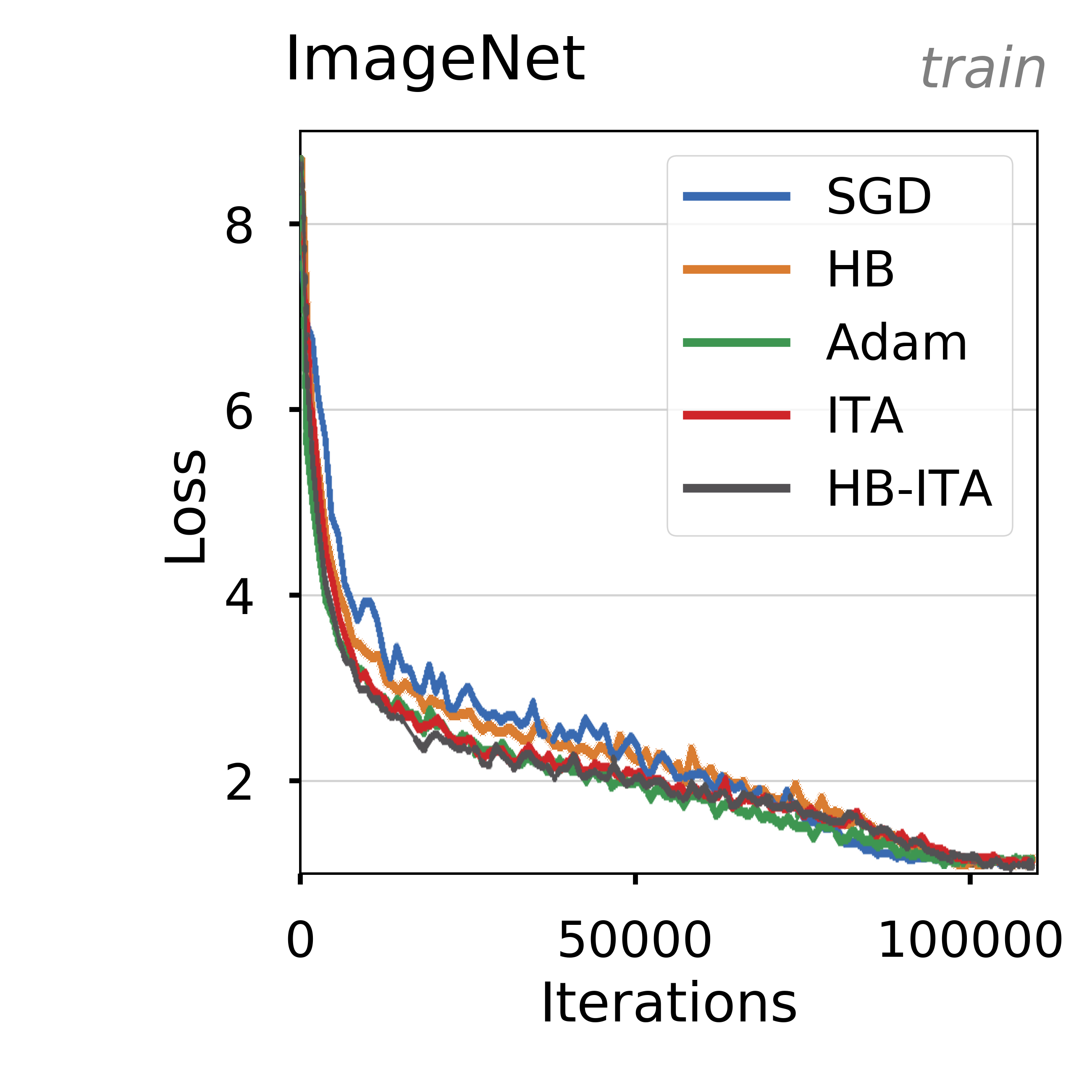}
\end{subfigure}
\begin{subfigure}{.48\textwidth}
    \includegraphics[width=\textwidth]{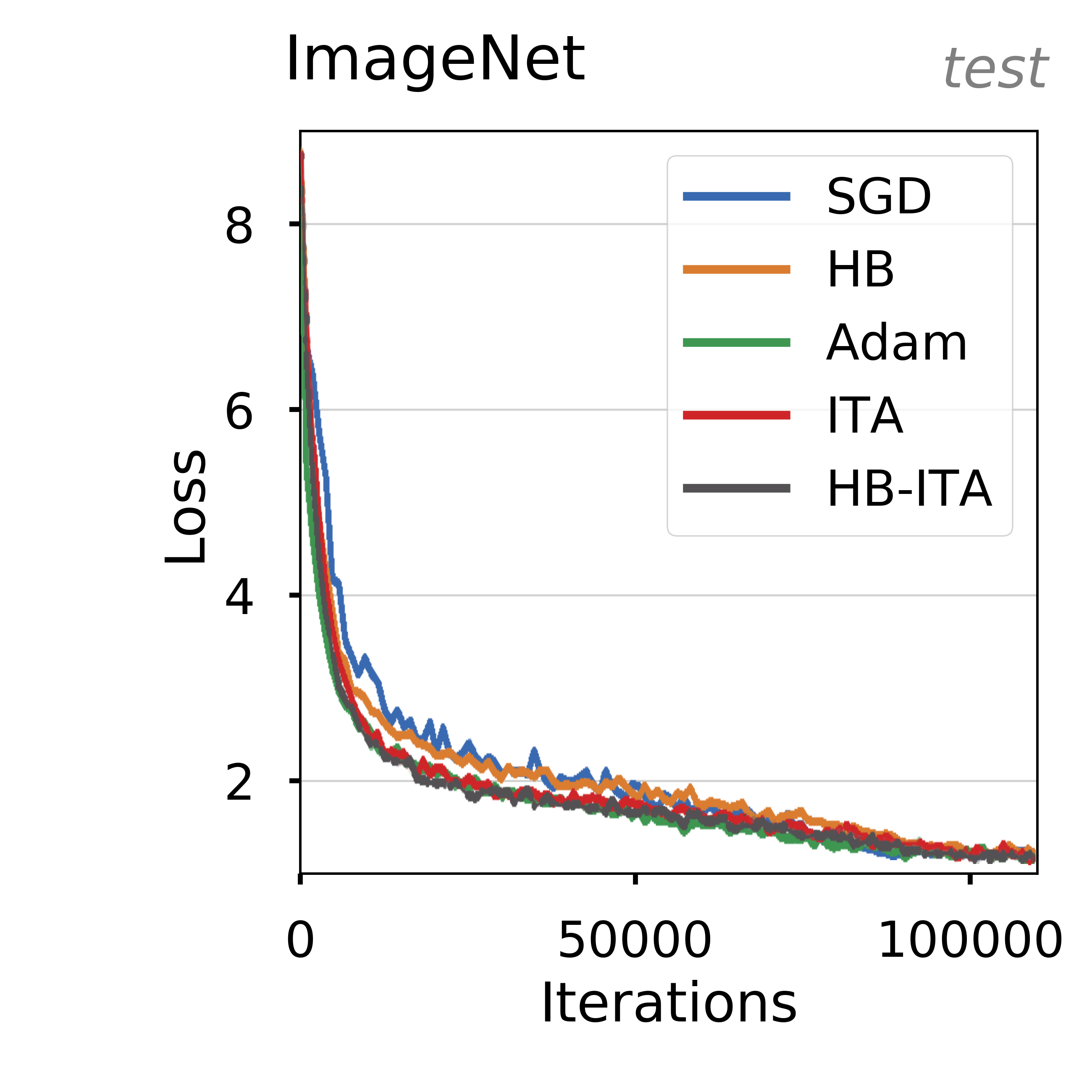}
\end{subfigure}
\\
\begin{subfigure}{.48\textwidth}
    \includegraphics[width=\textwidth]{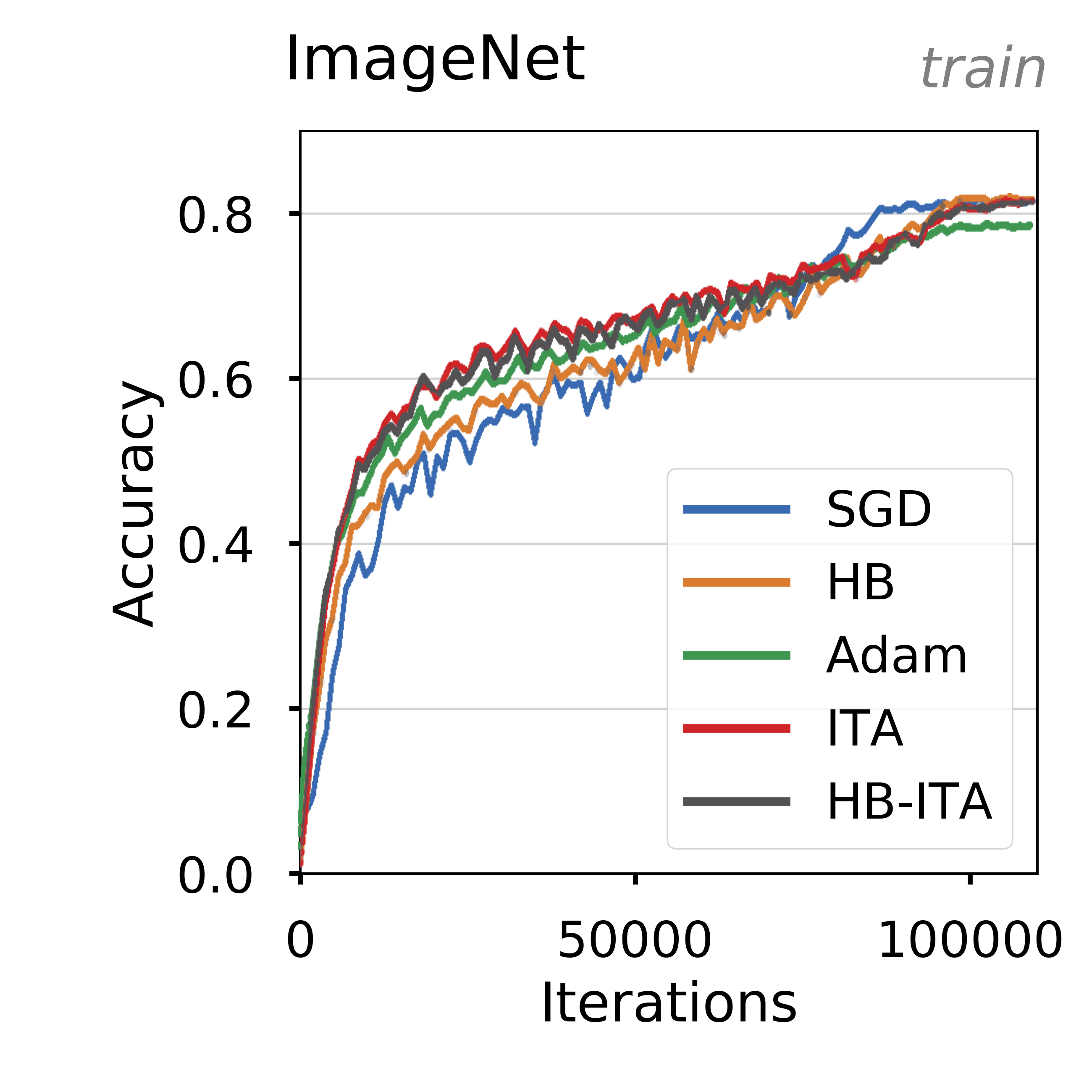}
\end{subfigure}
\begin{subfigure}{.48\textwidth}
    \includegraphics[width=\textwidth]{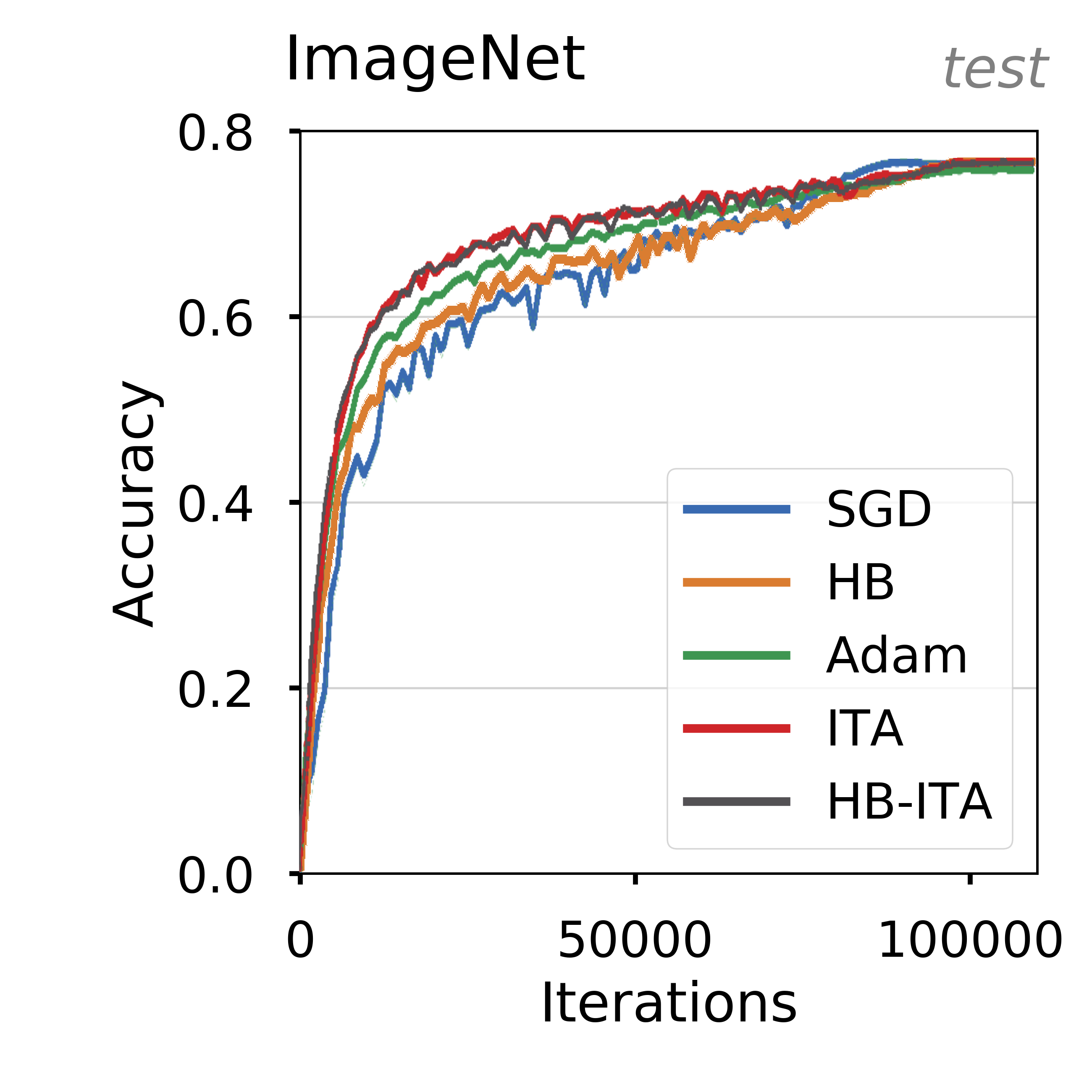}
\end{subfigure}
\end{center}
\caption{ResNet-50 on ImageNet. \textbf{Left}: Train loss. \textbf{Center}: Train accuracy.
\textbf{Right}: Test accuracy.}  \label{fig:imagenet-train-test}
\end{figure}

\paragraph{IMDb sentiment analysis}
We train a bi-directional LSTM on the IMDb Large Movie Review Dataset for 200 epochs. \cite{maas2011learning}
We observe that while the training convergence is comparable to HB, HB-ITA performs better in terms of validation and test accuracy.
In addition to the baseline and IGT methods, we also train a variant of Adam using the ITA gradients, dubbed \textbf{Adam-ITA}, which performs similarly to Adam.

\begin{figure}
    \centering
    \includegraphics[width=0.48\textwidth]{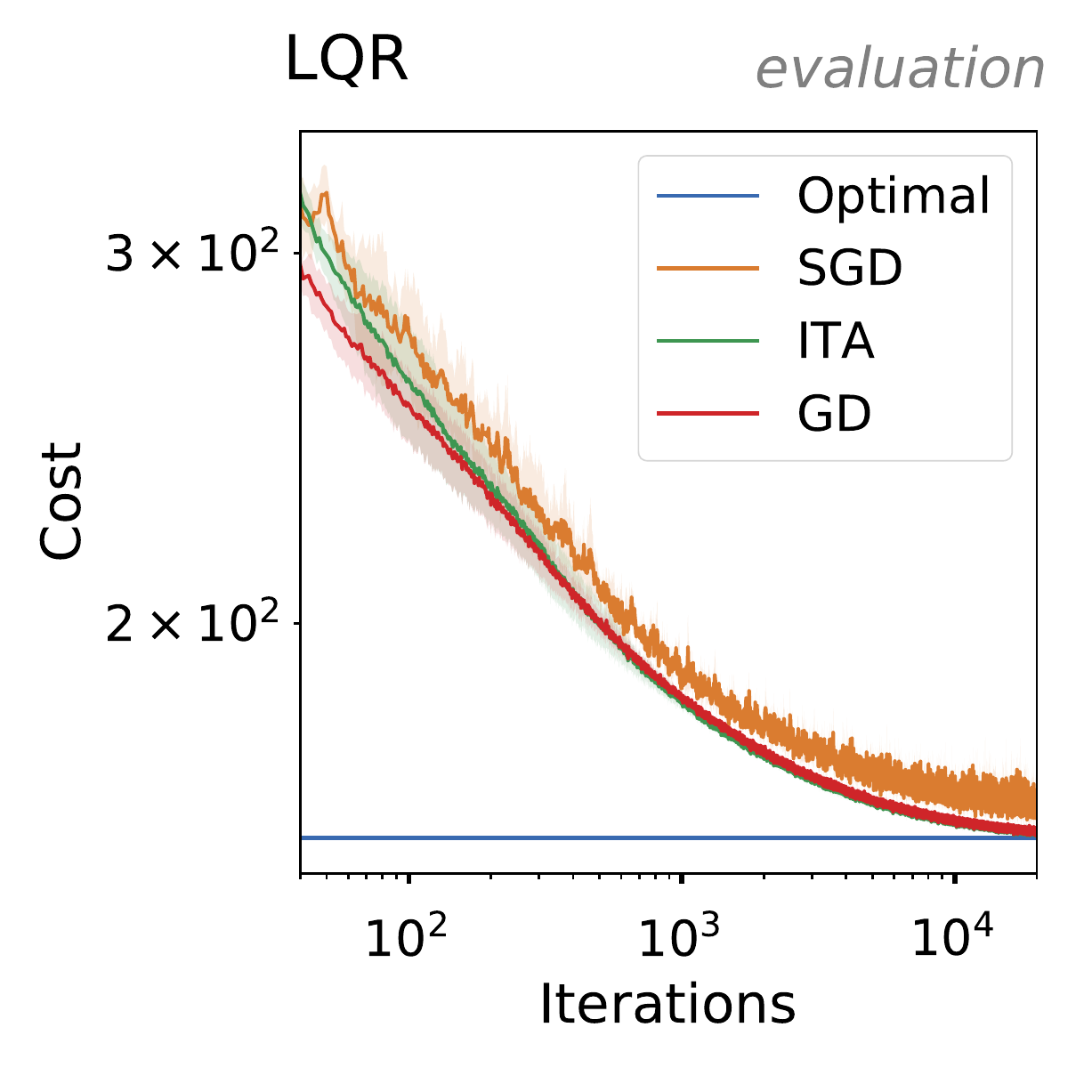}
    \includegraphics[width=0.48\textwidth]{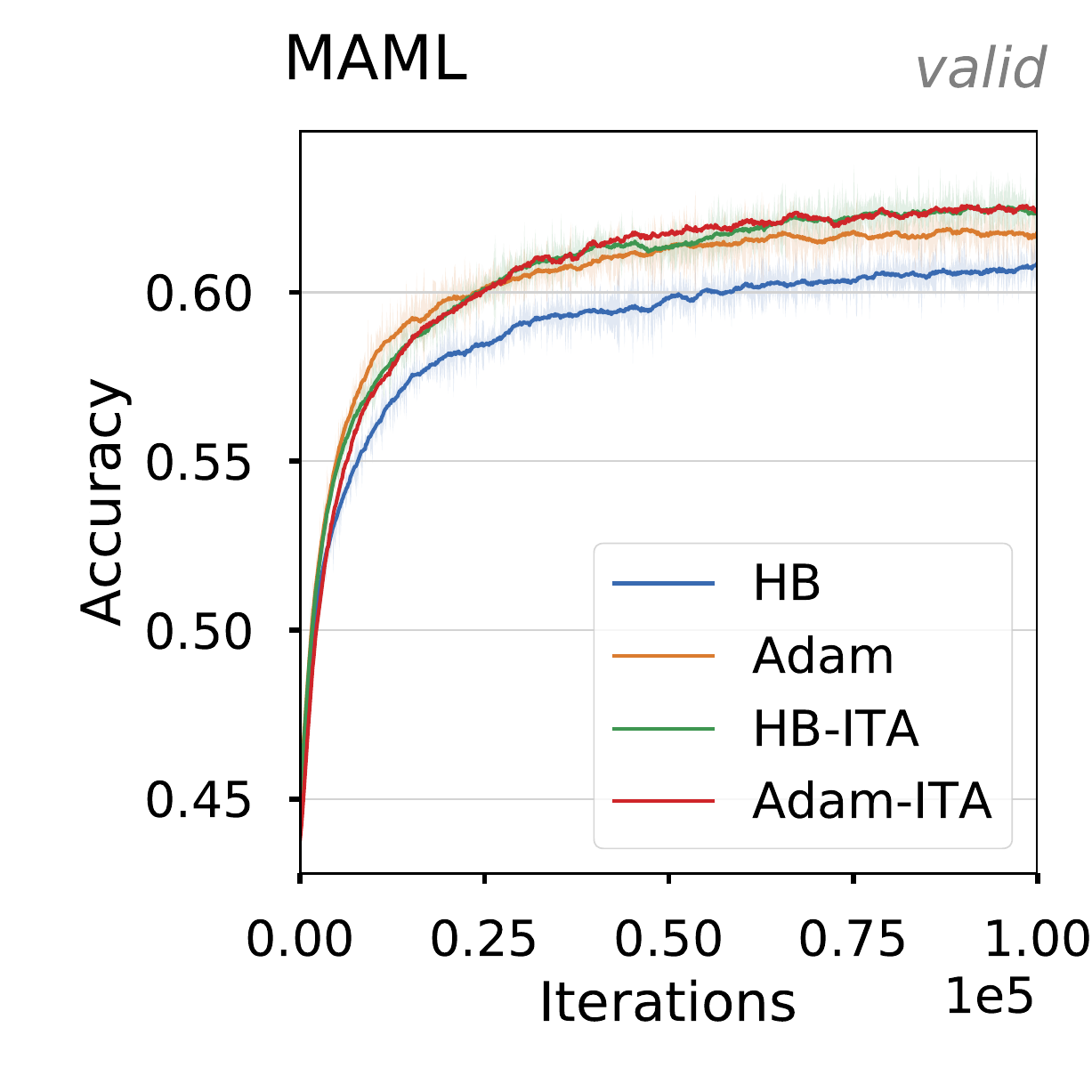}
    \caption{
        Validation curves for different large-scale machine learning settings.
        Shading indicates one standard deviation computed over three random seeds.
        \textbf{Left}: Reinforcement learning via policy gradient on a LQR system.
        \textbf{Right}: Meta-learning using MAML on Mini-Imagenet.
        \label{fig:largescale}
    }
\end{figure}

\subsection{Reinforcement learning}


\paragraph{Linear-quadratic regulator}
We cast the classical linear-quadratic regulator (LQR) \cite{kwakernaak1972linear} as a policy learning problem to be optimized via gradient descent.
This setting is extensively described in Appendix~\ref{sec:exp-details}.
Note that despite their simple linear dynamics and a quadratic cost functional, LQR systems are notoriously difficult to optimize due to the non-convexity of the loss landscape. \cite{fazel2018global}

The left chart in Figure \ref{fig:largescale} displays the evaluation cost computed along training and averaged over three random seeds.
The first method (\textbf{Optimal}) indicates the cost attained when solving the algebraic Riccati equation of the LQR -- this is the optimal solution of the problem.
\textbf{SGD} minimizes the costs using the REINFORCE \cite{williams1992simple} gradient estimator, averaged over 600 trajectories.
\textbf{ITA} is similar to SGD but uses the ITA gradient computed from the REINFORCE estimates.
Finally, \textbf{GD} uses the analytical gradient by taking the expectation over the policy.

We make two observations from the above chart.
First, ITA initially suffers from the stochastic gradient estimate but rapidly matches the performance of GD.
Notably, both of them converge to a solution significantly better than SGD, demonstrating the effectiveness of the variance reduction mechanism.
Second, the convergence curve is smoother for ITA than for SGD, indicating that the ITA iterates are more likely to induce similar policies from one iteration to the next.
This property is particularly desirable in reinforcement learning as demonstrated by the popularity of trust-region methods in large-scale applications. \cite{schulman2015trust,openai2018learning}

\subsection{Meta-learning}

\paragraph{Model-agnostic meta-learning}
We now investigate the use of IGT in the \textit{model-agnostic meta-learning} (MAML) setting. \cite{finn2017model}
We replicate the 5 ways classification setup with 5 adaptation steps on tasks from the Mini-Imagenet dataset \cite{ravi2016optimization}.
This setting is interesting because of the many sources contributing to noise in the gradient estimates: the stochastic meta-gradient depends on the product of 5 stochastic Hessians computed over only 10 data samples, and is averaged over only 4 tasks.
We substitute the meta-optimizer with each method, select the stepsize that maximizes the validation accuracy after 10K iterations, and use it to train the model for 100K iterations.

The right graph of Figure \ref{fig:largescale} compares validation accuracies for three random seeds.
We observe that methods from the IGT family significantly outperform their stochastic meta-gradient counter-part, both in terms of convergence rate and final accuracy.
Those results are also reflected in the final test accuracies where Adam-ITA ($65.16\%$) performs best, followed by HB-ITA ($64.57\%$), then Adam ($63.70\%$), and finally HB ($63.08\%$).

\section{Conclusion and open questions}
We proposed a simple optimizer which, by reusing past gradients and transporting them, offers excellent performance on a variety of problems. While it adds an additional parameter, the ratio of examples to be kept in the tail averaging, it remains competitive across a wide range of such values. Further, by providing a higher quality gradient estimate that can be plugged in any existing optimizer, we expect it to be applicable to a wide range of problems.
As the IGT is similar to momentum, this further raises the question on the links between variance reduction and curvature adaptation. Whether there is a way to combine the two without using momentum on top of IGT remains to be seen.

\subsubsection*{Acknowledgments}
The authors would like to thank Liyu Chen for his help with the LQR experiments and Fabian Pedregosa for insightful discussions.

\bibliographystyle{plain}
\bibliography{biblio}

\clearpage

\appendix

\section{Experimental Details}
\label{sec:exp-details}
This section provides additional information regarding the experiments included in the main text.

For each experimental setting we strive to follow the \textit{reproducibility checklist}, and provide:

\begin{itemize}
    \item a description and citation of the dataset,
    \item a description of pre-processing steps,
    \item training / validation / testing splits,
    \item a description of the hyper-parameter search process and chosen values for each method,
    \item the exact number of evaluation runs,
    \item a clear definition of statistics reported, and
    \item a description of the computing infrastructure.
\end{itemize}

\subsection{CIFAR10 image classification}
\label{sec:appendix_cifar}

\paragraph{Dataset}
The CIFAR10 dataset \cite{krizhevsky2009learning} consists 50k training and 10k testing images, partitioned over 10 classes.
We download and pre-process the images using the TensorFlow \texttt{models} package, available at the following URL: \url{https://github.com/tensorflow/models}

\paragraph{Model}
We use a residual convolutional network \cite{DBLP:journals/corr/HeZRS15} with 56 layers as defined in the \texttt{models} package.
Specifically, we use the second version whose blocks are built as a batch normalization, then a ReLU activation, and then a convolutional layer. \cite{he2016identity}

\paragraph{Hyper-parameters}
We use the exact setup from \url{https://github.com/tensorflow/models/officials/resnet}. As such, training is carried out with minibatches of 128 examples for 182 epochs and the training data is augmented with random crops and horizontal flips. Also note
this setup multiplies the step size by the size of the minibatch. One deviation from the setup is our use of a linearly 
decaying learning rate instead of an explicit schedule. The linearly decaying learning rate schedule is simple and
was shown to perform well \cite{shallue2018measuring}. This schedule is specified using two parameters: the decay rate,
a multiplier specifying the final step size (0.1 or 0.01), and the decay step, specifying the step at which the fully decayed rate is reached (always set to 90\% of the training steps). To factor in ease of tuning\cite{wilson2017marginal} we used Adam’s default parameter values and a value of 0.9 for HB’s parameter. We used IGT with the exponential Anytime Tail Averaging approach. For
the tail fraction, we tried two values: the number of epochs and a tenth of that number (180 and 18). We ran using the 
following learning rate: (1e0, 3e-1, 1e-1, 3e-2, 1e-2) for SGD, HB and the IGT variants and (1e-2, 3e-3, 1e-3, 3e-4, 1e-4)
for Adam. We ran a grid search over the base learning rate and its decay rate with a single run per combination.
For each optimizer we selected the hyperparameter combination that is fastest to reach a consistently attainable target train loss of 0.2 \cite{shallue2018measuring}. Note that selecting the hyperparameter combination reaching the lowest training loss yields qualitatively identical curves.

The resulting hyper-parameters are:
\begin{itemize}
	\item SGD stepsize 0.3, decay 0.01
	\item HB stepsize 0.03, decay 0.01
	\item Adam stepsize 0.001, decay 0.01
	\item ITA stepsize 0.3, decay 0.01, tail fraction 18
	\item HB-ITA stepsize 0.03, decay 0.1, tail fraction 18
\end{itemize}

\paragraph{Infrastructure and Runs}
The experiments were run using P100 GPUs (single GPU).

\paragraph{Additional Results}
We provide all learning curves for the methods comparison presented in the main manuscript in figure~\ref{fig:cifar10-comp}.

\begin{figure}
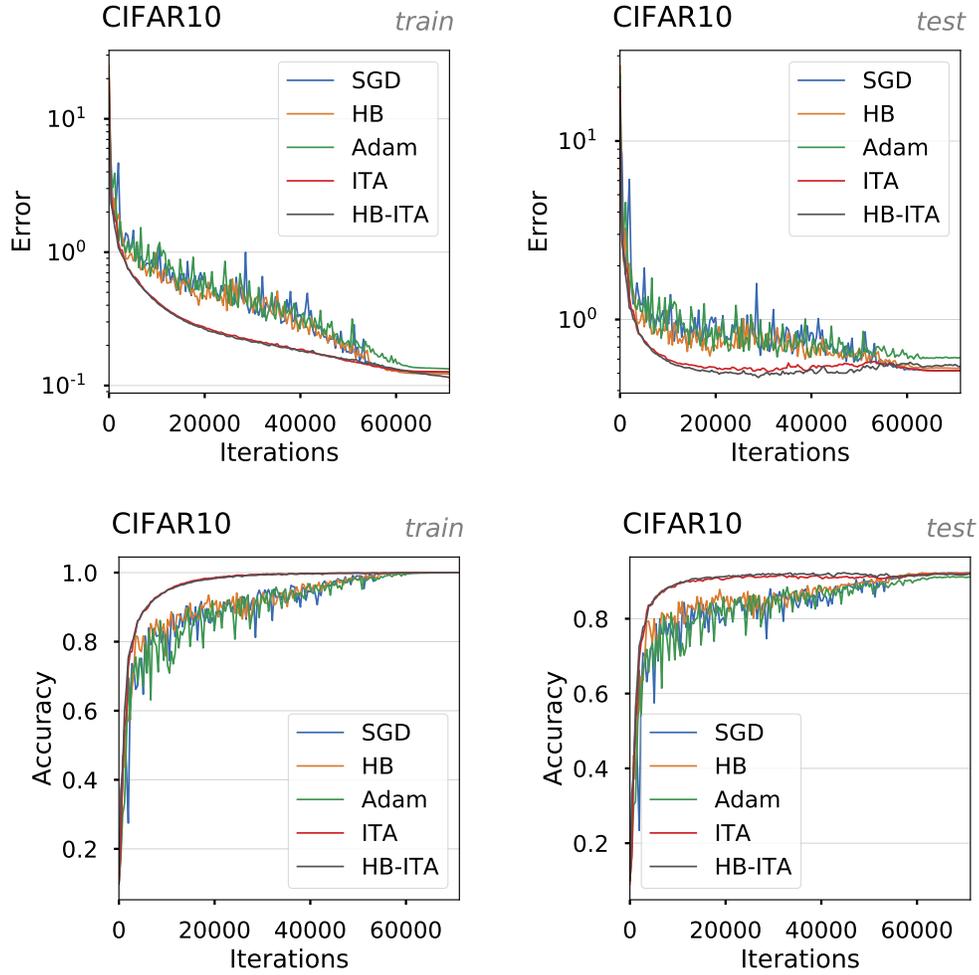

    \centering
    \includegraphics[width=0.48\textwidth]{figs/pa_baselines/train_err.pdf}
    \includegraphics[width=0.48\textwidth]{figs/pa_baselines/test_err.pdf}
    \\                        8
    \includegraphics[width=0.48\textwidth]{figs/pa_baselines/train_acc.pdf}
    \includegraphics[width=0.48\textwidth]{figs/pa_baselines/test_acc.pdf}
    \caption{
        Convergence and accuracy curves along training for the CIFAR10 experiments comparing baseline methods to ours.
        \textbf{Left}: Training.
        \textbf{Right}: Testing.
        \label{fig:cifar10-comp}
    }
\end{figure}

\paragraph{Ablation study: importance of IGT correction}
We confirm the importance of the implicit gradient transport correction by running an experiment 
in which an increasing momentum is used without transport. The results appear in figure~\ref{fig:cifar10-avg}.

The resulting hyper-parameters are:
\begin{itemize}
	\item ATA stepsize 0.3, decay 0.01, tail fraction 18
	\item HB-ATA stepsize 0.03, decay 0.01, tail fraction 18
\end{itemize}

\begin{figure}
    \centering
    \includegraphics[width=0.48\textwidth]{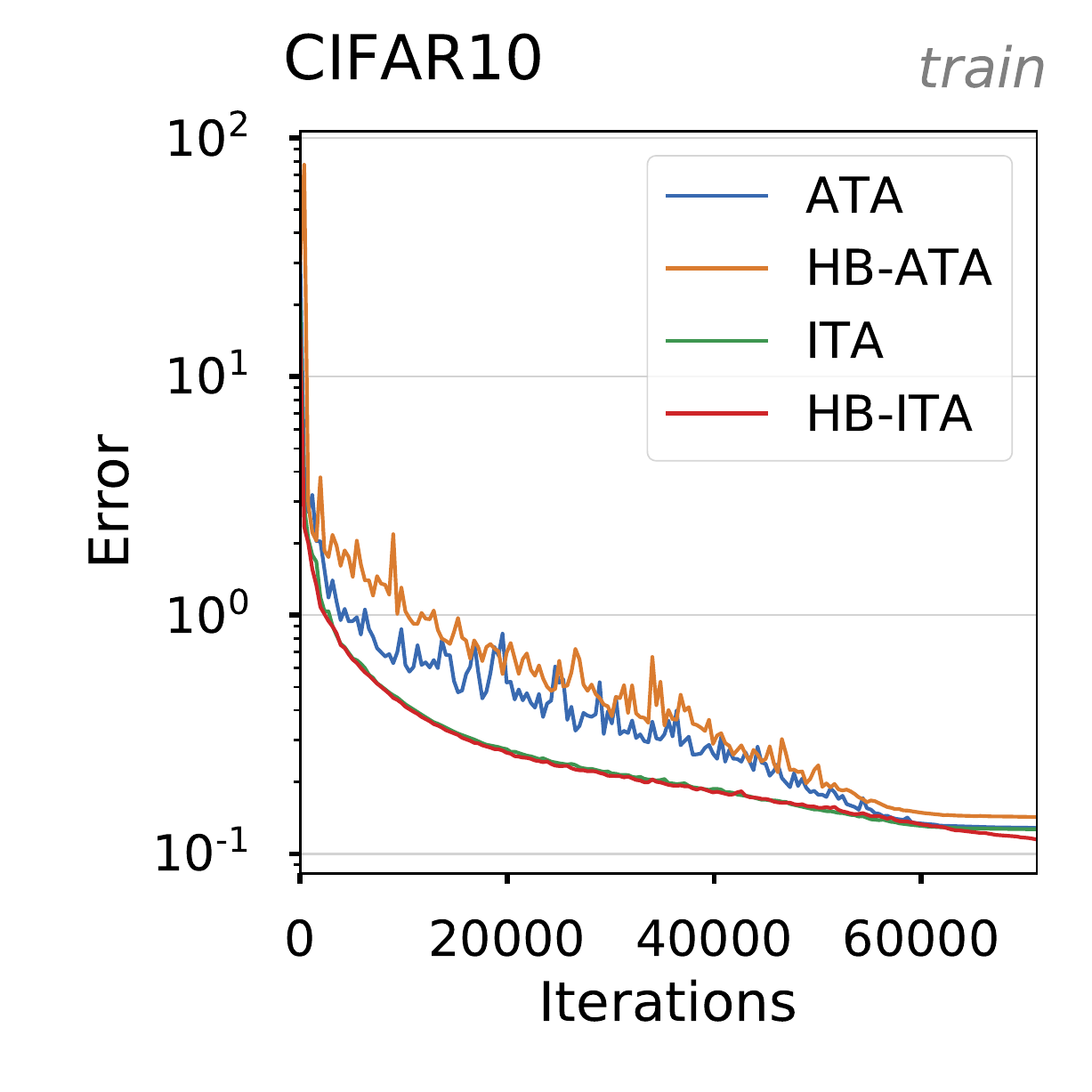}
    \includegraphics[width=0.48\textwidth]{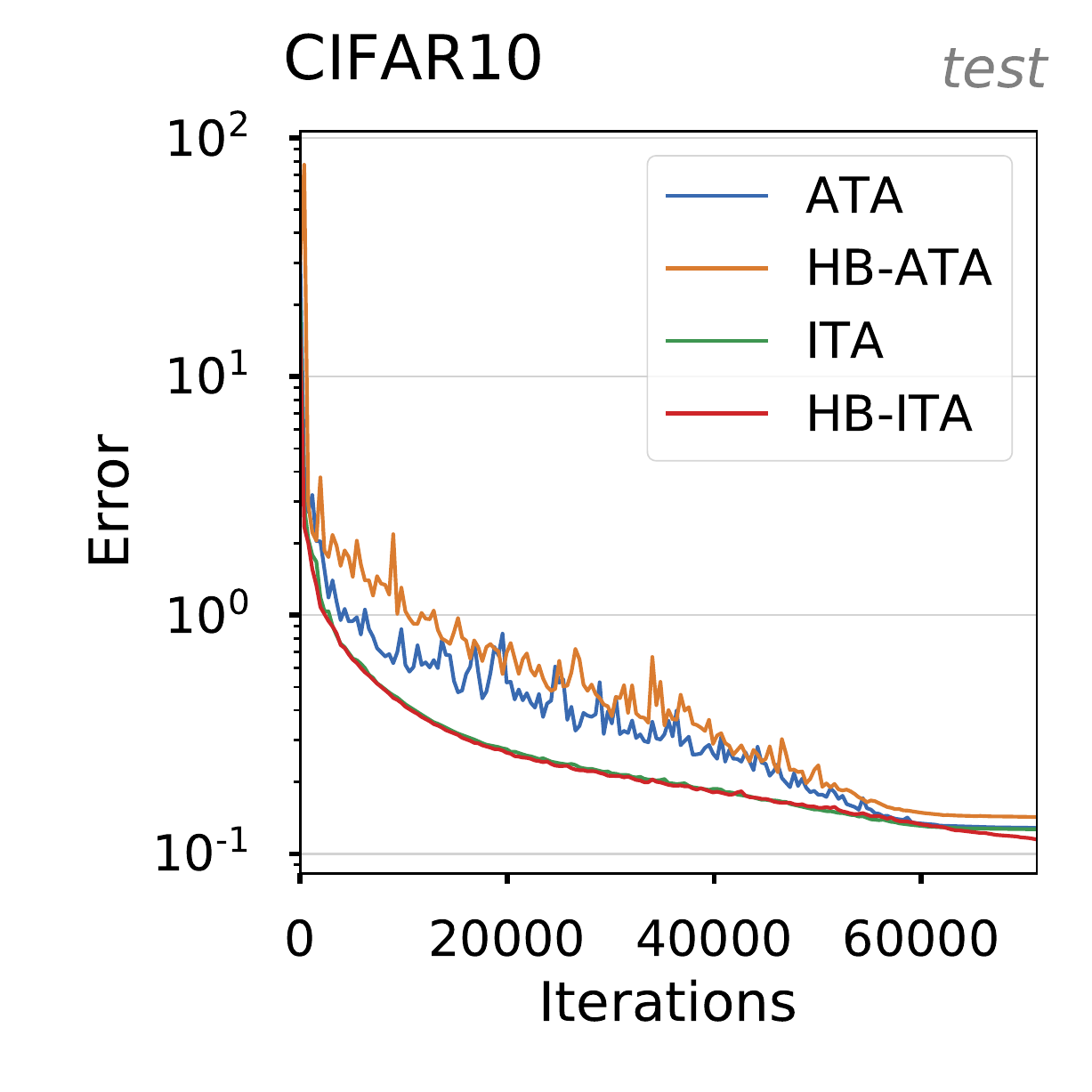}
    \\
    \includegraphics[width=0.48\textwidth]{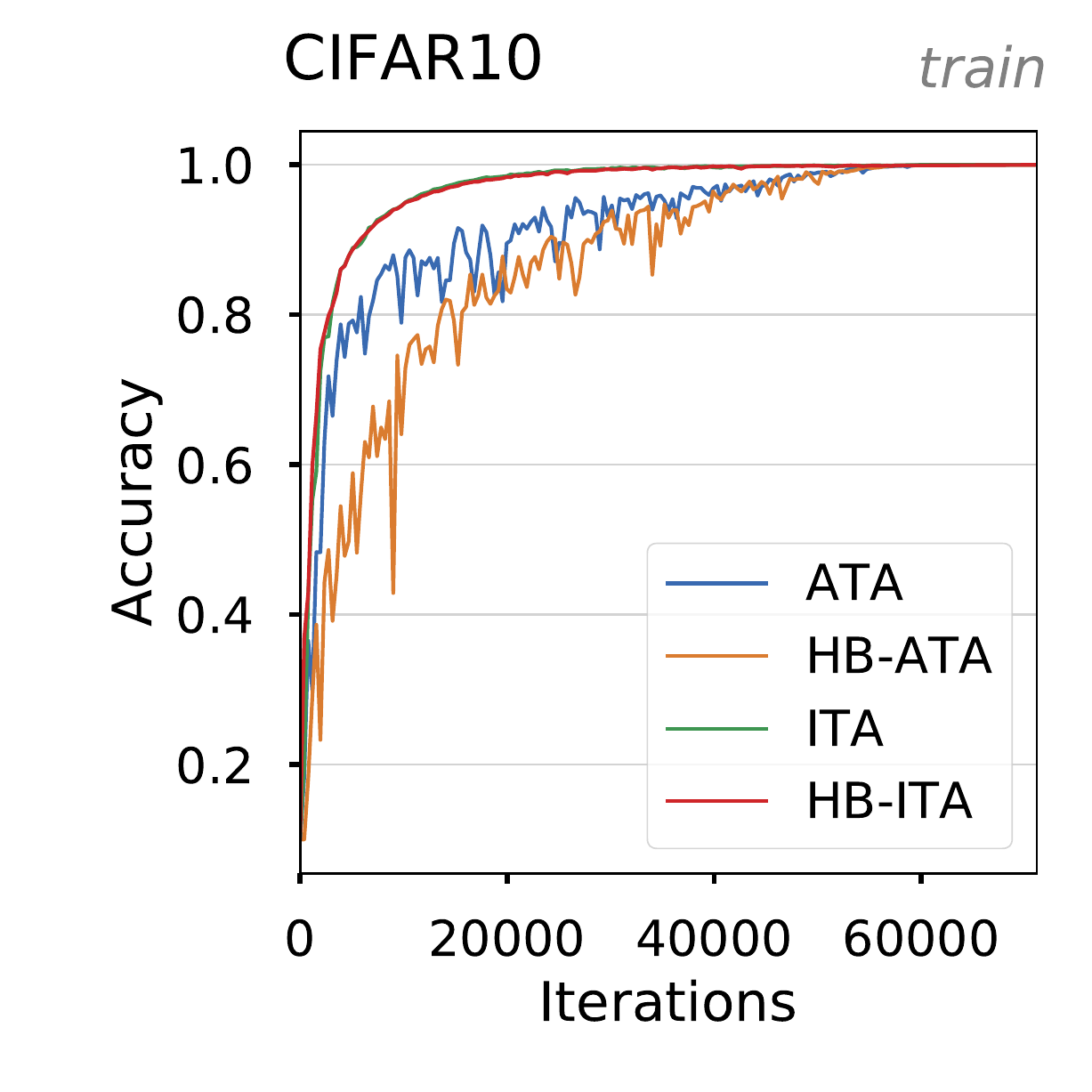}
    \includegraphics[width=0.48\textwidth]{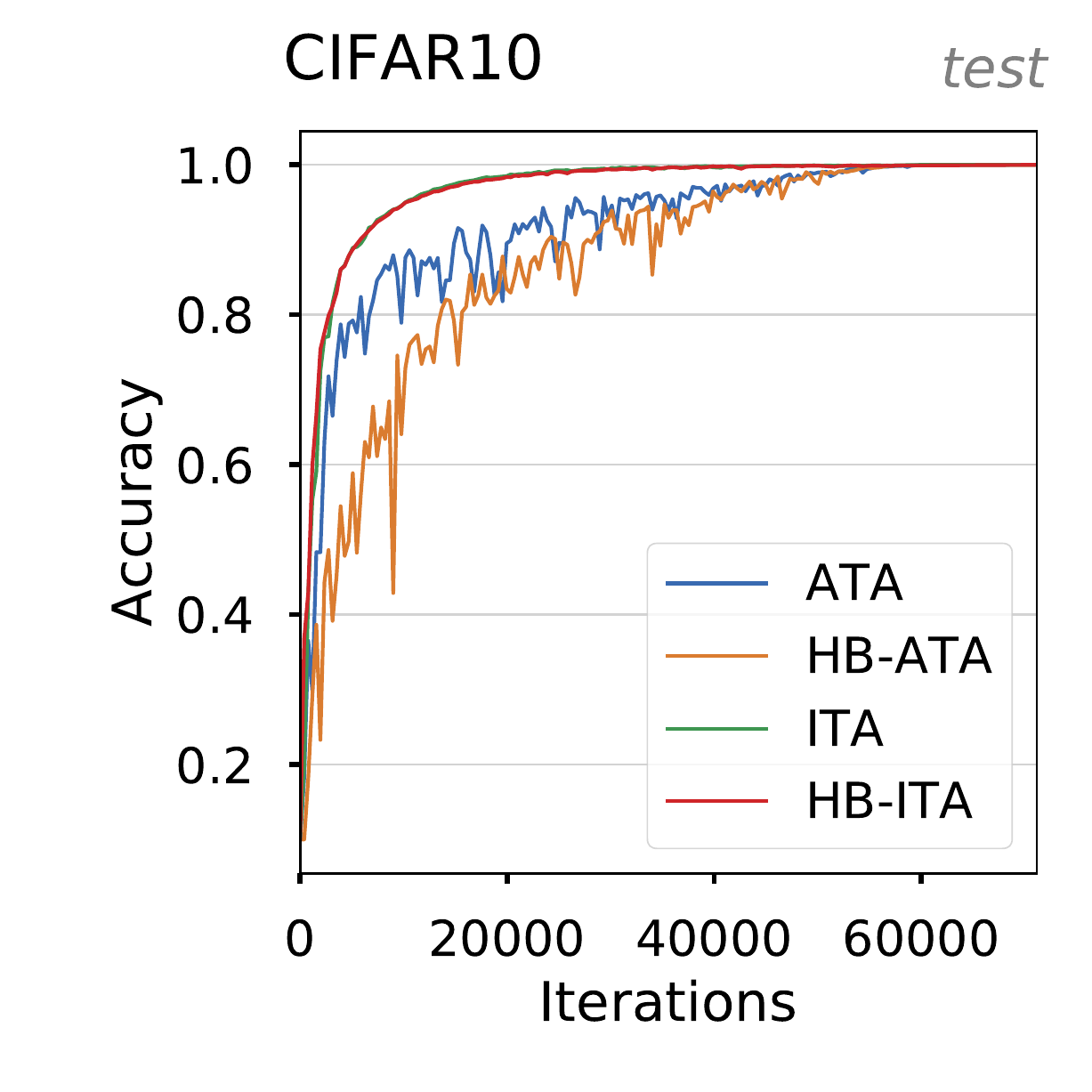}
    \caption{
        Convergence and accuracy curves along training for the CIFAR10 experiments comparing the use of ATA combined with our proposed implicit transport mechanism.
        \textbf{Left}: Training.
        \textbf{Right}: Testing.
        \label{fig:cifar10-avg}
    }
\end{figure}

\paragraph{Effect of the batch size}
We look into the effect of the batch size. To do so, we plot the number of steps
required to reach a reliably attainable training loss of 0.4 as a function of the batch size. 
We ran using the following mini-batch sizes: 32, 128, 512 and 2048. Running with larger minibatches led to out of memory errors
on our single GPU setup.
The results presented in figure~\ref{fig:analysis} show the benefit of IGT lowers as the batch size increases. Note that Adam's ability to keep benefiting from larger batch sizes
is consistent with previous observations.

\begin{figure}
    \begin{minipage}{0.5\textwidth}
    \centering
    \includegraphics[width=1.0\textwidth]{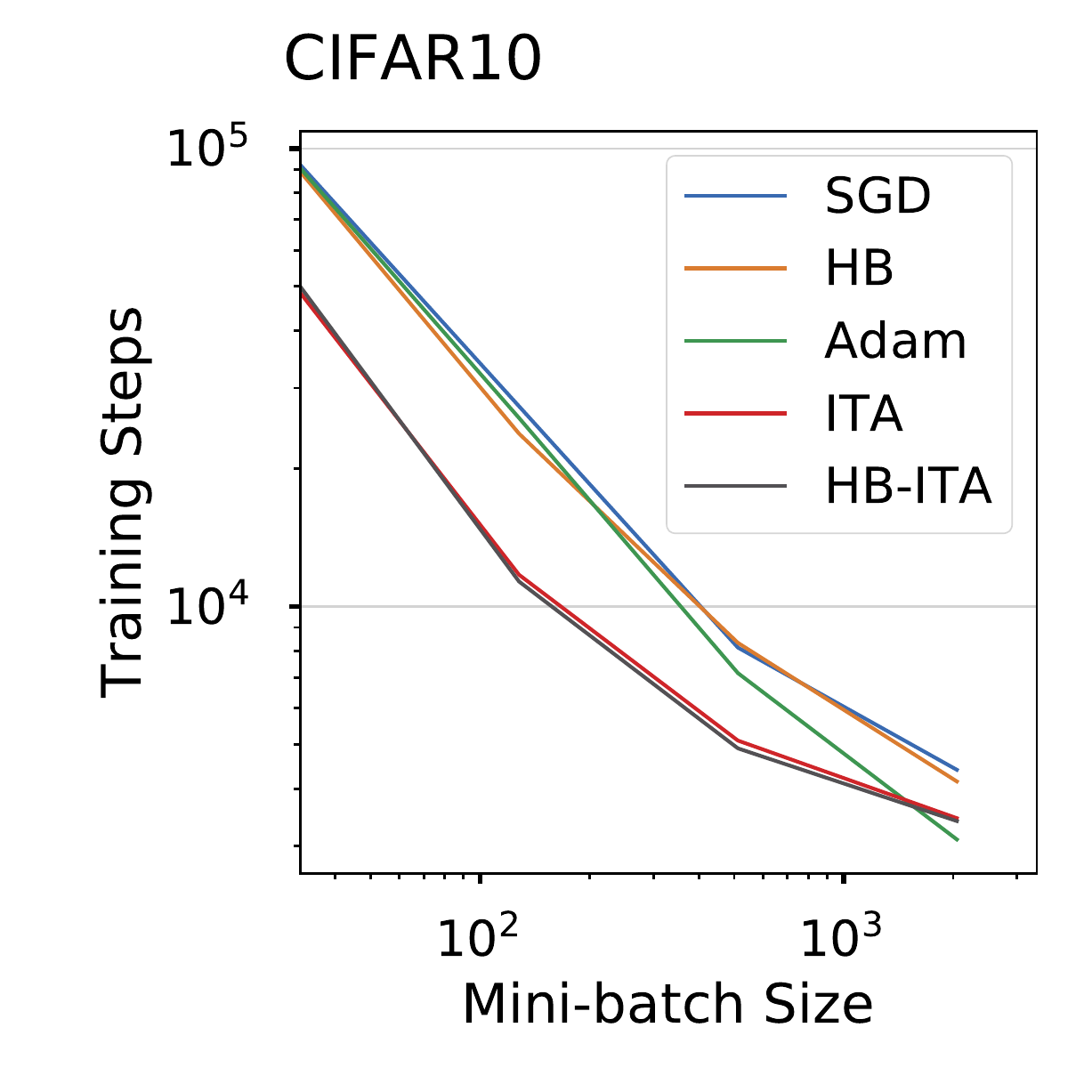}    
    \end{minipage}
    \begin{minipage}{0.49\textwidth}    
    \caption{
        Effect of mini-batch size on the number of steps to reach a target training loss.
        \label{fig:analysis}
    }
    \end{minipage}
\end{figure}

\subsection{ImageNet image classification}

\paragraph{Dataset}
We use the 2015 edition of the \textit{ImageNet Large-Scale Visual Recognition Challenge} (ImageNet) \cite{ILSVRC15} dataset.
This dataset consists of 1.2M images partitioned into 1'000 classes.
We use the pre-processing and loading utilities of the TensorFlow \texttt{models} package, available at the following URL: \url{https://github.com/tensorflow/models}

\paragraph{Model}
Our model is again a large residual network, consisting of 50 layers.
Similar to our CIFAR10 experiments above, we use the implementation described in \cite{he2016identity}.

\paragraph{Hyper-parameters}
We used the same setup and approach as for the CIFAR-10 experiments. The setup trains for 90 epochs using
mini-batches of 1024 examples. We used
a larger grid for the learning rate schedule: decay (0.1, 0.01, 0.001) and decay step fraction (0.7, 0.8, 0.9).

The resulting hyper-parameters are:
\begin{itemize}
	\item SGD stepsize 0.3, decay 0.01, decay step 0.8
	\item HB stepsize 0.03, decay 0.001, decay step 0.9
	\item Adam stepsize 0.0001, decay 0.01, decay step 0.9
	\item ITA stepsize 0.3, decay 0.01, tail fraction 90, decay step 0.9
	\item HB-ITA stepsize 0.03, decay 0.01, tail fraction 90, decay step 0.9
\end{itemize}

\paragraph{Infrastructure and Runs}
We ran these experiments using Google TPUv2.

\paragraph{Additional Results}

We include error and accuracy curves for training and testing in Figure~\ref{fig:all-ilsrvc}.

\begin{figure}
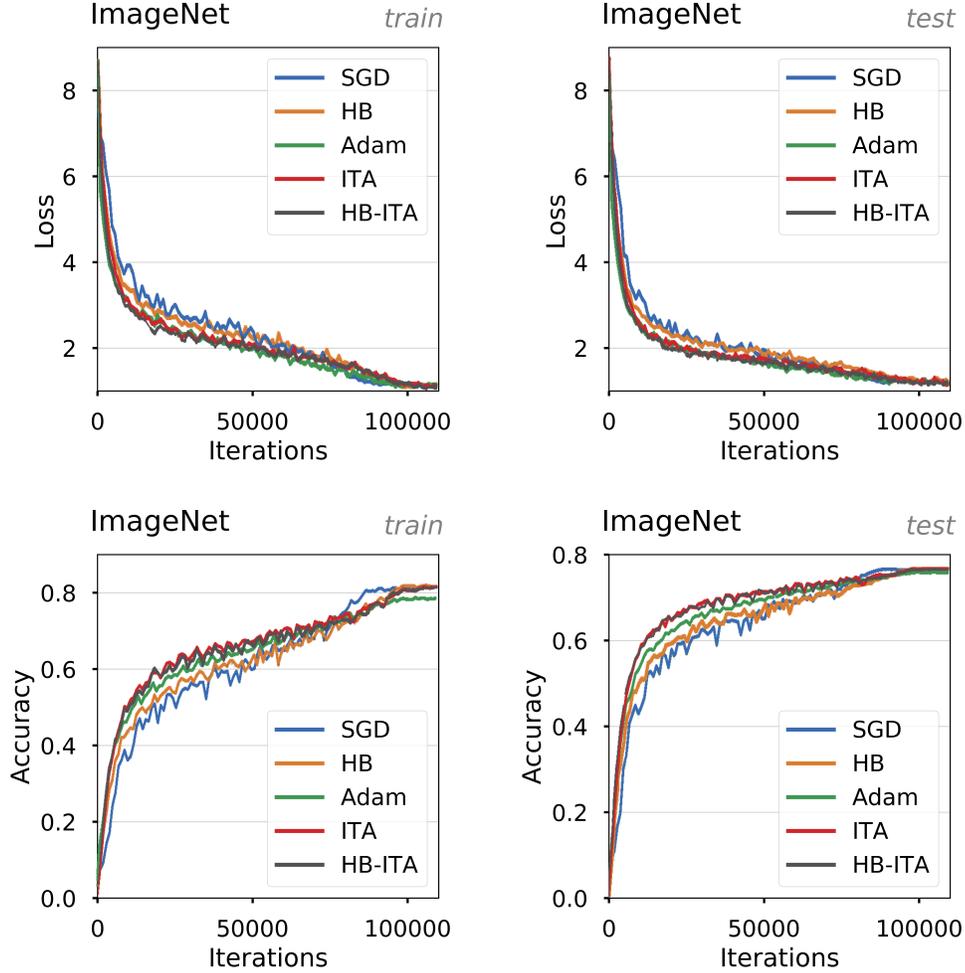

    \centering
    \includegraphics[width=0.48\textwidth]{figs/pa_imagenet/train_err.png}
    \includegraphics[width=0.48\textwidth]{figs/pa_imagenet/test_err.png}
    \\
    \includegraphics[width=0.48\textwidth]{figs/pa_imagenet/train_acc.png}
    \includegraphics[width=0.48\textwidth]{figs/pa_imagenet/test_acc.png}
    \caption{
        Convergence and accuracy curves along training for the ImageNet experiments.
        \textbf{Left}: Training.
        \textbf{Right}: Testing.
        \label{fig:all-ilsrvc}
    }
\end{figure}

\subsection{IMDb sentiment analysis}

\paragraph{Dataset}
The \textit{Internet Movie Database} (IMDb) \cite{maas2011learning} consists of 25'000 training and 25'000 test movie reviews.
The objective is binary sentiment classification based on the review's text.
We randomly split the training set in two folds of 17'536 and 7'552 reviews, the former being used for training and the latter for testing.
The data is downloaded, splitted, and pre-processed with \texttt{torchtext} package, available at the following URL: \url{https://github.com/pytorch/text}
More specifically, we tokenize the text at the word-level using the \texttt{spaCy} package, and embed the tokens using the 100-dimensional GloVe 6B \cite{pennington2014glove} distributed representations.

\paragraph{Model}
The model consists of an embedding lookup-table, followed by a bi-directional LSTM with dropout, and then by a fully-connected layer.
The LSTM uses 256 hidden units and the dropout rate is set to 0.5.
The whole model consists of 3.2M trainable parameters, with the embedding lookup-table initilized with the GloVe vectors.
The model is trained to minimize the \texttt{BCEWithLogitsLoss} with a mini-batch size of 64.

\paragraph{Hyper-parameters}
For each method, we used a grid-search to find the stepsize minimizing validation error after 15 epochs.
The grid starts at 0.00025 and doubles until reaching 0.1024, so as to ensure that no chosen value lies on its boundaries.
When applicable, the momentum factor is jointly optimized over values 0.1 to 0.95.
The final hyper-parameters are displayed in the following table for each method.

\begin{table}[h]
    \centering
    \caption{Hyperparameters for IMDb experiments.}
    \begin{tabular}{@{}lccccc@{}}
    \toprule
             & HB        & Adam     & ASGD     & HB-IGT        & HB-ITA        \\ \midrule
    $\alpha$ & 0.032     & 0.0005   & 0.064    & 0.128         & 0.064         \\
    $\mu$    & 0.95      & 0.95     &  n/a     & 0.9           & 0.9           \\
    $\xi$    & n/a       & n/a      & 100      & n/a           & n/a           \\
    $\kappa$ & n/a       & n/a      & $10^5$   & n/a           & n/a           \\ \bottomrule
    \end{tabular}
\end{table}

\paragraph{Infrastructure and Runs}
All IMDb experiments use a single NVIDIA GTX 1080, with PyTorch v0.3.1.post2, CUDA 8.0, and cuDNN v7.0.5.
We run each final configurations with 5 different random seeds and always report the mean tendency $\pm$ one standard deviation.
Each run lasts approximately three hours and thirty minutes.

\paragraph{Additional Results}
In addition to the results reported in the main text, we include training, validation, and testing curves for each method in Figure~\ref{fig:imdb}.
Shading indicates the one standard deviation interval.
Note that our focus is explicitly on optimization: in the specific case of IMDb, training for 200 epochs is completely unnecessary from a generalization standpoint as performance degrades rapidly after 15-20 epochs.

\begin{figure}
    \centering
    \includegraphics[width=0.48\textwidth]{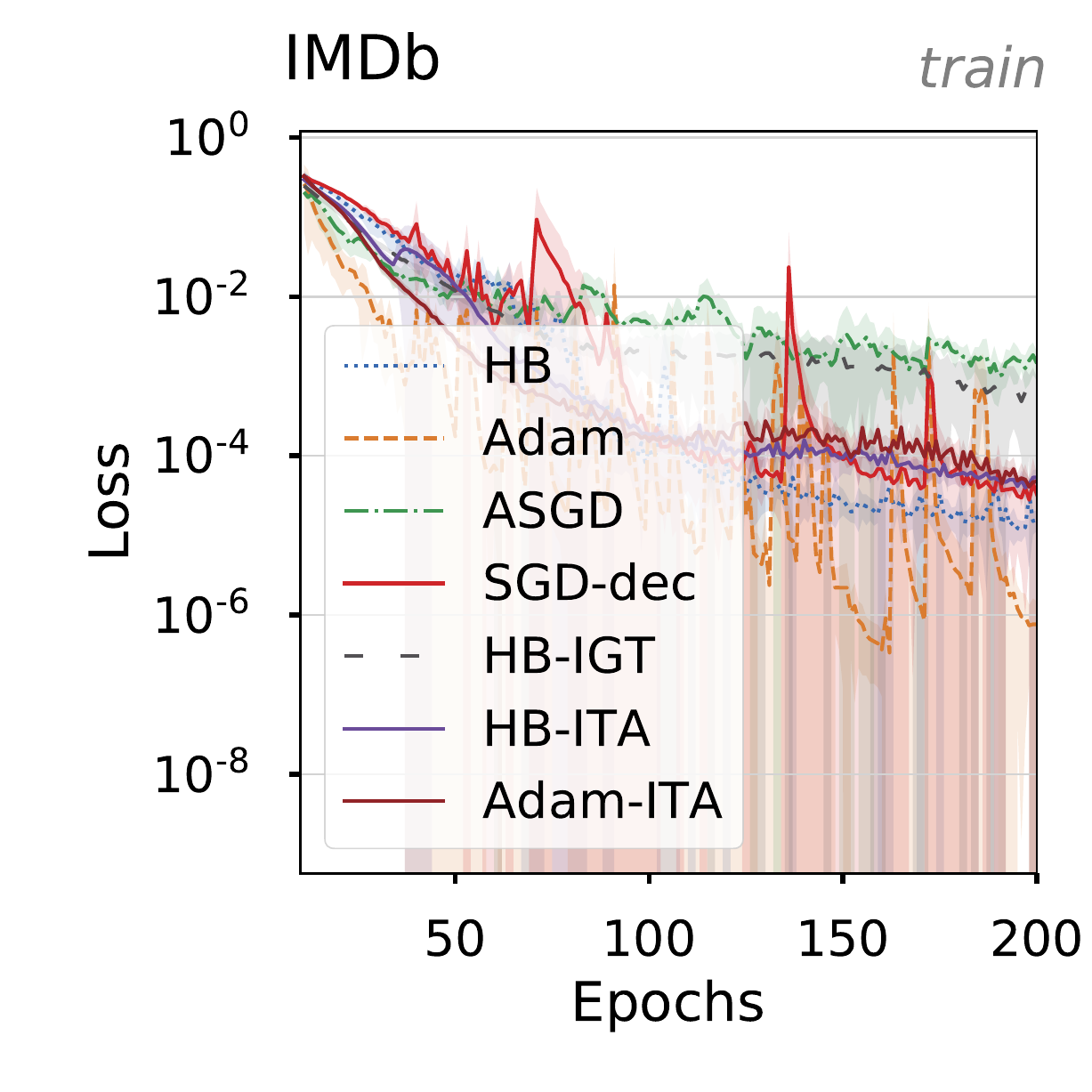}
    \includegraphics[width=0.48\textwidth]{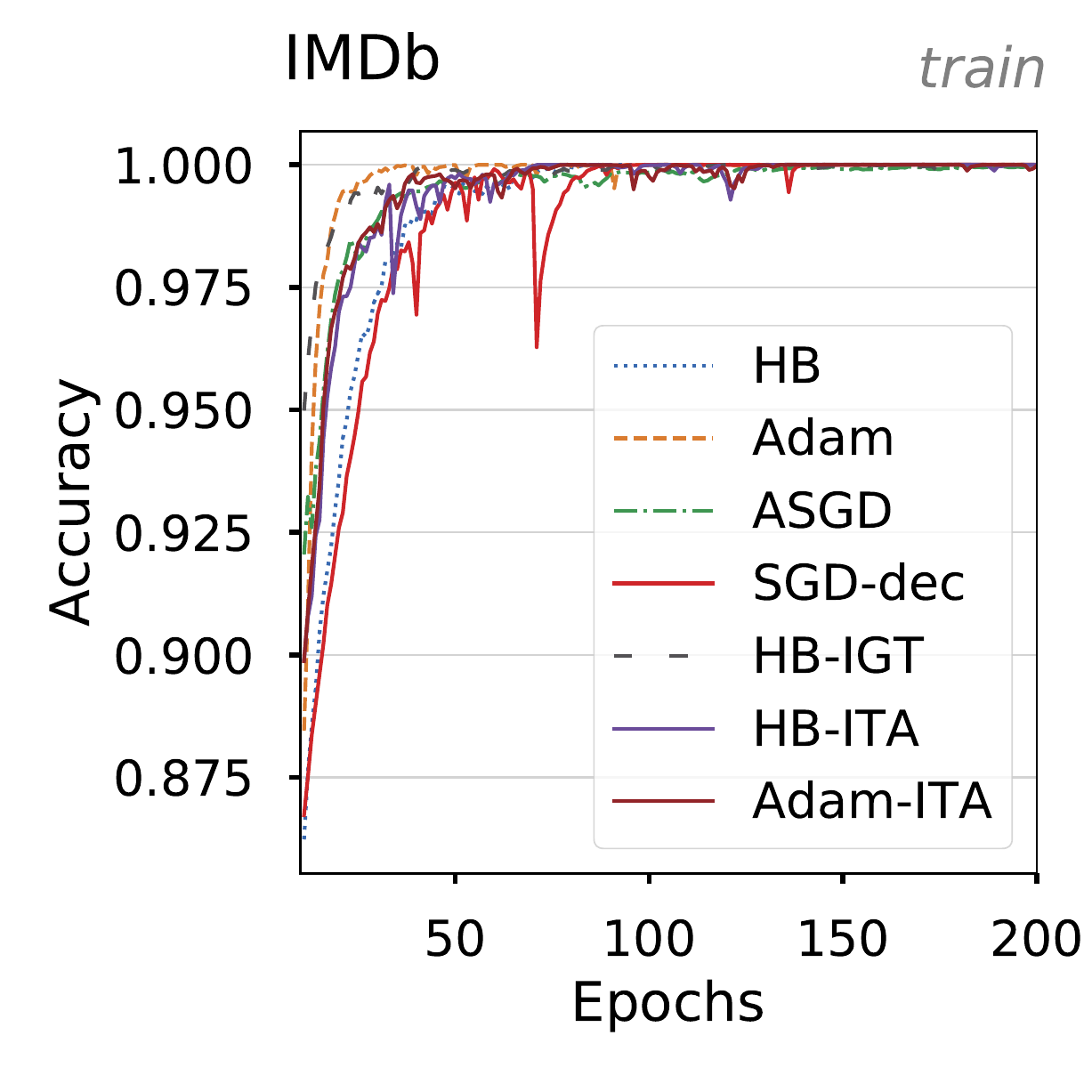}
    \\
    \includegraphics[width=0.48\textwidth]{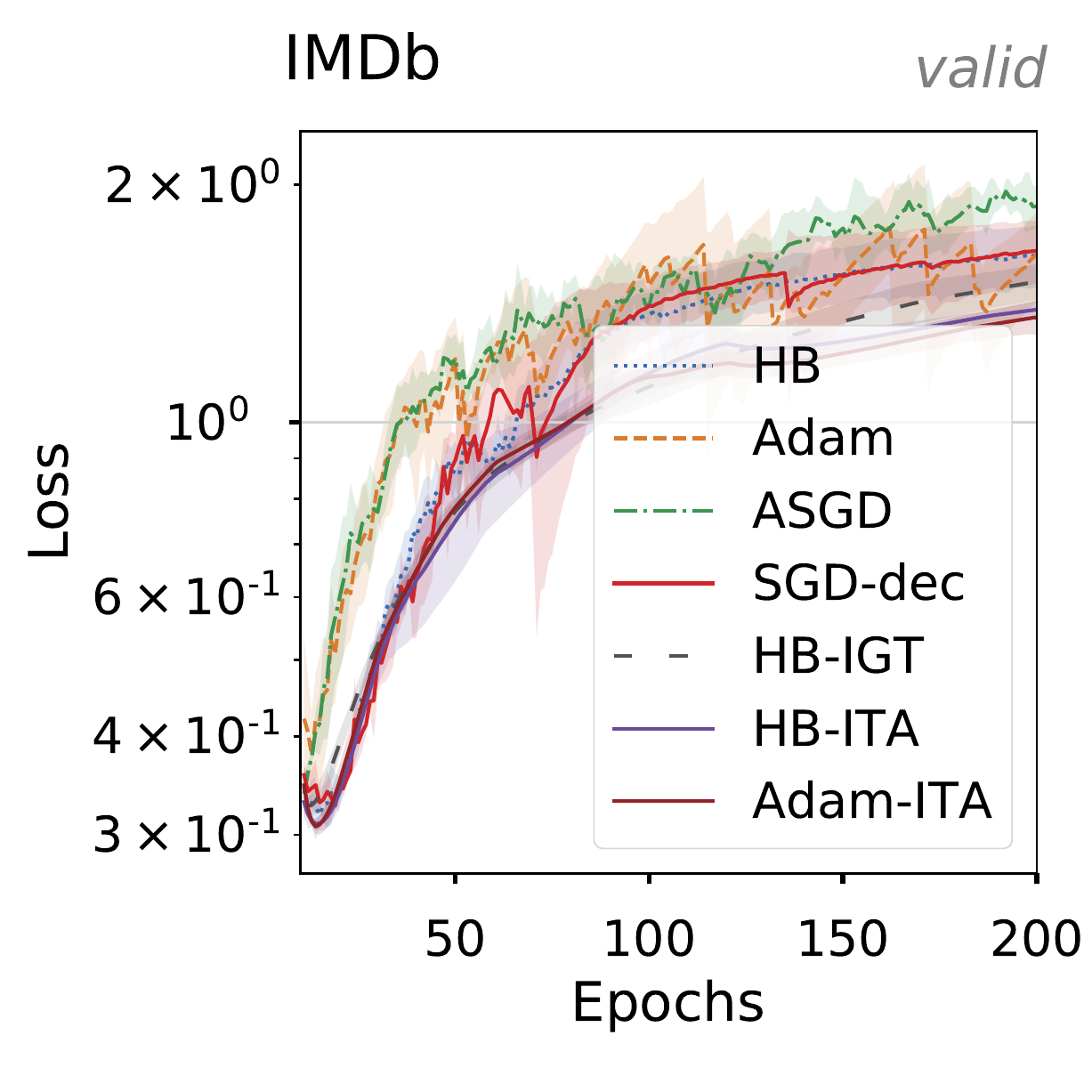}
    \includegraphics[width=0.48\textwidth]{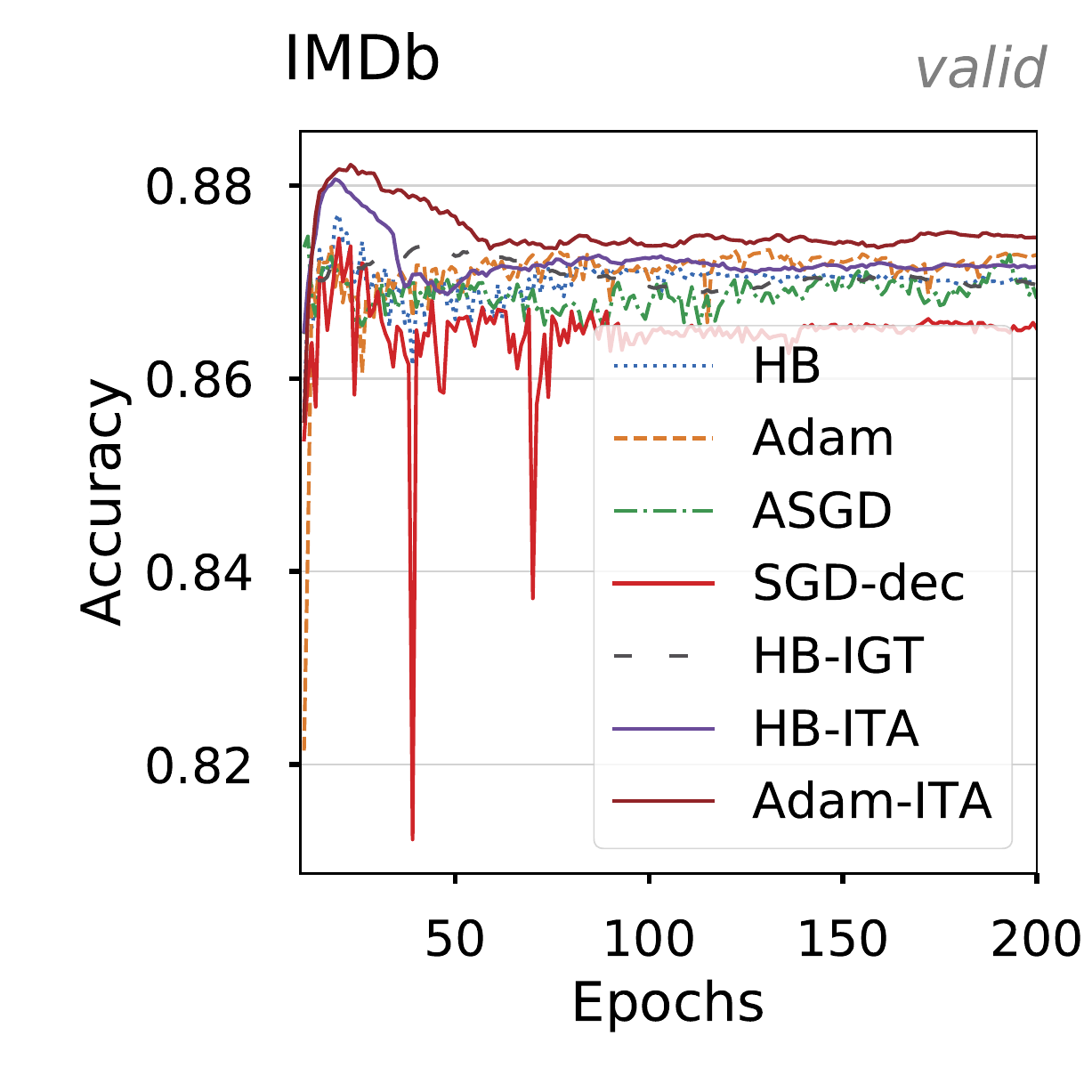}
    \\
    \includegraphics[width=0.48\textwidth]{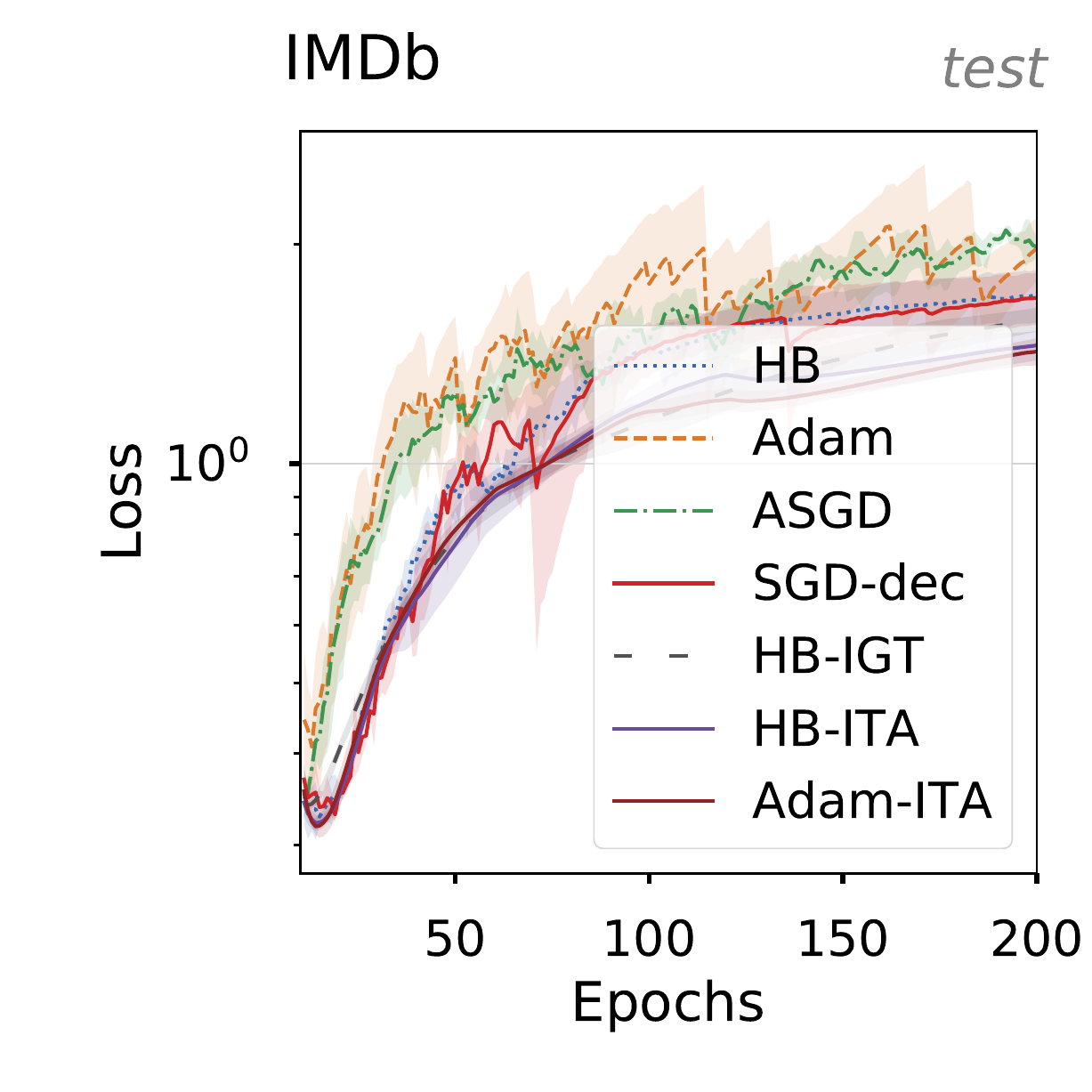}
    \includegraphics[width=0.48\textwidth]{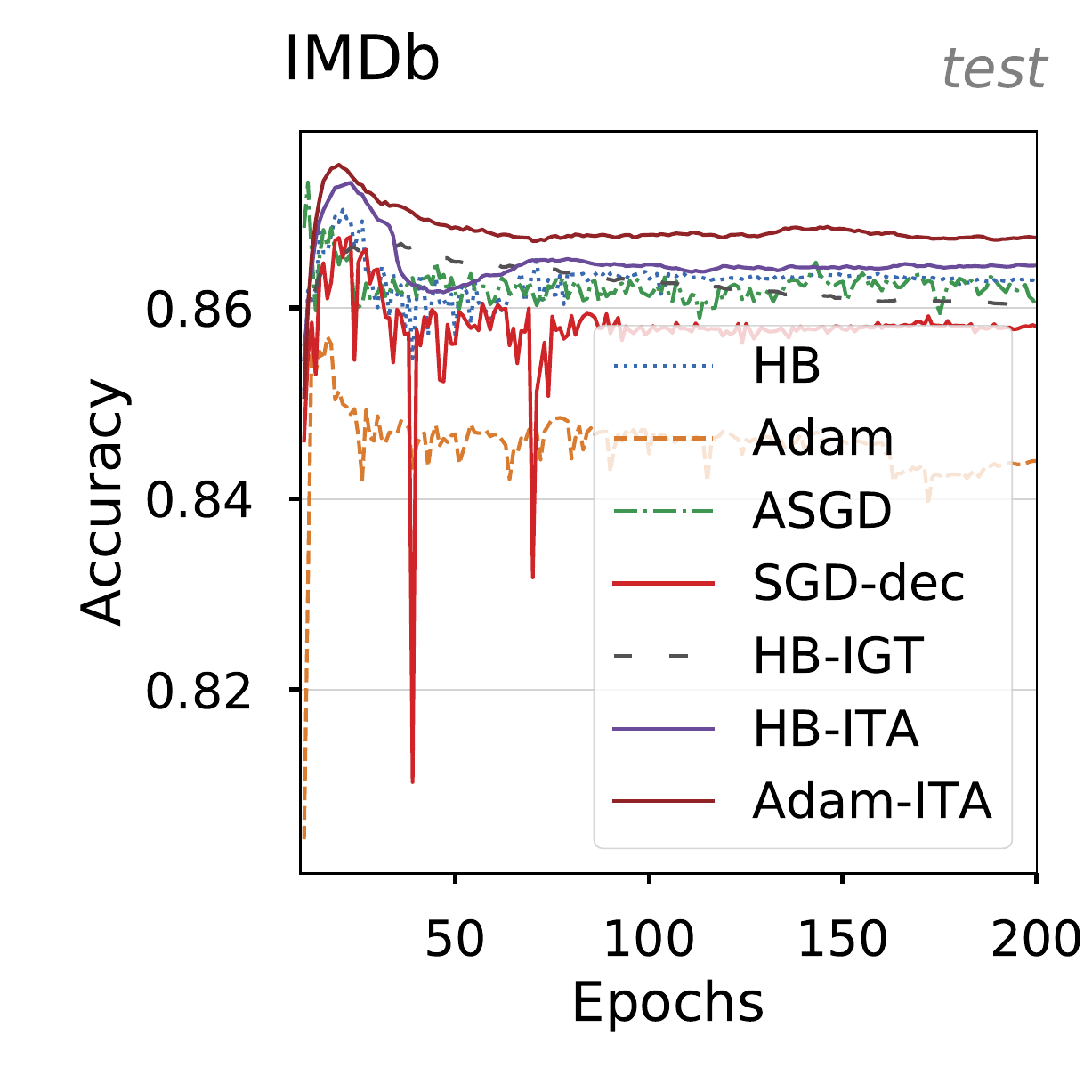}
    \caption{
        Convergence and accuracy curves along training for the IMDb experiments.
        \textbf{Left}: Convergence.
        \textbf{Right}: Accuracy.
        \label{fig:imdb}
    }
\end{figure}

\subsection{Linear-quadratic regulator}

\paragraph{Setup}
Our linear-quadratic regulator \cite{kwakernaak1972linear} implements the following equations.
The cost functional is evaluated at every timestep $h$ and is given by

\begin{align}
    C(s_h, a_h) = s_h^\top Q s_h + a_h^\top R a_h,
\end{align}

for random symmetric positive definite matrices $Q\in \R^{20 \times 20}$ and $R\in \R^{12\times 12}$ each with condition number 3.
The initial state $s_0 \sim \mathcal{N}(0, 3\cdot I_{20})$ is sampled around the origin, and the subsequent states evolve according to 

\begin{align}
    s_{h+1} = As_h + Ba_h,
\end{align}

where entries of $A\in\R^{20\times 20}, B\in\R^{20\times 12}$ are independently sampled from a Normal distribution and then scaled such that the matrix has unit Frobenius norm.
The actions are given by the linear stochastic policy $a_h = Ks_h + \epsilon_h^a$, where $\epsilon_h^a \sim \mathcal{N}(0, I)$ and $K$ are the parameters to be optimized.

Gradient methods in this manuscript optimize the sum of costs using the REINFORCE estimate \cite{williams1992simple} given by 

\begin{align}
    \nabla_K \E \sum_h^{10} C(s_h, a_h) = \E \left(\sum_h^{10} \nabla_K \log \pi_K(a_h \vert s_h)\right) \left(\sum_h^{10} C(s_h, a_h)\right).
\end{align}

In our experiments, the above expectation is approximated by the average of 600 trajectory rollouts.
Due to the noisy dynamics of the system, it is possible for the gradient norm to explode leading to numerical unstabilities -- especially when using larger stepsizes.
To remedy this issue, we simply discard such problematic trajectories from the gradient estimator.

For each training iteration, we first gather 600 trajectories used for learning and then 600 more used to report evaluation metrics.

\paragraph{Hyper-parameters}
Due to the simplicity of the considered methods, the only hyper-parameter is the stepsize.
For each method, we choose the stepsize from a logarithmically-spaced grid so as to minimize the evaluation cost after 600 iterations on a single seed.
Incidentally, the optimal stepsize for GD, SGD, and ITA is 0.0002.

\paragraph{Infrastructure and Runs}
We use an Intel Core i7-5820K CPU to run the LQR experiments.
All methods are implemented using \texttt{numpy} v1.15.4.
We present results averaged over 3 random seeds, and also report the standard deviation.
For stochastic gradient methods (SGD, ITA) training for 20K iterations takes about 3h, for full-gradient method (GD) about 10h, and computing the solution of the Riccati equation takes less than 5 seconds.

\paragraph{Additional Results}
In addition to the evaluation cost reported in the main text, we also include the cost witnessed during training (and used for optimization) in Figure~\ref{fig:lqr}.

\begin{figure}
    \centering
    \includegraphics[width=0.42\textwidth]{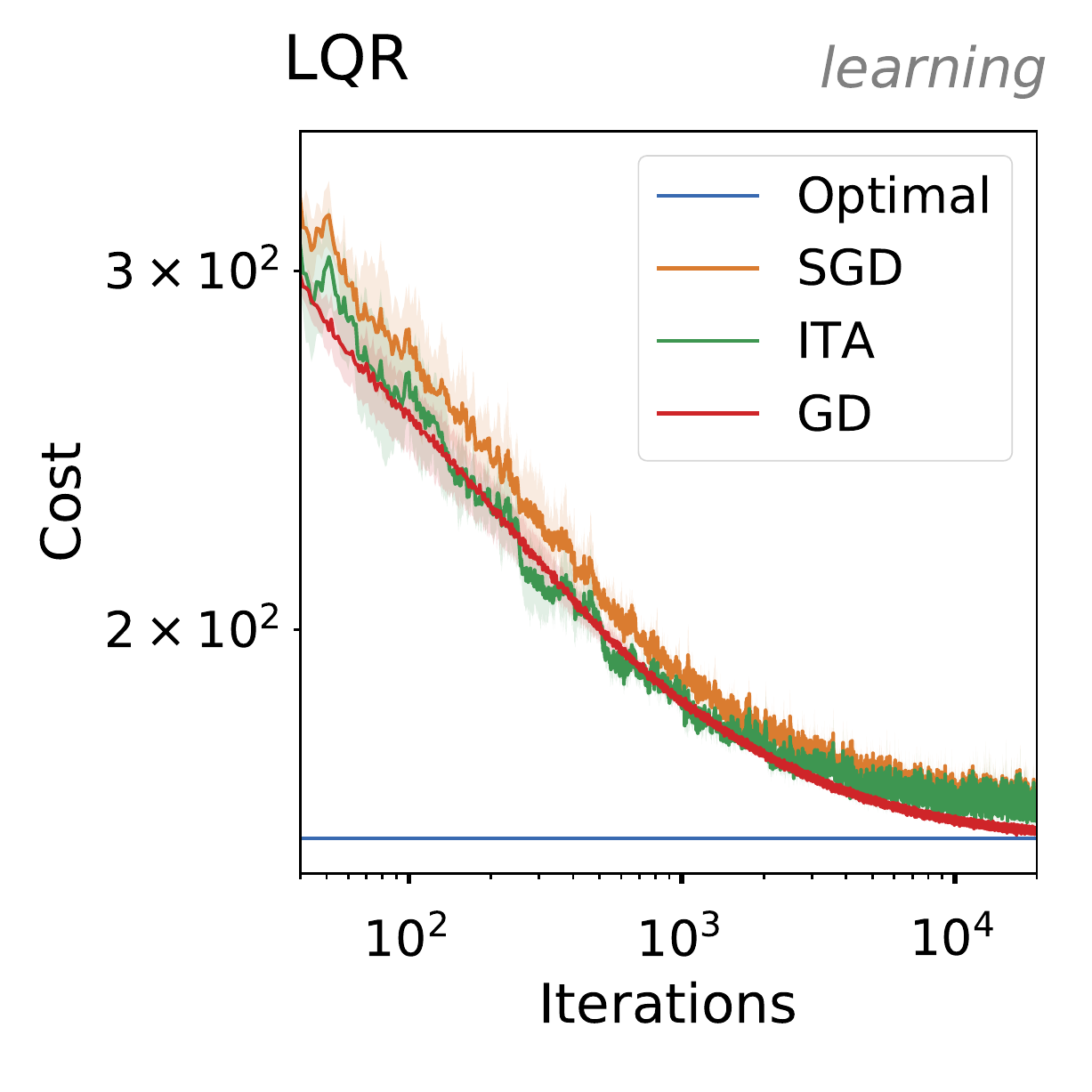}
    \includegraphics[width=0.42\textwidth]{figs/neurips19_lqr/test.pdf}
    \caption{
        LQR costs along training iterations.
        \textbf{Left}: Costs used for learning.
        \textbf{Right}: Costs used for evaluation.
        \label{fig:lqr}
    }
\end{figure}

We notice that the training cost curve of ITA is not as smooth as the evaluation one.
Similarly, the observed learning costs never reach as good a minima as the evaluation ones.
This phenomena is easily clarified: during learning, ITA esimates the gradient using the shifted parameters $K_t + \frac{\gamma_t}{1-\gamma_t}(K_t - K_{t-1})$ as opposed to the true parameters $K_t$.
Those shifted parameters are not subject to a reduced variance, hence explaining the observed noise in the cost as well as deteriorated convergence.

\subsection{Model-agnostic meta-learning}

\paragraph{Dataset}
We use the Mini-Imagenet dataset \cite{ravi2016optimization} in our model-agnostic meta-learning (MAML) \cite{finn2017model} experiments.
This dataset comprises 64 training, 12 validation, and 24 test classes.
For each of train, validation, and test sets, tasks are constructed by sampling 5 classes from their respective split, and further sampling 5 images per class.
Images are downsampled to 84x84x3 tensors of RGB values.
For more details, please refer to the official code repository -- which we carefully replicated -- at the following URL: \url{https://github.com/cbfinn/maml}

Our implementation departs in two ways from the reference.
First, we train our models for 100k iterations as opposed to 60k and only use 5 image samples to compute a meta-gradient whereas the reference implementation uses 15.
Second, we only use 5 adaptation steps at evaluation time, when the reference uses 10.

\paragraph{Model}
The model closely replicates the convolutional neural network of MAML \cite{finn2017model}.
It consists of 4 layers, each with 32 3x3 kernels, followed by batch normalization and ReLU activations.
For specific implementation details, we refer the reader to the above reference implementation.

\paragraph{Hyper-parameters}
We only tune the meta-stepsize for the MAML experiment.
We set the momentum constant to 0.9, the adaptation-stepsize to 0.01, and average the meta-gradient of 4 tasks per iterations.
Due to the reduced variance in the gradients, we found it necessary to increase the $\epsilon$ of Adam-ITA to 0.01.

For each method, we search over stepsize values on a logarithmically-spaced grid and select those values that maximize validation accuracy after 10k meta-iterations.
These values are reported in Table~\ref{tab:maml}.

\begin{table}[h]
    \centering
    \begin{tabular}{@{}lccccc@{}}
    \toprule
             & HB        & Adam     & HB-ITA  & Adam-ITA \\ \midrule
    $\alpha$ & 0.008     & 0.001    & 0.008   & 0.0005   \\\bottomrule
    \end{tabular}
    \caption{Stepsizes for MAML experiments.\label{tab:maml}}
\end{table}

\paragraph{Infrastructure and Runs}
Each MAML experiment is run on a single NVIDIA GTX TITAN X, with PyTorch v1.1.0, CUDA 8.0, and cuDNN v7.0.5.
We run each configuration with 3 different random seeds and report the mean tendency $\pm$ one standard deviation.
Each run takes approximately 36 hours, and we evaluate the validation and testing accuracy every 100 iteration.

\paragraph{Additional Results}

We complete the MAML validation curves from the main manuscript with training and testing accuracy curves in Figure~\ref{fig:maml}.
Moreover, we recall the final test accuracies for each method: Adam-ITA reaches $65.16\%$, HB-ITA $64.57\%$, Adam $63.70\%$, and HB $63.08\%$.

\begin{figure}
    \centering
    \includegraphics[width=0.32\textwidth]{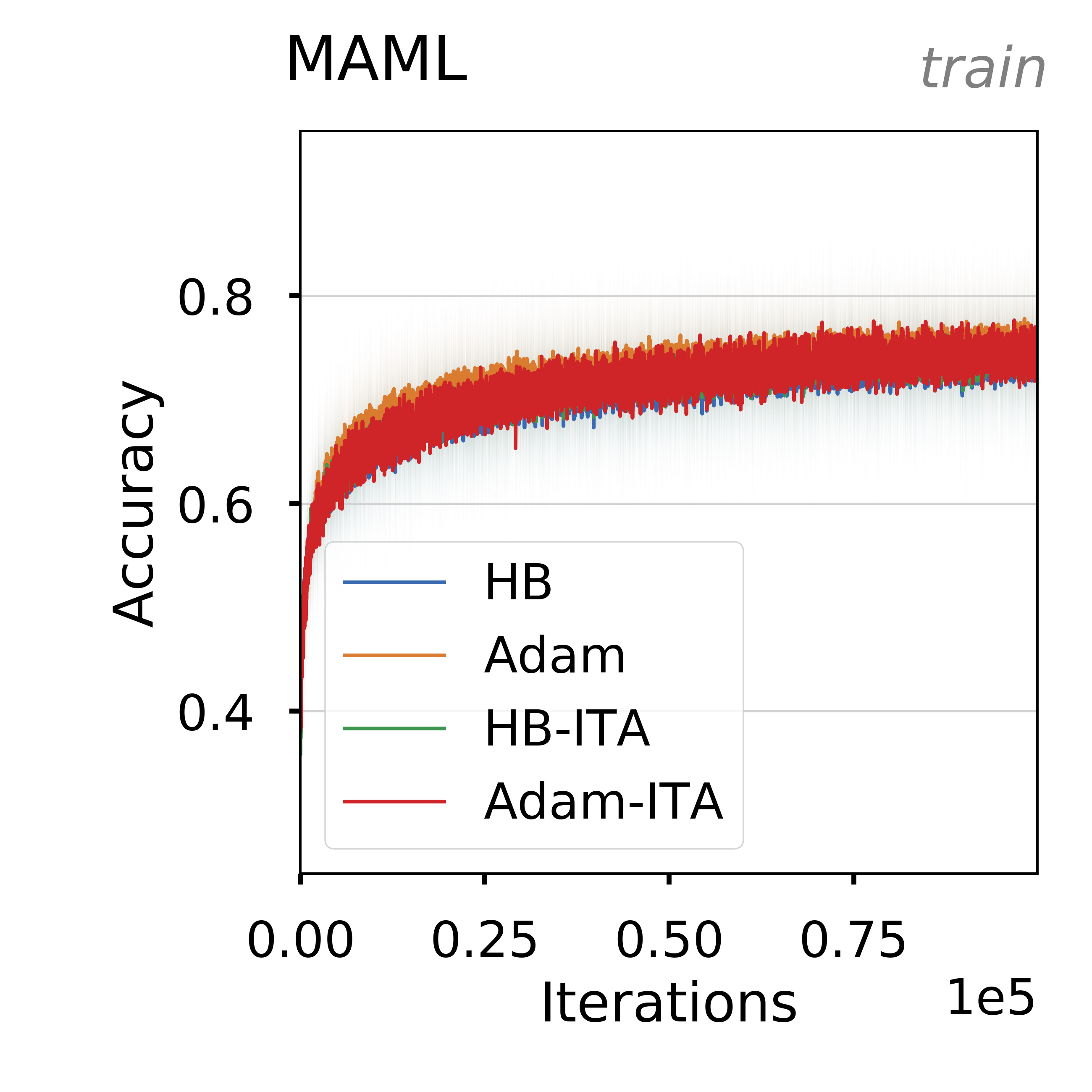}
    \includegraphics[width=0.32\textwidth]{figs/neurips19_maml/valid_acc.pdf}
    \includegraphics[width=0.32\textwidth]{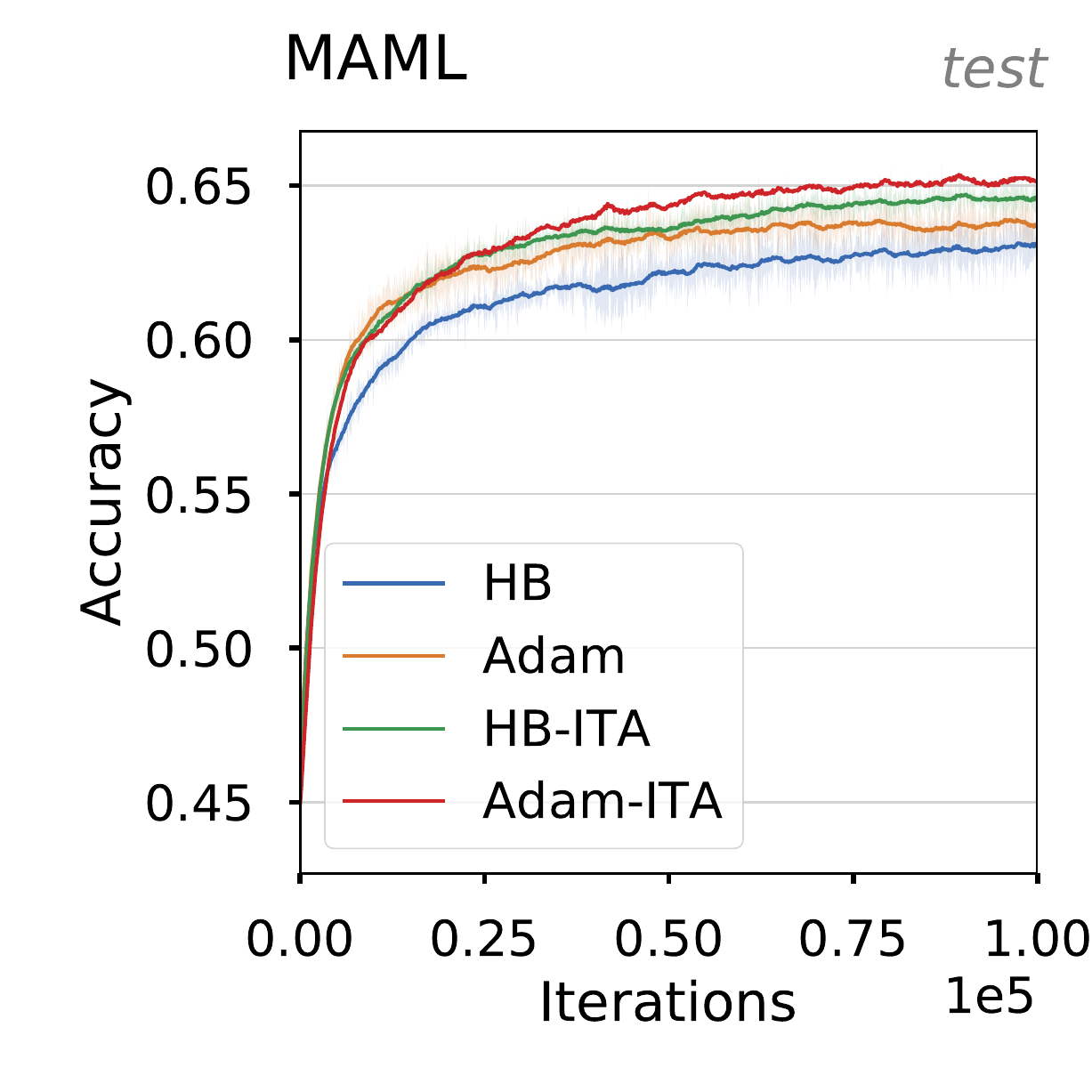}
    \caption{
        Training, validation, and testing accuracies for the MAML experiments along training.
        Shading indicates the 1 standard deviation interval.
        \textbf{Left}: Training.
        \textbf{Center}: Validation.
        \textbf{Right}: Testing.
        \label{fig:maml}
    }
\end{figure}

\section{Additional Experiments}
\label{sec:additional-experiments}
This section presents additional experiments to the ones included in the main text.

\subsection{Baselines comparisons}

While experiments in Section~\ref{sec:igt_ideal} highlighted properties of IGT and HB-IGT when the assumption of identical, constant Hessians was verified, we now turn to more realistic scenarios where individual functions are neither quadratic nor have the same Hessian to compare our proposed methods against popular baselines for the online stochastic optimization setting.
We target optimization benchmarks and focus on training loss minimization.

\begin{figure}
    \centering
    \includegraphics[width=0.32\textwidth]{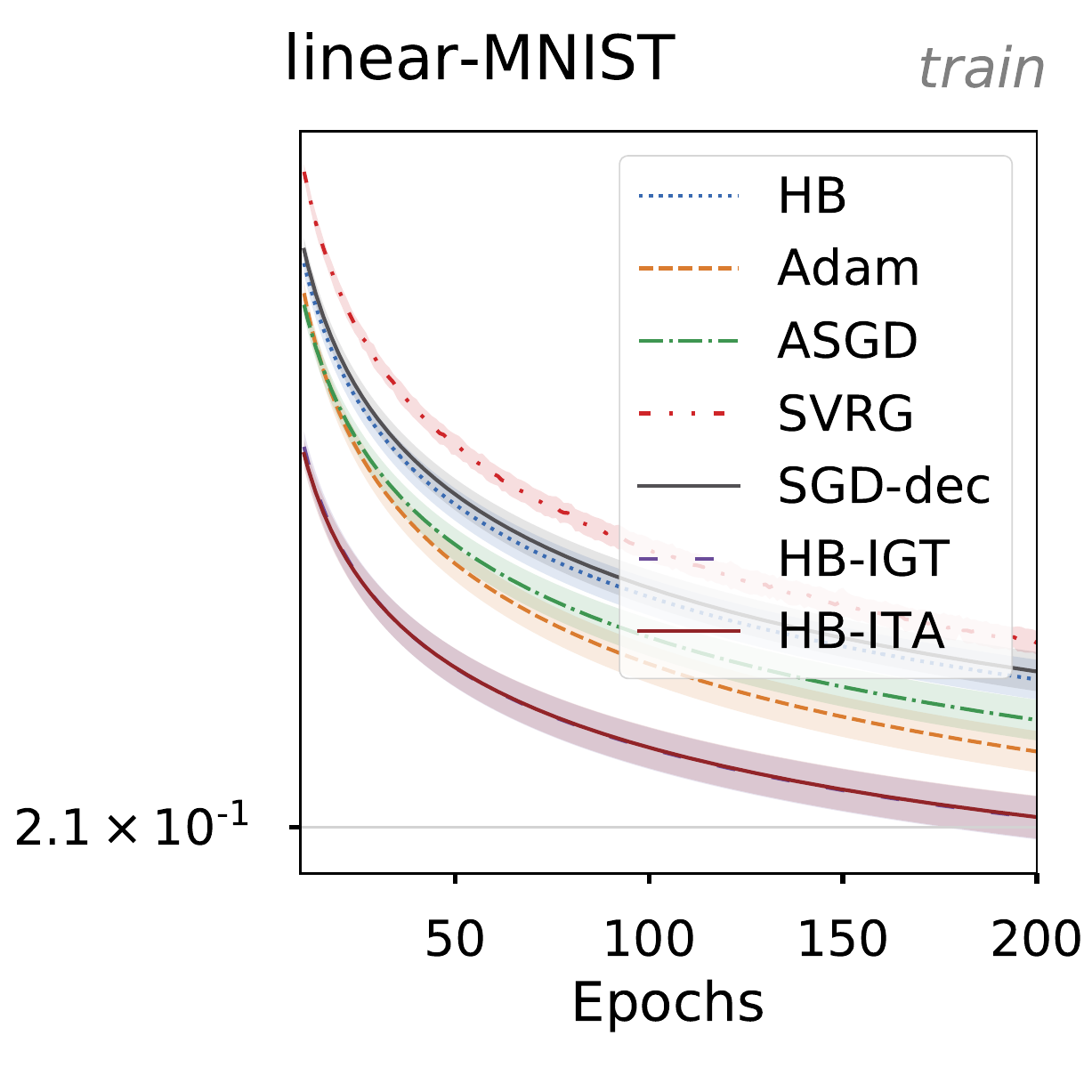}
    \includegraphics[width=0.32\textwidth]{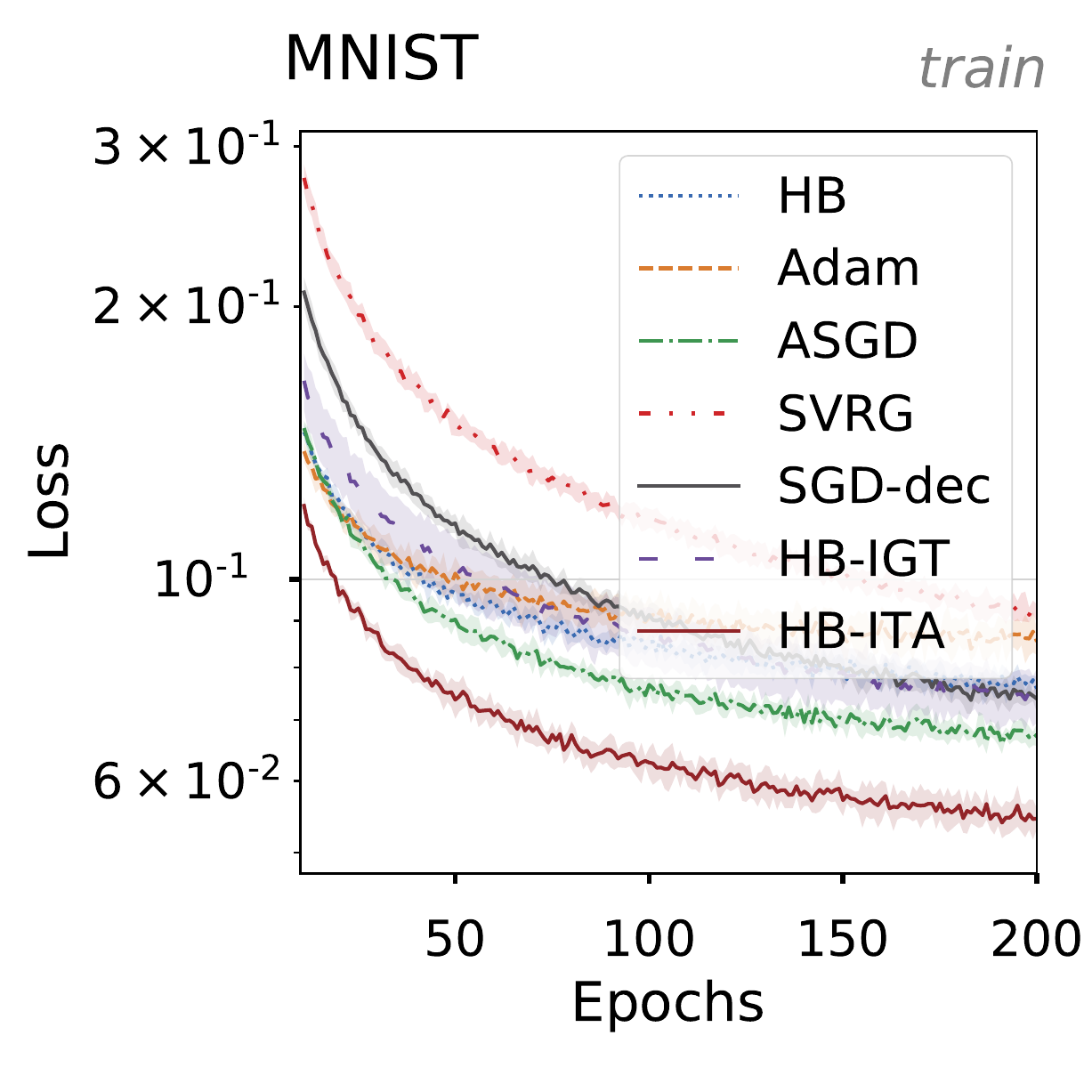}
    \includegraphics[width=0.32\textwidth]{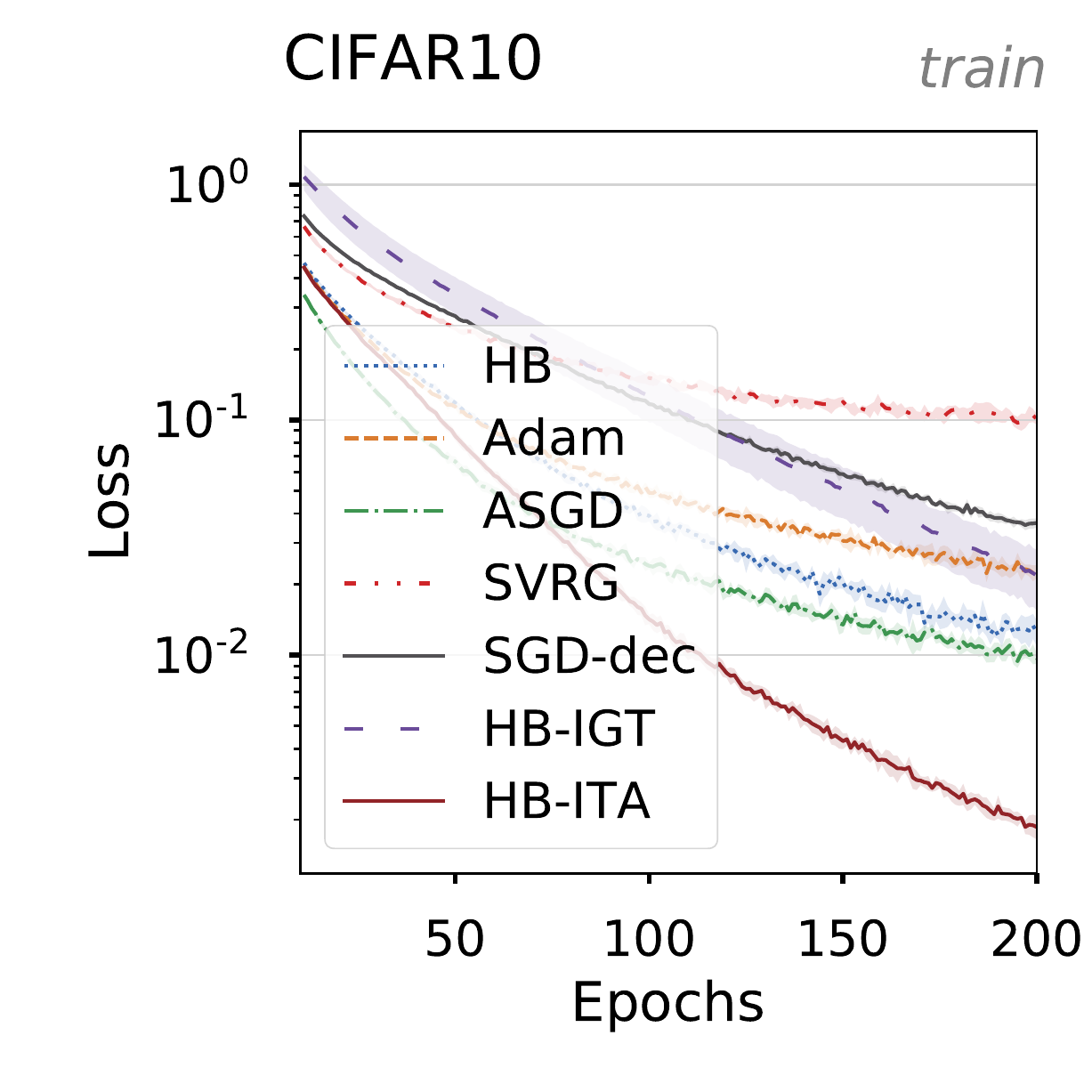}
    \caption{
        Training loss curves for different optimization algorithms on several popular benchmarks.
        For each method, the hyper-parameters are tuned to minimize the training error after 15 epochs.
        Algorithms using the IGT gradient estimates tend to outperform their stochastic gradient counter-parts.
        \textbf{Left}: Logistic regression on MNIST.
        \textbf{Center}: LeNet5 on MNIST.
        \textbf{Right}: MobileNetv2 on CIFAR10.
        \label{fig:baselines}
    }
\end{figure}

We investigate three different scenarios: (a) \textbf{linear-mnist}: a logistic regression model on MNIST, (b) \textbf{mnist}: a modified version of LeNet5 \cite{lecun2010mnist} on MNIST and (c) \textbf{cifar-small}: the MobileNetv2 architecture \cite{sandler2018mobilenetv2} consisting of 19 convolutional layers on CIFAR10.
All models are trained with a mini-batch size of 64, while the remaining hyper-parameters are available in Tables~\ref{tab:lin-mnist},~\ref{tab:mnist}, and \ref{tab:cifar_mobile}.

For each of the above tasks, models are trained for 200 epochs.
We compare the following methods: 

\begin{itemize}
    \item \textbf{HB}: the heavy ball method~\cite{polyak1992acceleration},
    \item \textbf{Adam}~\cite{kingma2014adam},
    \item \textbf{ASGD}~\cite{jain2018accelerating},
    \item \textbf{SVRG}~\cite{johnson2013accelerating},
    \item \textbf{SGD-dec}: stochastic gradient method with an exponential learning rate schedule and exponential constant $0.999$,
    \item \textbf{HB-IGT}: the heavy ball using the IGT as a plug-in estimator, and
    \item \textbf{HB-ITA}: same as HB-IGT but using the anytime tail averaging to forget the oldest gradients.
\end{itemize}

The hyperparameters of each method, and in particular the stepsize, are tuned independently according to a logarithmic grid so as to minimize the mean training error after epoch 15 on one seed. We then use those parameters on 5 random seeds and report the mean and standard deviation of the performance.

Figure~\ref{fig:baselines} shows the training curves for the five algorithms in the three settings.

First, we observe that, for the logistic regression, HB-IGT performs on par with HB-ITA and far better than all the other methods, even though the assumption on the Hessians is violated. When using a ConvNet, however, we see that HB-IGT is not competitive with state-of-the-art methods such as Adam or ASGD. HB-ITA, on the other hand, with its smaller reliance on the assumption, once again performs much better than all other methods. In fact, HB-ITA not only converges to a lower final train error but also has a faster initial rate.

While our focus is on optimization, we also report generalization metrics in Table~\ref{table:acc}.
For each algorithm, we computed the best mean accuracy after each epoch on the test set and report this value together with its standard deviation.
The importance of the Anytime Tail-Averaging mechanism is again apparent: without it, Heavyball-IGT is unable to improve for more than a few epochs on the CIFAR10 validation set, regardless of the stepsize choice.
On the other hand, it is evident from those results that the solutions found by Heavyball-ITA are competitive with the ones discovered by other optimization algorithms.

\begin{table*}[]
    \centering
    \begin{tabular}{@{}lcccc@{}}
    \toprule
                    & Linear-MNIST          & MNIST                 & CIFAR10             & IMDb             \\ \midrule
    Heavyball       & 92.52 $\pm$ 0.04      & 99.08 $\pm$ 0.07      & 91.55 $\pm$ 0.25    & 86.90 $\pm$ 0.67 \\
    Adam            & 92.57 $\pm$ 0.10      & 98.99 $\pm$ 0.05      & 89.36 $\pm$ 0.75    & 85.62 $\pm$ 0.63 \\
    ASGD            & 92.47 $\pm$ 0.08      & 99.15 $\pm$ 0.07      & 91.45 $\pm$ 0.20    & 87.31 $\pm$ 0.21 \\
    SVRG            & 92.51 $\pm$ 0.04      & 99.06 $\pm$ 0.08      & 86.84 $\pm$ 0.17    & n/a              \\
    SGD-dec         & 92.52 $\pm$ 0.06      & 99.11 $\pm$ 0.06      & 87.53 $\pm$ 0.23    & 86.73 $\pm$ 0.34 \\
    Heavyball-IGT   & 92.47 $\pm$ 0.10      & 99.00 $\pm$ 0.05      & 12.05 $\pm$ 0.21    & 86.61 $\pm$ 0.23 \\
    Heavyball-ITA   & 92.50 $\pm$ 0.10      & 99.19 $\pm$ 0.02      & 90.37 $\pm$ 0.31    & 87.26 $\pm$ 0.24 \\ \bottomrule
    \end{tabular}
    \caption{Test accuracies from the best validation epoch. \label{table:acc}}
\end{table*}

\begin{table}[h]
    \centering
    \begin{tabular}{@{}lccccc@{}}
    \toprule
             & HB        & Adam     & ASGD     & HB-IGT        & HB-ITA \\ \midrule
    $\alpha$ & 0.0128    & 0.0002   & $0.0032$ & $0.0032$      & $0.0016$      \\
    $\mu$    & 0.1       & 0.95     &  n/a     & 0.9           & 0.1           \\
    $\xi$    & n/a       & n/a      & $10$     & n/a           & n/a           \\
    $\kappa$ & n/a       & n/a      & $10^4$   & n/a           & n/a           \\ \bottomrule
    \end{tabular}
    \caption{Hyperparameters for linear-MNIST experiments.\label{tab:lin-mnist}}
\end{table}

\begin{table}[h]
    \centering
    \begin{tabular}{@{}lccccc@{}}
    \toprule
             & HB        & Adam     & ASGD     & HB-IGT        & HB-ITA \\ \midrule
    $\alpha$ & 0.0064    & 0.0016   & 0.0128   & 0.0032        & 0.0032        \\
    $\mu$    & 0.9       & 0.95     &  n/a     & 0.95          & 0.95          \\
    $\xi$    & n/a       & n/a      & $10$     & n/a           & n/a           \\
    $\kappa$ & n/a       & n/a      & $10^4$   & n/a           & n/a           \\ \bottomrule
    \end{tabular}
    \caption{Hyperparameters for MNIST experiments. \label{tab:mnist}}
\end{table}

\begin{table}[h]
    \centering
    \begin{tabular}{@{}lccccc@{}}
    \toprule
             & HB        & Adam     & ASGD     & HB-IGT        & HB-ITA \\ \midrule
    $\alpha$ & 0.0512    & 0.0512   & 0.1024   & 0.0128        & 0.0512        \\
    $\mu$    & 0.95      & 0.9      &  n/a     & 0.9           & 0.1           \\
    $\xi$    & n/a       & n/a      & 100      & n/a           & n/a           \\
    $\kappa$ & n/a       & n/a      & $10^5$   & n/a           & n/a           \\ \bottomrule
    \end{tabular}
    \caption{Hyperparameters for MobileNetV2 on CIFAR10 experiments.\label{tab:cifar_mobile}}
\end{table}

\section{Proofs}

\subsection{Transport formula}
    \begin{align*}
        g_t(\theta_t)		&= \frac{t}{t+1} g_{t-1}(\theta_{t}) + \frac{1}{t+1} g(\theta_t, x_t)\\
        &= \frac{t}{t+1} \left(g_{t-1}(\theta_{t-1}) + H(\theta_t - \theta_{t-1})\right) + \frac{1}{t+1} g(\theta_t, x_t) \tag{Quadratic $f$}\\
        &= \frac{t}{t+1} g_{t-1}(\theta_{t-1})+ \frac{1}{t+1} \left(g(\theta_t, x_t) + t H (\theta_t - \theta_{t-1})\right)\\
        &= \frac{t}{t+1} g_{t-1}(\theta_{t-1}) + \frac{1}{t+1} g(\theta_t + t (\theta_t - \theta_{t-1}), x_t) \tag{Identical Hessians}\\
        &\approx \frac{t}{t+1} \widehat{g}_{t-1}(\theta_{t-1}) + \frac{1}{t+1} g(\theta_t + t (\theta_t - \theta_{t-1}), x_t) \tag{$\widehat{g}_{t-1}$ is an approximation}\; .
    \end{align*}
    
\section{Proof of Prop.~\ref{prop:cvg}}
In this proof, we assume that $g$ is a strongly-convex quadratic function with hessian $H$.

At timestep $t$, we have access to a stochastic gradient $g(\theta, x_t) = g(\theta_t) + \epsilon_t$ where the $\epsilon_t$ are i.i.d. with variance $C \preceq \sigma^2 H$.

We first prove a simple lemma:
\begin{lemma}
\label{lemma:velocity}
If $v_0 = g(\theta_0) + \epsilon_0$ and, for $t > 0$, we have
\begin{align*}
v_t &= \frac{t}{t+1} v_{t-1} + \frac{1}{t + 1}g\left(\theta_t + t (\theta_t - \theta_{t-1})\right) + \frac{1}{t+1} \epsilon_t \; ,
\end{align*}
then
\begin{align*}
v_t &= g(\theta_t) + \frac{1}{t+1} \sum_{i=0}^t \epsilon_i \; .
\end{align*}
\begin{proof}
Per our assumption, this is true for $t=0$. Now let us prove the result by induction. Assume this is true for $t-1$. Then we have:
\begin{align*}
v_t &= \frac{t}{t+1} v_{t-1} + \frac{1}{t + 1}g(\theta_t + t (\theta_t - \theta_{t-1})) + \frac{1}{t+1} \epsilon_t\\
	&= \frac{t}{t+1} g(\theta_{t-1}) + \frac{1}{t+1} \sum_{i=0}^{t-1} \epsilon_i\\
	&\qquad + \frac{1}{t + 1}g(\theta_t + t (\theta_t - \theta_{t-1})) + \frac{1}{t+1} \epsilon_t \tag{recurrence assumption}\\
	&= \frac{t}{t+1} g(\theta_{t-1}) + \frac{1}{t+1} \sum_{i=0}^{t-1} \epsilon_i\\
	&\qquad + g(\theta_t) - \frac{t}{t+1}g(\theta_{t-1}) + \frac{1}{t+1} \epsilon_t\tag{$g$ is quadratic}\\
	&= g(\theta_t) + \frac{1}{t+1} \sum_{i=0}^t \epsilon_i \; .
\end{align*}
This concludes the proof.
\end{proof}
\end{lemma}

\begin{lemma}
\label{lemma:noise_recur}
Let us assume we perform the following iterative updates:
\begin{align*}
v_t &= \frac{t}{t+1} v_{t-1} + \frac{1}{t + 1}g(\theta_t + t (\theta_t - \theta_{t-1})) + \frac{1}{t+1} \epsilon_t\\
\theta_{t+1} &= \theta_t - \alpha v_t \; ,
\end{align*}
starting from $v_0 = g(\theta_0) + \epsilon_0$. Then, denoting $\Delta_t = \theta_t - \theta^\ast$, we have
\begin{align*}
\Delta_t &= (I - \alpha H)^t \Delta_0 - \alpha\sum_{i=0}^{t-1} N_{i, t} \epsilon_i
\end{align*}
with
\begin{align*}
N_{i, 0} 	&= 0\\
N_{i, t}	&= (I - \alpha H)N_{i, t-1} + 1_{i < t}\frac{1}{t} I \; .
\end{align*}
\end{lemma}
\begin{proof}
The result is true for $t=0$. We now prove the result for all $t$ by induction. Let us assume this is true for $t-1$.
Using Lemma~\ref{lemma:velocity}, we have
\begin{align*}
v_{t-1} &= g(\theta_{t-1}) + \frac{1}{t} \sum_{i=0}^{t-1} \epsilon_i
\end{align*}
and thus, using $g(\theta_{t-1}) = H\Delta_{t-1}$,
\begin{align*}
\Delta_t &= \Delta_{t-1} - \alpha v_{t-1}\\
		&= \Delta_{t-1} - \alpha H\Delta_{t-1} -\frac{\alpha}{t}\sum_{i=0}^{t-1} \epsilon_i\\
		&= (I - \alpha H)\Delta_{t-1} - \frac{\alpha}{t}\sum_{i=0}^{t-1} \epsilon_i\\
		&= (I - \alpha H)^t \Delta_0 -\alpha \sum_{i=0}^{t-2}(I - \alpha H)N_{i, t-1} \epsilon_i - \frac{\alpha}{t}\sum_{i=0}^{t-1} \epsilon_i \tag{recurrence assumption}\\
		&= (I - \alpha H)^t \Delta_0 -\alpha \sum_{i=0}^{t-1}N_{i, t}\epsilon_i
\end{align*}
with
\begin{align*}
N_{i,t}	&= (I - \alpha H)N_{i, t-1} + 1_{i < t}\frac{1}{t} I \; .
\end{align*}
This concludes the proof.
\end{proof}

For the following lemma, we will assume that the Hessian is diagonal and will focus on one dimension with eigenvalue $h$. Indeed, we know that there are no interactions between the eigenspaces and that we can analyze each of them independently~\cite{o2015adaptive}.
\begin{lemma}
\label{lemma:noise}
Denote $r_h = 1 - \alpha h$. We assume $\alpha \leq \frac{1}{L}$. Then, for any $i$ and any $t$, we have
\begin{align*}
N_{i, t} &\geq 0\tag{Positivity}\\
N_{i, t} &= 0 \quad \textrm{ if } t \leq i \tag{Zero-start}\\
N_{i, t} &\leq \log\left(\frac{2}{i(1 - r_h)}\right) \quad \textrm{ if } i < t \leq \frac{2}{1 - r_h)}\tag{Constant bound}\\
N_{i, t} & \leq \frac{\max\left\{1 + r_h, 2\log\left(\frac{2}{i(1 - r_h)}\right)\right\}}{t (1 - r_h)} \quad \textrm{if } \frac{2}{1 - r_h}\leq t \; .\tag{Decreasing bound}
\end{align*}
\end{lemma}
\begin{proof}
The Zero-start case $i \geq t$ is immediate from the recursion of Lemma~\ref{lemma:noise_recur}. The Positivity property of $N_{i, t}$ is also immediate from the recursion since the stepsize $\alpha$ is such that $r_h = 1 - \alpha h$ is positive.

We now turn to the Constant bound property.
We have, for $t > i$,
\begin{align*}
N_{i, t}  	&= r_h N_{i, t-1} + \frac{1}{t}\\
			&\leq N_{i, t-1} + \frac{1}{t} \; .
\end{align*}
Thus, $N_{i, t} - N_{i, t-1} \leq \frac{1}{t}$. Summing these inequalities, we get a telescopic sum and, finally:
\begin{align*}
N_{i, t}	&\leq \sum_{j=i+1}^t \frac{1}{j}\\
			&\leq \int_{x = i}^t \frac{dx}{x}\\
			&= \log\left(\frac{t}{i}\right) \; .
\end{align*}
This bound is trivial in the case $i=0$. In that case, we keep the first term in the sum separate and get
\begin{align*}
N_{0, t}	&\leq 1 + \log t \; .
\end{align*}
In the remainder, we shall keep the $\log\left(\frac{t}{i}\right)$ bound for simplicity.
The upper bound on the right-hand size is increasing with $t$ and its value for $t = \frac{2}{1 - r_h}$ is thus an upper bound for all smaller values of $t$. Replacing $t$ with $\frac{2}{1 - r_h}$ leads to
\begin{align*}
N_{i, \frac{2}{1 - r_h}}	&\leq \log\left(\frac{\frac{2}{1 - r_h}}{i}\right)\\
							&= \log\left(\frac{2}{i(1 - r_h)}\right) \; .
\end{align*}
This proves the third inequality.

We shall now prove the Decreasing bound by induction. This bound states that, for $t$ large enough, each $N_{i, t}$ decreases as $O(1/t)$. Using the second and third inequalities, we have
\begin{align*}
N_{i, \frac{2}{1 - r_h}}	&\leq \log\left(\frac{2}{i(1 - r_h)}\right) \frac{\frac{2}{1 - r_h}}{\frac{2}{1 - r_h}}\\
							&= \frac{\log\left(\frac{2}{i(1 - r_h)}\right)}{1 - r_h}\frac{2}{\frac{2}{1 - r_h}}\\
							&\leq \frac{\max\left\{1 + r_h, 2\log\left(\frac{2}{i(1 - r_h)}\right)\right\}}{\frac{2}{1 - r_h}(1 - r_h)}\ .
\end{align*}
The maximum will help us prove the last property. Thus, for $t = \frac{2}{1 - r_h}$, we have
\begin{align*}
N_{i, t}		&\leq \frac{\max\left\{1 + r_h, 2\log\left(\frac{1}{i(1 - r_h)}\right)\right\}}{t(1 - r_h)}\\
				&\leq \frac{\nu_i}{t} \; ,
\end{align*}

with $\nu_i = \frac{\max\left\{1 + r_h, 2\log\left(\frac{1}{i(1 - r_h)}\right)\right\}}{(1 - r_h)}$. The Decreasing bound is verified for $t = \frac{2}{1 - r_h}$.

We now show that if, for any $t > \frac{2}{1 - r_h}$, we have $N_{i, t-1} \leq \frac{\nu_i}{t-1}$, then $N_{i, t} \leq \frac{\nu_i}{t}$. Assume that there is such at $t$. Then
\begin{align*}
N_{i, t}	&= r_h N_{i, t-1} + \frac{1}{t}\\
			&\leq \frac{r_h\nu_i}{t-1} + \frac{1}{t}\\
			&= \frac{r_h t \nu_i + t - 1}{t(t - 1)}\\
			&=\frac{(t - 1)\nu_i + (r_h - 1) t \nu_i + \nu_i + t - 1}{t(t-1)}\\
			&=\frac{\nu_i}{t} + \frac{(r_h - 1) t \nu_i + \nu_i + t - 1}{t(t - 1)} \; .
\end{align*}
We now shall prove that $(r_h - 1) t \nu_i + \nu_i + t - 1 = [(r_h - 1) \nu_i + 1]t + \nu_i -1$ is negative.
First, we have that
\begin{align*}
(r_h - 1)\nu_i +1		&= 1 - \max\left\{1 + r_h, 2\log\left(\frac{1}{i(1 - r_h)}\right)\right\}\\
						&\leq 0 \; .
\end{align*}
Then,
\begin{align*}
[(r_h - 1) \nu_i + 1]t + \nu_i -1	\leq 0	&\Longleftrightarrow t \geq \frac{\nu_i - 1}{(1 - r_h)\nu_i - 1}
\end{align*}since $(r_h - 1)\nu_i + 1 \leq 0$. Thus, the property is true for every $t \geq \frac{\nu_i - 1}{(1 - r_h)\nu_i - 1}$. In addition, we have
\begin{align*}
\nu_i &\geq \frac{1 + r_h}{1-r_h}\\
\nu_i (1 - r_h) &\geq 1 + r_h\\
2\nu_i (1 - r_h) - 2 &\geq \nu_i (1 - r_h) - 1 + r_h\\
\frac{2}{1 - r_h}	&\geq \frac{\nu_i - 1}{\nu_i (1 - r_h) - 1} \; ,
\end{align*}
and the property is also true for every $t \geq \frac{2}{1 - r_h}$.
This concludes the proof.
\end{proof}

Finally, we can prove the Proposition~\ref{prop:cvg}:
\begin{proof}
The expectation of $\Delta_t$ is immediate using Lemma~\ref{lemma:noise_recur} and the fact that the $\epsilon_i$ are independent, zero-mean noises.
The variance is equal to $V[\Delta_t]	= \alpha^2 B \sum_{i=0}^t N_{i, t}^2$. While our analysis was only along one eigenspace of the Hessian with associated eigenvalue $h$, we must now sum over all dimensions. We will thus define
\begin{align*}
\bar{\nu}_i	&= \frac{\max\left\{2 - \alpha \mu, 2\log\left(\frac{1}{i\alpha \mu}\right)\right\}}{\alpha \mu} \quad \textrm{ for } i > 0\\
\bar{\nu}_0	&= \frac{2 + 2\log\left(\frac{1}{\alpha \mu}\right)}{\alpha \mu} \; ,
\end{align*}
which is, for every $i$, the maximum $\nu_i$ across all dimensions. We get
\begin{align*}
V[\Delta_t]		&\leq d\alpha^2 B \sum_{i=0}^t \frac{\bar{\nu}_i^2}{t^2}\\
				&\leq d\alpha^2 B \sum_{i=0}^t \frac{\bar{\nu}_0^2}{t^2} \quad \textrm{since } \nu_i \geq \nu_{i+1} \; \forall i\\
				&\leq \frac{d\alpha^2 B \bar{\nu}_0^2}{t} \; .
\end{align*}
Since we have
\begin{align*}
    E[\theta_t - \theta^\ast]  &= (I - \alpha H)^t (\theta_0 - \theta^\ast) \; ,
\end{align*}
we get
\begin{align*}
    E[\|\theta_t - \theta^\ast\|^2]    &= \|E[\theta_t - \theta^\ast]\|^2 + V[\Delta_t]\\
        &\leq (\theta_0 - \theta^\ast)^\top (I - \alpha H)^{2t} (\theta_0 - \theta^\ast) + \frac{d\alpha^2 B \bar{\nu}_0^2}{t}\\
        &\leq \left(1 - \frac{1}{\kappa}\right)^{2t} \|\theta_0 - \theta^\ast\|^2 + \frac{d\alpha^2 B \bar{\nu}_0^2}{t} \; .
\end{align*}

This concludes the proof.
\end{proof}

\section{Proof of Proposition~\ref{prop:hb-igt-nonstoch} and Proposition~\ref{prop:hb-igt-stoch}}

In this section we list and prove all lemmas used in the proofs of Proposition~\ref{prop:hb-igt-nonstoch} and Proposition~\ref{prop:hb-igt-stoch}; all lemmas are stated in the same conditions as the proposition. 

We start the following proposition:
\begin{proposition}
        \label{prop:hbigt}
        Let $f$ be a quadratic function with positive definite Hessian $H$ with largest eigenvalue $L$ and condition number $\kappa$ and if the stochastic gradients satisfy $g(\theta, x) = g(\theta) + \epsilon$ with $\epsilon$ a random uncorrelated noise with covariance bounded by $BI$.

        Then, Algorithm~\ref{alg:heavyball-igt}  leads to iterates $\theta_t$ satisfying
        \begin{align}
            E[\theta_t - \theta^\ast]
            &= \begin{pmatrix}
            I \\ 0 \end{pmatrix}
            A^t 
            \begin{pmatrix}
                E[\theta_1 - \theta^\ast] \\
                E[\theta_0 - \theta^\ast]
            \end{pmatrix}
        \end{align}
        where
        \begin{align}
            A &= \begin{pmatrix} I - \alpha H + \mu I & -\mu I \\ I & 0 \end{pmatrix}
        \end{align}
        governs the dynamics of this bias.
        In particular, when its spectral radius, $\rho(A)$ is less than $1$,
        the iterates converge linearly to $\theta^\ast$.

        In a similar fashion, the variance dynamics of Heavyball-IGT are governed by the matrix
        {\small
        \begin{align*}
            D_i&= \begin{pmatrix}
                (1 - \alpha h_i + \mu)^2 + 2\alpha^2 h_i^2  &  \mu^2  &  -2\mu(1 - \alpha h_i + \mu)^2  \\
                1  &  0  &  0 \\
                1 - \alpha h_i + \mu  &  0  &  -\mu
            \end{pmatrix}
        \end{align*}
        }
        If the spectral radius of $D_i$, $\rho(D_i)$, is strictly less than 1 or all $i$, then there exist constants $t_0>0$ and $C>0$ for which
        \begin{align*}
            \Var{(\theta_t)} 
            &\leq
            2\alpha^2 d B C
            \frac{ \log(t)}{t}, \quad \textrm{for}\ t>t_0
        \end{align*}
        where $B$ is a bound on the variance of noise variables $\epsilon_i$.
    \end{proposition}

\begin{lemma}[IGT estimator as true gradient plus noise average]
\label{lem:igt_grad_plus_var}
    If $v_0 = g(\theta_0) + \epsilon_0$ and for $t > 0$ we have 
    $$ v_t = \frac{t}{t+1} v_{t-1} + \frac{1}{t+1} g(\theta_t + t (\theta_t - \theta_{t-1})) + \frac{1}{t+1} \epsilon_t, $$
    then
    $$v_t = g(\theta_t) + \frac{1}{t+1} \sum_{i=0}^t \epsilon_i. $$

\end{lemma}
This lemma is already proved in the previous section for the IGT estimator (Lemma~\ref{lemma:velocity}) and is just repeated here for completeness.
We will use this result in the next few lemmas.

\begin{lemma}[The IGT gradient estimator is unbiased on quadratics]
\label{lem:igt-zero-mean}
For the IGT gradient estimator, $v_t$, corresponding to parameters $\theta_t$ we have
$$
\E \left[ v_t \right] =  g(\E \theta_t),
$$
where the expectation is over all gradient noise vectors $\epsilon_0, \epsilon_1, \ldots, \epsilon_t$. 
\end{lemma}

\begin{proof}
The proof proceeds by induction.
The base case holds as we have

$$\E \left[ v_0 \right] = \E \left[ g_0 + \epsilon_0 \right] = g(\theta_0).$$

For the inductive case, we can write

\begin{align*}
	\e{v_{t}}
	&= \e{\frac{t}{t+1} v_{t-1} + \frac{1}{t+1} \hat{g}(\theta_t + t(\theta_t - \theta_{t-1}))} \\
	&= \e{\frac{t}{t+1} v_{t-1} + \frac{1}{t+1} g_t + \frac{t}{t+1} g_t - \frac{t}{t+1} g_{t-1} + \frac{1}{t+1}\epsilon_t} \\
	&= \frac{t}{t+1}\e{v_{t-1} - g_{t-1}} + \e{g_{t}} + \frac{t}{t+1} \e{\epsilon_t} \\
	&= \e{g_t} = g(\e{\theta_t}).
\end{align*}

Where, in the third equality, $\e{v_{t-1} - g_{t-1}} = 0$ by the inductive assumption,
and the last equality because the gradient of a quadratic function is linear.
\end{proof}

\begin{lemma}[Bounding the IGT gradient variance]
    \label{lem:igt-grad-variance}
	Let $v_t$ be the IGT gradient estimator.
	Then
	$$\var{v_t} \leq 2 h^2 \var{\theta_t - \theta^\star} + \frac{2B}{t} ,$$
	
	where $B$ is the variance of the homoscedastic noise $\epsilon_t$.
\end{lemma}
\begin{proof}
    \begin{align*}
	\var{v_t}
        &= \var{g_t + \frac{1}{t+1} \sum_{i=0}^{t} \epsilon_i} \\
        &= \var{h\theta_t} + \var{\frac{1}{t+1} \sum_{i=0}^t \epsilon_i} \\
        & \qquad + 2 \cov{h\theta_t, \frac{1}{t+1} \sum_{i=0}^{t} \epsilon_i} \\
        &\leq  2\var{h\theta_t} + 2\var{\frac{1}{t+1} \sum_{i=0}^t \epsilon_i} \\
        &= 2 h^2 \var{\theta_t - \theta^\star} + 2 \frac{B}{t} 
    \end{align*}
\end{proof}

Now that we have these basic results on the IGT estimator, we can analyze the evolution of the bias and variance of Heavyball-IGT. We use the quadratic assumption to decouple the vector dynamics of Heavyball-IGT into independent scalar dynamics.
If the Hessian, $H$, has eigenvalues
$L \geq h_1 \geq h_2 \geq \ldots
\geq h_n = L/\kappa$, 
then we can assume without loss of generality that $H$ is diagonal with $H_{ii} = h_i$.

\begin{lemma}[Evolution of bias for scalar quadratic]
Assume that the Hessian, second derivative, is h.

Starting with $v_0 = g(\theta_0) + \epsilon_0$ and $w_0 = 0$, performing the following iterative updates (Heavyball-IGT, Algorithm~\ref{alg:heavyball-igt}):
$$
v_t = \frac{t}{t + 1} v_{t-1} + \frac{1}{t+1} g(\theta + t(\theta_t - \theta_{t-1})) + \frac{1}{t+1} \epsilon_t,
$$
$$
w_{t+1} = \mu w_{t} + \alpha v_t,
\qquad
\theta_{t+1} = \theta_{t} - w_{t+1}
$$
results in
$$
\Delta_t = A^t \Delta_0 - \alpha \sum_{i=0}^{t-1} N_{i, t} \noisevector
$$

where 
$ N_{j, 0} = 0_{2\times 2}, \qquad N_{i, t} = AN_{i, t-1} + 1_{i<t} \frac{1}{t} I$,

$\Delta_t = \begin{bmatrix} \theta_t - \theta^\ast \\ \theta_{t-1} - \theta^\ast \end{bmatrix}$
and
$A = \begin{pmatrix} 1 - \alpha h + \mu & -\mu \\ 1 & 0 \end{pmatrix}.$

\end{lemma}

\begin{proof}

The proof proceeds by induction.
First notice that for $t = 0$ the equality naturally holds.
We make the inductive assumption that it holds for $t - 1$, and start by using Lemma~\ref{lem:igt_grad_plus_var}:

\begin{align*}
    \Delta_t &= A\Delta_{t-1}  - \frac{\alpha}{t} \sum_{i=0}^{t-1} \noisevector \\
             &= A(A^{t-1}\Delta_0 - \alpha \sum_{i=0}^{t-2} N_{i, t} \noisevector)  - \frac{\alpha}{t} \sum_{i=0}^{t-1} \noisevector \tag{Inductive assumption} \\
             &= A^{t}\Delta_0 - \alpha (\sum_{i=0}^{t-2} A N_{i, t} \noisevector  + \frac{1}{t} \sum_{i=0}^{t-1} \noisevector) \\
             &= A^{t}\Delta_0 - \alpha \sum_{i=0}^{t-1} N_{i, t} \noisevector \tag{Def. of $N_{i, t}$}
\end{align*}
\end{proof}




\begin{lemma}[Evolution of variance]
    Let $U_t = \var{\theta_t}$ and $V_t = \cov{\theta_{t}, \theta_{t-1}}$, where $\theta_t$ is the $t$-th iterate of Heavyball-IGT on a 1-dimensional quadratic function with curvature $h$.
    The following matrix describes the variance dynamics of Heavyball-IGT.
    \begin{equation}
        D
        = 
        \begin{pmatrix}
            (1 - \alpha h + \mu)^2 + 2\alpha^2 h^2  &  \mu^2  &  -2\mu(1 - \alpha h + \mu)^2  \\
            1  &  0  &  0 \\
            1 - \alpha h + \mu  &  0  &  -\mu
        \end{pmatrix}
    \end{equation}
    
    If the spectral radius of $D$, $\rho(D)$, is strictly less than 1, then there exist constants $t_0>0$ and $C>0$ for which
    \begin{align*}
        \Var{(\theta_t)} 
        &\leq
        2\alpha^2 B C
            \frac{ \log(t)}{t}
    \end{align*},
    where $B$ is a bound on the variance of the noise.
\end{lemma}

\begin{proof}
    The proof (and lemma) is similar to the proof of Lemma 9 in \cite{zhang2017yellowfin}.
    We start by expanding $U_{t+1}$ as follows.

    \begin{align*}
        U_{t+1} &= \e{(\theta_{t+1} - \bar{\theta}_{t+1})^2} \\
            &= \e{(\theta_t - \alpha v_t + \mu     (\theta_t - \theta_{t-1}) -      \bar{\theta}_t + \alpha g_t -        \mu(\bar{\theta}_t -             \bar{\theta}_{t-1}))^2} \\
            &= \mathbb{E}[(\theta_t - \alpha g_t + \mu (\theta_t - \theta_{t-1}) - \bar{\theta}_t + \alpha g_t\\
                &\qquad - \mu(\bar{\theta}_t - \bar{\theta}_{t-1}) + \alpha (g_t - v_t))^2] \\
            &= \e{((1 - \alpha h + \mu)(\theta_t - \bar{\theta}_t) - \mu (\theta_{t-1} - \bar{\theta}_{t-1}))^2}\\
                &\qquad + \alpha^2 \e{(g_t - v_t)^2} \\
            &\leq \e{((1 - \alpha h + \mu)(\theta_t - \bar{\theta}_t) - \mu (\theta_{t-1} - \bar{\theta}_{t-1}))^2}\\
                &\qquad 
                + \alpha^2 \left(
                    2h^2 \e{(\theta_t - \bar{\theta}_t)^2} + \frac{2B}{t+1}
                \right) \\
            &\leq \left[ (1 - \alpha h + \mu)^2
                +2\alpha^2 \mu^2 \right] \e{(\theta_t - \bar{\theta}_t)^2}\\
                &\qquad - 2\mu(1 - \alpha h + \mu) \e{(\theta_t - \bar{\theta}_t)(\theta_{t-1} - \bar{\theta}_{t-1})} \\
                &\qquad + \mu^2 \e{(\theta_{t-1} - \bar{\theta}_{t-1})^2} + \alpha^2 \frac{2B}{t+1} \; .
    \end{align*}

    Where the fourth equality is obtained since we know that the IGT gradient estimator is unbiased, i.e. $\e{g_t - v_t} = 0$.
    The first inequality stems from Lemma~\ref{lem:igt-grad-variance}.
    We similarly expand $V_t$.

    \begin{align*}
        V_t &= \e{(\theta_t - \bar{\theta}_t)(\theta_{t-1} - \bar{\theta}_{t-1})} \\
            &= \mathbb{E}\left[\left((1 - \alpha h + \mu)(\theta_{t-1} - \bar{\theta}_{t-1}) - \mu (\theta_{t-2} - \bar{\theta_{t-2}})+ \alpha (g_t - v_t)\right)\right.\\
            &\qquad \left.(\theta_{t-1} - \bar{\theta}_{t-1})\right] \\
            &= (1 - \alpha h + \mu)\e{(\theta_{t-1} - \bar{\theta}_{t-1})^2}\\
            &\qquad - \mu \e{(\theta_{t-1} - \bar{\theta}_{t-1})(\theta_{t-2} - \bar{\theta}_{t-2})}
    \end{align*}

    From the above expressions, we obtain

    \begin{align*}
        \begin{pmatrix}
            U_{t+1} \\ U_t \\ V_{t+1}
        \end{pmatrix}
        &\leq 
            D
            \begin{pmatrix}
                U_{t} \\ U_{t-1} \\ V_{t}
            \end{pmatrix}
            +
            \begin{pmatrix}
                \alpha^2 \frac{2B}{t+1} \\ 0 \\ 0
            \end{pmatrix} 
            \\
        &\leq  
            2\alpha^2 B \sum_{i=0}^{t} D^i
            \begin{pmatrix}
                \frac{1}{t+1-i} \\ 0 \\ 0
            \end{pmatrix} \\
        &\leq  
            2\alpha^2 B \left( 
                \sum_{i=0}^{s-1} D^i
            \begin{pmatrix}
                \frac{1}{t+1-i} \\ 0 \\ 0
            \end{pmatrix}
            +
            \sum_{i=s}^{t} D^i
            \begin{pmatrix}
                \frac{1}{t+1-i} \\ 0 \\ 0
            \end{pmatrix}
            \right) 
    \end{align*}
    where an inequality of vectors implies the corresponding elementwise inequalities.
    
    If the spectral radius of $D$, $\rho(D)$ is strictly less than 1, then there exists constant $C'>0$ such that 
    \begin{align*}
        \begin{pmatrix}
            1 \\ 0 \\ 0
        \end{pmatrix}^T 
        \sum_{i=0}^{s-1} D^i
        \begin{pmatrix}
            \frac{1}{t+1-i} \\ 0 \\ 0
        \end{pmatrix} 
    &\leq 
        C' \sum_{i=0}^{s-1} 
            \frac{1}{t+1-i}
            \\
    &\leq 
        \frac{C' s}{t+2-s}
    \end{align*}

    If the spectral radius of $D$, $\rho(D)$, is strictly less than 1, then there exists constant $\zeta>0$ and constant $C''(\zeta)>0$ such that, $\rho(D)+\zeta<1$ and
    \begin{align*}
        \begin{pmatrix}
            1 \\ 0 \\ 0
        \end{pmatrix}^T 
        \sum_{i=s}^{t} D^i
            \begin{pmatrix}
                \frac{1}{t+1-i} \\ 0 \\ 0
            \end{pmatrix}
        \leq&             
       \begin{pmatrix}
            1 \\ 0 \\ 0
        \end{pmatrix}^T 
        \sum_{i=s}^{t} D^i
            \begin{pmatrix}
                1 \\ 0 \\ 0
            \end{pmatrix} \\
        \leq&
        C''
        \sum_{i=s}^{t} (\rho(D)+\zeta)^s \\
        =&
        C''
        (t-s+1)(\rho(D)+\zeta)^s
    \end{align*}
    
    Let $\rho'=\rho(D)+\zeta$ and 
    $s = \lceil 2 \log_{1/\rho'}t \rceil$.
    Then $(\rho(D)+\zeta)^s = 1/t^2$,
    and putting the above two bounds together, 
    
    \begin{align*}
        U_{t+1} 
        &\leq 
        2\alpha^2 B
        \left(
            \frac{2 C' \log_{1/\rho'}t}{t+2-2\log_{1/\rho'}t}
            +
            C'' \frac{t-2 \log_{1/\rho'}t + 1}{t^2}
        \right) \\
        &\leq
        2\alpha^2 B C
            \frac{ \log(t+1)}{t+1}
    \end{align*} 
    where the last inequality holds for 
    $t > t_0$ for some $t_0$ and some constant $C>0$.

\end{proof}

We can now prove Proposition~\ref{prop:hbigt}.

\begin{proof}[Proof of Proposition~\ref{prop:hbigt}]
The bias statement of the proposition follows directly from taking an expectation on the bound of Lemma~B.4, and the variance statement from summing up the $d$ different variance terms given for each scalar component by Lemma~B.5.
\end{proof}

\subsection{Proof of Proposition~\ref{prop:hb-igt-nonstoch}}
This Proposition follows from the observation that, in the noiseless case, $\epsilon_t=0$ in our model. 
    In that case, Lemma~\ref{lem:igt-zero-mean} shows that Heavyball-IGT reduces to the heavy ball, and the rest follows from the optimal tuning of the heavy ball \cite{zhang2017yellowfin}.

\subsection{Proof of Proposition~\ref{prop:hb-igt-stoch}}
\label{sec:proof-hb-igt-stoch}
\begin{proof}
    Like we did in previous proofs, we can assume without loss of generality that the Hessian, $H$, is diagonal with elements $h_i$. 
    For a diagonal $H$, matrix $A$ can be permuted to be block diagonal with blocks
    \[
        A_i = 
        \begin{pmatrix}
            1-\alpha h_i + \mu & \mu \\
            1 & 0 \\
        \end{pmatrix}.
    \]
    To prove that $\rho(A)<1$ it suffices to prove that
    $\rho(A_i)<1$ for all $i$.
    For the rest of the proof we will focus on the dynamics along a single eigendirection with curvature $h_i$.
    The rest of this proof used $D$ to denote $D_i$, $A$ to denote $A_i$ and $h$ to denote $h_i$.
  
    To make explicit the dependence of matrices $A$ and $D$ on hyperparameters and curvature, we write $A(\alpha, \mu, h)$ and $D(\alpha, \mu, h)$.
    Let $0<\alpha<2/(3h)$ and $\mu_0=0$.
    Using hyperparameters $(\alpha, \mu_0)$ one gets the  results for gradient descent without momentum. 
    In particular $\rho(A(\alpha, \mu_0, h))=|1-\alpha h |< 1$, and the spectral radius of 
    $D$ is $\rho(D(\alpha, \mu_0, h))=|(1-\alpha h)^2 + 2 \alpha^2 h^2|<1$.
    
    We will argue that there exists $\mu>0$, such that 
   $\rho(A(\alpha, \mu, h))<1$, and the spectral radius of 
    $D$ is $\rho(D(\alpha, \mu, h))<1$.
    Then the previous lemma implies that bias converges linearly, and variance is $O(\log(t)/t)$.
    
    To argue the existence of $\mu>0$, we will perform eigenvalue perturbation analysis using the Bauer-Fike theorem. Note that
    $
        A(\alpha, \mu, h)
        =
        A(\alpha, \mu_0, h)
        +
        \mu \Delta_A
    $
    where 
    \[ \Delta_A = \begin{pmatrix}
            1 & -1\\
            0 & 0
        \end{pmatrix}.
    \]
    Similarly,
    $
    D(\alpha, \mu, h)
    \approx 
    D(\alpha, \mu_0, h)
    +
    \mu \Delta_D
    $
    where 
    \[ \Delta_D = \begin{pmatrix}
            2(1-\alpha h) & 0 & -2(1-\alpha h) \\
            0 & 0 & 0 \\
            1 & 0 & -1
        \end{pmatrix}.
    \]
    This last approximate inequality is a first-order approximation, in the sense that we are working with arbitrarily small, positive values of $\mu$, and we
    have kept terms linear in $\mu$ but ignored higher order terms, like $\mu^2$.
    
    We will apply the Bauer-Fike theorem to bound the eigenvalues of $D(\alpha, \mu, h)$.
    Consider the eigendecomposition
    $D(\alpha, \mu_0, h) = V \Lambda V^{-1}$. 
    We can compute 
    \[
        V = \begin{pmatrix}
            0 & 0 & \frac{1-2\alpha h + 3 \alpha^2 h^2}{1-\alpha h }  \\
            0 & 1 & \frac{1}{1-\alpha h} \\
            1 & 0 & 1
        \end{pmatrix}
    \]
    and 
    \[
        V^{-1} = \begin{pmatrix}
            \frac{1-\alpha h }{1-2\alpha h + 3 \alpha^2 h^2} & -\frac{1}{1-2\alpha h + 3 \alpha^2 h^2} & 0  \\
            0 & 1 & 0 \\
            1 & 0 & 0 
        \end{pmatrix}.
    \]
    Note that because we assume $\alpha < 2/(3h)$ we get
    $1-\alpha h > 0$. 
    Also, $1-2\alpha h + 3\alpha^2 h^2 > 0$ regardless 
    of the choice of hyperparameters. 
    This means that matrices $V$ and $V^{-1}$ are singular and of finite norm.
    The norm of $\Delta_D$ is also finite. 
    The Bauer-Fike theorem states that, 
    if $\nu$ is an eigenvalue of $D(\alpha, \mu_0, h)$,
    then there exists an eigenvalue $\lambda$ of $D(\alpha,  \mu, h)$ such that
    \[
        |\lambda - \nu |
        \leq 
        \| V \|_p \|V^{-1}\|_p \|\mu \Delta_D \|_p,
    \]
    for any $p$-norm.
    Since by construction $|\nu| \leq \rho(D(\alpha, \mu_0, h))<1$, the above means that there exists a sufficiently small, but strictly positive value of $\mu$, 
    such that $\lambda < 1$.
    By repeating this argument for all pairs of eigenvalues, we get the stated result. 
    The same argument can be repeated to prove the existence of a strictly positive $\mu$ such that
    $\rho(A(\alpha, \mu, h))<1$.

\end{proof}

\end{document}